\documentclass[english]{article}
\usepackage[T1]{fontenc}
\usepackage[latin9]{inputenc}
\usepackage{geometry}
\usepackage{color}
\usepackage{babel}
\usepackage{booktabs}
\usepackage{mathtools}
\usepackage{bm}
\usepackage{amsmath}
\usepackage{amsthm}
\usepackage{amssymb}
\usepackage{graphicx}
\usepackage[unicode=true,pdfusetitle,
 bookmarks=true,bookmarksnumbered=true,bookmarksopen=true,bookmarksopenlevel=2,
 breaklinks=false,pdfborder={0 0 0},pdfborderstyle={},backref=false,colorlinks=true]
 {hyperref}
\hypersetup{
 linkcolor=blue, urlcolor=blue, citecolor=blue, linktoc=all, linktocpage=false}

\makeatletter

\providecommand{\tabularnewline}{\\}

\theoremstyle{plain}
\newtheorem{fact}{\protect\factname}
\theoremstyle{remark}
\newtheorem{rem}{\protect\remarkname}
\theoremstyle{plain}
\newtheorem{assumption}{\protect\assumptionname}
\theoremstyle{plain}
\newtheorem{thm}{\protect\theoremname}
\theoremstyle{plain}
\newtheorem{cor}{\protect\corollaryname}
\theoremstyle{plain}
\newtheorem{lem}{\protect\lemmaname}
\theoremstyle{plain}
\newtheorem{prop}{\protect\propositionname}
\theoremstyle{definition}
 \newtheorem{example}{\protect\examplename}

\usepackage{babel}
\usepackage{bm}
\usepackage{bbm}

\usepackage[linesnumbered,ruled,vlined]{algorithm2e}
\usepackage{algorithmic}
\renewcommand{\Pr}{{\mathbb{P}}}
\newcommand{\E}{{\mathbb{E}}}
\newcommand{\R}{{\mathbb{R}}}
\newcommand{\Ncal}{{\mathcal{N}}}

\let\hat\widehat
\let\bar\overline
\let\tilde\widetilde

\allowdisplaybreaks

\usepackage{color}

\definecolor{todo}{RGB}{0,200,200}

\definecolor{new}{RGB}{0,200,200}
\newcommand{\new}[1]{\textcolor{todo}{[NEW:]}}

\definecolor{emerald}{rgb}{0.31, 0.78, 0.47}

\theoremstyle{plain}

\newcommand{\bA}{\bm{A}}
\newcommand{\bB}{\bm{B}}
\newcommand{\bC}{\bm{C}}
\newcommand{\bD}{\bm{D}}
\newcommand{\bE}{\bm{E}}

\newcommand{\bG}{\bm{G}}
\newcommand{\bH}{\bm{H}}
\newcommand{\bI}{\bm{I}}

\newcommand{\bL}{\bm{L}}
\newcommand{\bM}{\bm{M}}
\newcommand{\bN}{\bm{N}}

\newcommand{\bP}{\bm{P}}
\newcommand{\bQ}{\bm{Q}}
\newcommand{\bR}{\bm{R}}
\newcommand{\bS}{\bm{S}}

\newcommand{\bU}{\bm{U}}
\newcommand{\bV}{\bm{V}}
\newcommand{\bW}{\bm{W}}
\newcommand{\bX}{\bm{X}}
\newcommand{\bY}{\bm{Y}}
\newcommand{\bZ}{\bm{Z}}

\newcommand{\ba}{\bm{a}}
\newcommand{\bb}{\bm{b}}

\newcommand{\be}{\bm{e}}
\newcommand{\bg}{\bm{g}}
\newcommand{\bh}{\bm{h}}

\newcommand{\bu}{\bm{u}}
\newcommand{\bv}{\bm{v}}
\newcommand{\bw}{\bm{w}}
\newcommand{\bx}{\bm{x}}
\newcommand{\by}{\bm{y}}
\newcommand{\bz}{\bm{z}}

\makeatother

\providecommand{\assumptionname}{Assumption}
\providecommand{\corollaryname}{Corollary}
\providecommand{\examplename}{Example}
\providecommand{\factname}{Fact}
\providecommand{\lemmaname}{Lemma}
\providecommand{\propositionname}{Proposition}
\providecommand{\remarkname}{Remark}
\providecommand{\theoremname}{Theorem}

\begin{document}
\title{Learning Mixtures of Linear Dynamical Systems\footnotetext{Corresponding author: Yanxi Chen (email: \texttt{yanxic@princeton.edu}).}}
\author{Yanxi Chen\thanks{Department of Electrical and Computer Engineering, Princeton University, Princeton,
NJ 08544, USA; email: \texttt{\{yanxic,poor\}@princeton.edu}.} \and H.~Vincent Poor\footnotemark[1]}
\date{\today}

\maketitle
\global\long\def\poly{\mathsf{poly}}%
\global\long\def\plog{\mathsf{polylog}}%
\global\long\def\Frm{{\rm F}}%
\global\long\def\Tr{\mathsf{Tr}}%

\global\long\def\rank{\mathsf{rank}}%
\global\long\def\spn{\mathsf{span}}%
\global\long\def\col{\mathsf{col}}%
\global\long\def\row{\mathsf{row}}%
\global\long\def\vc{\mathsf{vec}}%
\global\long\def\mat{\mathsf{mat}}%
\global\long\def\SNR{\mathsf{SNR}}%

\global\long\def\Pr{\mathbb{P}}%
\global\long\def\E{\mathbb{E}}%
\global\long\def\R{\mathbb{R}}%
\global\long\def\Ncal{\mathcal{N}}%
\global\long\def\Dcal{\mathcal{D}}%
\global\long\def\ind{\mathbbm{1}}%
\global\long\def\Id{{\bI_{d}}}%

\global\long\def\var{\mathsf{var}}%
\global\long\def\cov{\mathsf{cov}}%
\global\long\def\lambmax{\lambda_{\mathsf{max}}}%
\global\long\def\lambmin{\lambda_{\mathsf{min}}}%

\global\long\def\iid{\overset{\mathsf{i.i.d.}}{\sim}}%
\global\long\def\median{\mathsf{median}}%
\global\long\def\Fcal{\mathcal{F}}%

\global\long\def\xt{{\bx_{t}}}%
\global\long\def\zt{{\bz_{t}}}%
\global\long\def\xtp{{\bx_{t+1}}}%
\global\long\def\ztp{{\bz_{t+1}}}%
\global\long\def\wt{{\bw_{t}}}%
\global\long\def\vt{{\bv_{t}}}%

\global\long\def\xs{{\bx_{s}}}%
\global\long\def\zs{{\bz_{s}}}%
\global\long\def\bGamma{\boldsymbol{\Gamma}}%

\global\long\def\Ak{{\bA^{(k)}}}%
\global\long\def\Wk{{\bW^{(k)}}}%
\global\long\def\Gammak{{\boldsymbol{\Gamma}^{(k)}}}%
\global\long\def\Yk{{\bY^{(k)}}}%

\global\long\def\Akhat{{\hat{\bA}^{(k)}}}%
\global\long\def\Wkhat{{\hat{\bW}^{(k)}}}%
\global\long\def\Gammakhat{{\hat{\boldsymbol{\Gamma}}^{(k)}}}%
\global\long\def\Ykhat{{\hat{\bY}^{(k)}}}%

\global\long\def\Ahat{{\hat{\bA}}}%
\global\long\def\What{{\hat{\bW}}}%
\global\long\def\Gammahat{{\hat{\boldsymbol{\Gamma}}}}%
\global\long\def\Yhat{{\hat{\bY}}}%
\global\long\def\Zhat{\hat{{\bZ}}}%

\global\long\def\Alhat{{\hat{\bA}^{(\ell)}}}%
\global\long\def\Wlhat{{\hat{\bW}^{(\ell)}}}%
\global\long\def\Gammalhat{{\hat{\boldsymbol{\Gamma}}^{(\ell)}}}%
\global\long\def\Ylhat{{\hat{\bY}^{(\ell)}}}%

\global\long\def\Al{{\bA^{(\ell)}}}%
\global\long\def\Wl{{\bW^{(\ell)}}}%
\global\long\def\Gammal{{\boldsymbol{\Gamma}^{(\ell)}}}%
\global\long\def\Yl{{\bY^{(\ell)}}}%

\global\long\def\Aki{{(\bA^{(k)})_{i}}}%
\global\long\def\Gammaki{{(\boldsymbol{\Gamma}^{(k)})_{i}}}%
\global\long\def\Yki{{(\bY^{(k)})_{i}}}%

\global\long\def\Gammakihat{{(\hat{\boldsymbol{\Gamma}}^{(k)})_{i}}}%
\global\long\def\Ykihat{{(\hat{\bY}^{(k)})_{i}}}%

\global\long\def\Ali{{(\bA^{(\ell)})_{i}}}%
\global\long\def\Gammali{{(\boldsymbol{\Gamma}^{(\ell)})_{i}}}%
\global\long\def\Yli{{(\bY^{(\ell)})_{i}}}%

\global\long\def\Ui{{\bU_{i}}}%
\global\long\def\Vi{{\bV_{i}}}%

\global\long\def\Hk{{\bH^{(k)}}}%
\global\long\def\Lk{{\bL^{(k)}}}%
\global\long\def\tmix{t_{\mathsf{mix}}}%

\global\long\def\dlyap{\mathsf{dlyap}}%

\global\long\def\Omegaone{\Omega_{1}}%
\global\long\def\Omegatwo{\Omega_{2}}%
\global\long\def\pk{p^{(k)}}%

\global\long\def\sumk{\sum_{k=1}^{K}}%
\global\long\def\sumtraj{\sum_{\text{trajectory}}}%
\global\long\def\sumi{\sum_{i=1}^{d}}%
\global\long\def\sumt{\sum_{t=1}^{T}}%

\global\long\def\Tsone{{T_{1}}}%
\global\long\def\Tstwo{{T_{2}}}%
\global\long\def\Tsthree{{T_{3}}}%
\global\long\def\Ts{{T}}%

\global\long\def\Ttone{{T_{\mathsf{total,1}}}}%
\global\long\def\Tttwo{{T_{\mathsf{total,2}}}}%
\global\long\def\Ttthree{{T_{\mathsf{total,3}}}}%
\global\long\def\Tt{{T_{\mathsf{total}}}}%

\global\long\def\Ttktwo{{T_{\mathsf{total,2}}^{(k)}}}%
\global\long\def\Ttkthree{{T_{\mathsf{total,3}}^{(k)}}}%
\global\long\def\Ttk{T_{\mathsf{total}}^{(k)}}%

\global\long\def\pmin{{p_{\mathsf{min}}}}%

\global\long\def\Gmax{\Gamma_{\mathsf{max}}}%
\global\long\def\Ebounded{\mathcal{E}_{\mathsf{bounded}}}%

\global\long\def\Dvec{D_{\mathsf{vec}}}%
\global\long\def\Dentry{D_{\mathsf{entry}}}%

\global\long\def\Mcal{\mathcal{M}}%
\global\long\def\Msubspace{\Mcal_{\mathsf{subspace}}}%
\global\long\def\Mclustering{\Mcal_{\mathsf{clustering}}}%
\global\long\def\Mclassification{\Mcal_{\mathsf{classification}}}%

\global\long\def\Tssubspace{T_{\mathsf{subspace}}}%
\global\long\def\Tsclustering{T_{\mathsf{clustering}}}%
\global\long\def\Tsclassification{T_{\mathsf{classification}}}%

\global\long\def\Ttsubspace{T_{\mathsf{total},\mathsf{subspace}}}%
\global\long\def\Ttclustering{T_{\mathsf{total},\mathsf{clustering}}}%
\global\long\def\Ttclassification{T_{\mathsf{total},\mathsf{classification}}}%

\global\long\def\Akm{\bA^{(k_{m})}}%
\global\long\def\Wkm{\bW^{(k_{m})}}%

\global\long\def\delAW{\delta_{A,W}}%
\global\long\def\delGY{\delta_{\Gamma,Y}}%
\global\long\def\Xm{\bX_{m}}%

\global\long\def\Gi{{\bG_{i}}}%
\global\long\def\Gihat{{\hat{\bG}_{i}}}%
\global\long\def\Hi{{\bH_{i}}}%
\global\long\def\Hihat{{\hat{\bH}_{i}}}%

\global\long\def\xmt{{\bx_{m,t}}}%
\global\long\def\xmtp{{\bx_{m,t+1}}}%
\global\long\def\wmt{{\bw_{m,t}}}%
\global\long\def\wmtp{{\bw_{m,t+1}}}%

\global\long\def\zmt{{\bz_{m,t}}}%
\global\long\def\vmt{{\bv_{m,t}}}%

\global\long\def\Mstar{{\bM_{\star}}}%
\global\long\def\yk{{\by^{(k)}}}%

\global\long\def\trunc{\mathsf{Trunc}}%
\global\long\def\Dinit{D_{\mathsf{init}}}%

\global\long\def\at{\ba_{t}}%
\global\long\def\xmsjt{\bx_{m,s_{j}(t)}}%
\global\long\def\xtilmsjt{\tilde{\bx}_{m,s_{j}(t)}}%

\global\long\def\Yktmix{{\bY_{\tmix-1}^{(k)}}}%
\global\long\def\Gammaktmix{{\boldsymbol{\Gamma}_{\tmix-1}^{(k)}}}%
\global\long\def\Yktmixi{{(\bY_{\tmix-1}^{(k)})_{i}}}%

\global\long\def\Fijtau{\boldsymbol{F}_{i,j,\tau}}%
\global\long\def\Ftilijtau{\tilde{\boldsymbol{F}}_{i,j,\tau}}%
\global\long\def\Dijtau{\boldsymbol{\Delta}_{i,j,\tau}}%
\global\long\def\Vistar{\bV_{i}^{\star}}%
\global\long\def\Uistar{\bU_{i}^{\star}}%

\global\long\def\stat{\mathsf{stat}}%
\global\long\def\Ck{{\mathcal{C}_{k}}}%
\global\long\def\Sigk{{\boldsymbol{\Sigma}^{(k)}}}%
\global\long\def\Sigl{{\boldsymbol{\Sigma}^{(\ell)}}}%
\global\long\def\cond{\kappa}%

\global\long\def\DelY{\Delta_{Y}}%
\global\long\def\Dsep{\Delta_{\Gamma,Y}}%
\global\long\def\DGY{\Delta_{\Gamma,Y}}%

\global\long\def\kwself{\kappa_{w}}%
\global\long\def\kwcross{\kappa_{w,\mathsf{cross}}}%
\global\long\def\lup{W_{\mathsf{max}}}%
\global\long\def\llb{W_{\mathsf{min}}}%
\global\long\def\Wmax{\lup}%
\global\long\def\Wmin{\llb}%

\global\long\def\kgself{\kappa_{\Gamma}}%
\global\long\def\kgcross{\kappa_{\Gamma,\mathsf{cross}}}%
\global\long\def\kg{\kappa_{\Gamma,\mathsf{cross}}}%
\global\long\def\ka{\kappa_{A}}%

\global\long\def\ut{{\bu_{t}}}%
\global\long\def\yt{{\by_{t}}}%
\global\long\def\Ftm{\mathcal{F}_{t-1}}%
\global\long\def\Ft{\mathcal{F}_{t}}%

\global\long\def\mukl{\boldsymbol{\mu}_{k,l}}%
\global\long\def\Sigkl{\boldsymbol{\Sigma}_{k,l}}%

\global\long\def\epsA{{\epsilon_{A}}}%
\global\long\def\epsW{{\epsilon_{W}}}%
\global\long\def\epsAcri{{\epsilon_{A,\mathsf{critical}}}}%
\global\long\def\epsWcri{{\epsilon_{W,\mathsf{critical}}}}%

\global\long\def\epsAcoa{{\epsilon_{A,\mathsf{coarse}}}}%
\global\long\def\epsWcoa{{\epsilon_{W,\mathsf{coarse}}}}%

\global\long\def\Tk{T^{(k)}}%
\global\long\def\Nk{N^{(k)}}%
\global\long\def\tmixk{\tmix^{(k)}}%

\global\long\def\Ck{\mathcal{C}_{k}}%
\global\long\def\Cpik{\mathcal{C}_{\pi(k)}}%
\global\long\def\Ckhat{\mathcal{C}_{\hat{k}}}%

\global\long\def\Sonetau{\mathcal{S}^{1,\tau_{1}}}%
\global\long\def\Stwotau{\mathcal{S}^{2,\tau_{2}}}%
\global\long\def\dxt{\boldsymbol{\delta}_{x,t}}%
\global\long\def\dzt{\boldsymbol{\delta}_{z,t}}%
\global\long\def\xtilt{\tilde{\bx}_{t}}%
\global\long\def\ztilt{\tilde{\bz}_{t}}%

\global\long\def\Dtau{\Delta^{\tau_{1},\tau_{2}}}%
\global\long\def\mukltil{\tilde{\boldsymbol{\mu}}_{k,\ell}}%
\global\long\def\Sigkltil{\tilde{\boldsymbol{\Sigma}}_{k,\ell}}%

\global\long\def\wtau{w^{\tau_{1},\tau_{2}}}%
\global\long\def\Dg{\Delta_{g}}%

\global\long\def\Yltmix{{\bY_{\tmix-1}^{(\ell)}}}%
\global\long\def\Gammaltmix{{\boldsymbol{\Gamma}_{\tmix-1}^{(\ell)}}}%

\global\long\def\Atilk{\tilde{\bA}^{(k)}}%
\global\long\def\Gtilk{\tilde{\boldsymbol{\Gamma}}^{(k)}}%
\global\long\def\Ytilk{\tilde{\bY}^{(k)}}%
\global\long\def\Wtilk{\tilde{\bW}^{(k)}}%

\global\long\def\Atill{\tilde{\bA}^{(l)}}%
\global\long\def\Gtill{\tilde{\boldsymbol{\Gamma}}^{(l)}}%
\global\long\def\Ytill{\tilde{\bY}^{(l)}}%
\global\long\def\Wtill{\tilde{\bW}^{(l)}}%

\global\long\def\stattau{\stat^{\tau_{1},\tau_{2}}}%
\global\long\def\stattautil{\tilde{\stat}^{\tau_{1},\tau_{2}}}%
\global\long\def\stattilg{\tilde{\stat}_{g}}%
\global\long\def\stattil{\tilde{\stat}}%

\global\long\def\statGammag{\stat_{\Gamma,g}}%
\global\long\def\statYg{\stat_{Y,g}}%

\global\long\def\xtilmt{\tilde{\bx}_{m,t}}%
\global\long\def\dmt{\boldsymbol{\delta}_{m,t}}%
\global\long\def\Ntil{\tilde{N}}%
\global\long\def\epsmix{\epsilon_{\mathsf{mix}}}%

\global\long\def\VT{{\bV_{T}}}%
\global\long\def\VTbar{{\bar{\bV}_{T}}}%
\global\long\def\Vlb{{\bV_{\mathsf{lb}}}}%
\global\long\def\Vup{{\bV_{\mathsf{up}}}}%
\global\long\def\Scal{\mathcal{S}}%
\global\long\def\sumtm{\sum_{t=0}^{T-1}}%

\global\long\def\wthat{{\hat{\bw}_{t}}}%
\global\long\def\DelA{\boldsymbol{\Delta}_{A}}%
\global\long\def\bDel{\boldsymbol{\Delta}}%
\global\long\def\wmthat{{\hat{\bw}_{m,t}}}%

\global\long\def\Dkl{D_{k,\ell}}%
\global\long\def\Fkl{F_{k,\ell}}%
\global\long\def\Dx{D_{x}}%
\global\long\def\Dklhat{\hat{D}_{k,\ell}}%
\global\long\def\DAW{\Delta_{A,W}}%

\global\long\def\XTX{\bX^{\top}\bX}%
\global\long\def\XTN{\bX^{\top}\bN}%
\global\long\def\NTX{\bN^{\top}\bX}%
\global\long\def\Nhat{\hat{\bN}}%

\global\long\def\DQ{D_{Q}}%
\global\long\def\ymt{\by_{m,t}}%
\global\long\def\umt{\bu_{m,t}}%
\global\long\def\Ajhat{\hat{\bA}^{(j)}}%
\global\long\def\Wjhat{\hat{\bW}^{(j)}}%
\global\long\def\Aj{\bA^{(j)}}%
\global\long\def\Wj{\bW^{(j)}}%

\global\long\def\Aone{\bA^{(1)}}%
\global\long\def\Atwo{\bA^{(2)}}%
\global\long\def\Wone{\bW^{(1)}}%
\global\long\def\Wtwo{\bW^{(2)}}%

\global\long\def\Stau{\mathcal{S}_{\tau}}%
\global\long\def\bSig{\boldsymbol{\Sigma}}%
\global\long\def\Akm{\bA^{(k_{m})}}%

\global\long\def\subspace{\mathsf{subspace}}%
\global\long\def\clustering{\mathsf{clustering}}%
\global\long\def\classification{\mathsf{classification}}%
\global\long\def\osf{\mathsf{o}}%

\global\long\def\pko{p_{\osf}^{(k)}}%
\global\long\def\pksubspace{p_{\subspace}^{(k)}}%
\global\long\def\pkclustering{p_{\clustering}^{(k)}}%
\global\long\def\pkclassification{p_{\classification}^{(k)}}%

\global\long\def\Tso{T_{\mathsf{o}}}%
\global\long\def\Tto{T_{\mathsf{total,o}}}%
\global\long\def\Mo{\Mcal_{\osf}}%

\global\long\def\Apik{\bA^{(\pi(k))}}%
\global\long\def\Wpik{\bW^{(\pi(k))}}%

\global\long\def\Omegatilde{\tilde{\Omega}}%

\global\long\def\ximt{{\bf \xi}_{m,t}}%
\global\long\def\DelW{\bDel_{W}}%
\global\long\def\DelG{\bDel_{\Gamma}}%


\global\long\def\Omegaj{\Omega_j}%
\global\long\def\Omegagj{\Omega_{g,j}}%
\global\long\def\Omegagone{\Omega_{g,1}}%
\global\long\def\Omegagtwo{\Omega_{g,2}}%


\global\long\def\hmij{\bh_{m,i,j}}%
\global\long\def\hmione{\bh_{m,i,1}}%
\global\long\def\hmitwo{\bh_{m,i,2}}%

\global\long\def\gmij{\bg_{m,i,j}}%
\global\long\def\gmione{\bg_{m,i,1}}%
\global\long\def\gmitwo{\bg_{m,i,2}}%

\global\long\def\hmigj{\bh_{m,i,g,j}}%
\global\long\def\hmigone{\bh_{m,i,g,1}}%
\global\long\def\hmigtwo{\bh_{m,i,g,2}}%

\global\long\def\gmigj{\bg_{m,i,g,j}}%
\global\long\def\gmigone{\bg_{m,i,g,1}}%
\global\long\def\gmigtwo{\bg_{m,i,g,2}}%

\global\long\def\hnigj{\bh_{n,i,g,j}}%
\global\long\def\hnigone{\bh_{n,i,g,1}}%
\global\long\def\hnigtwo{\bh_{n,i,g,2}}%

\global\long\def\gnigj{\bg_{n,i,g,j}}%
\global\long\def\gnigone{\bg_{n,i,g,1}}%
\global\long\def\gnigtwo{\bg_{n,i,g,2}}%

\global\long\def\gnij{\bg_{n,i,j}}%
\global\long\def\gnione{\bg_{n,i,1}}%
\global\long\def\gnitwo{\bg_{n,i,2}}%


\global\long\def\defhmij{\frac{1}{|\Omegaj|}\sum_{t\in \Omegaj}(\xmt)_i \ \xmt}%
\global\long\def\defhmione{\frac{1}{|\Omegaone|}\sum_{t\in \Omegaone}(\xmt)_i \ \xmt}%
\global\long\def\defhmitwo{\frac{1}{|\Omegatwo|}\sum_{t\in \Omegatwo}(\xmt)_i \ \xmt}%

\global\long\def\defgmij{\frac{1}{|\Omegaj|}\sum_{t\in \Omegaj}(\xmtp)_i \ \xmt}%
\global\long\def\defgmione{\frac{1}{|\Omegaone|}\sum_{t\in \Omegaone}(\xmtp)_i \ \xmt}%
\global\long\def\defgmitwo{\frac{1}{|\Omegatwo|}\sum_{t\in \Omegatwo}(\xmtp)_i \ \xmt}%

\global\long\def\defhmigj{\frac{1}{|\Omegagj|}\sum_{t\in \Omegagj}(\xmt)_i \ \xmt}%
\global\long\def\defhmigone{\frac{1}{|\Omegagone|}\sum_{t\in \Omegagone}(\xmt)_i \ \xmt}%
\global\long\def\defhmigtwo{\frac{1}{|\Omegagtwo|}\sum_{t\in \Omegagtwo}(\xmt)_i \ \xmt}%

\global\long\def\defgmigj{\frac{1}{|\Omegagj|}\sum_{t\in \Omegagj}(\xmtp)_i \ \xmt}%
\global\long\def\defgmigone{\frac{1}{|\Omegagone|}\sum_{t\in \Omegagone}(\xmtp)_i \ \xmt}%
\global\long\def\defgmigtwo{\frac{1}{|\Omegagtwo|}\sum_{t\in \Omegagtwo}(\xmtp)_i \ \xmt}%

\global\long\def\xnt{\bx_{n,t}}%
\global\long\def\xntp{\bx_{n,t+1}}%

\global\long\def\defhnigj{\frac{1}{|\Omegagj|}\sum_{t\in \Omegagj}(\xnt)_i \ \xnt}%
\global\long\def\defhnigone{\frac{1}{|\Omegagone|}\sum_{t\in \Omegagone}(\xnt)_i \ \xnt}%
\global\long\def\defhnigtwo{\frac{1}{|\Omegagtwo|}\sum_{t\in \Omegagtwo}(\xnt)_i \ \xnt}%

\global\long\def\defgnigj{\frac{1}{|\Omegagj|}\sum_{t\in \Omegagj}(\xntp)_i \ \xnt}%
\global\long\def\defgnigone{\frac{1}{|\Omegagone|}\sum_{t\in \Omegagone}(\xntp)_i \ \xnt}%
\global\long\def\defgnigtwo{\frac{1}{|\Omegagtwo|}\sum_{t\in \Omegagtwo}(\xntp)_i \ \xnt}%

\global\long\def\km{k_m}%
\global\long\def\kn{k_n}%

\global\long\def\Ykm{\bY^{(\km)}}%
\global\long\def\Ykn{\bY^{(\kn)}}%
\global\long\def\Gammakm{\bGamma^{(\km)}}%
\global\long\def\Gammakn{\bGamma^{(\kn)}}%

\begin{abstract}
We study the problem of learning a mixture of multiple linear dynamical
systems (LDSs) from unlabeled short sample trajectories, each generated
by one of the LDS models. Despite the wide applicability of mixture
models for time-series data, learning algorithms that come with end-to-end
performance guarantees are largely absent from existing literature.
There are multiple sources of technical challenges, including but
not limited to (1)~the presence of latent variables (i.e.~the unknown
labels of trajectories); (2)~the possibility that the sample trajectories
might have lengths much smaller than the dimension $d$ of the LDS
models; and~(3) the complicated temporal dependence inherent to time-series
data. To tackle these challenges, we develop a two-stage meta-algorithm,
which is guaranteed to efficiently recover each ground-truth LDS model
up to error $\tilde{O}(\sqrt{d/T})$, where $T$ is the total sample
size. We validate our theoretical studies with numerical experiments,
confirming the efficacy of the proposed algorithm. 
\end{abstract}

\tableofcontents{}

\section{Introduction \label{sec:formulation}}

Imagine that we are asked to learn multiple linear dynamical systems
(LDSs) from a mixture of unlabeled sample trajectories --- namely,
each sample trajectory is generated by one of the LDSs of interest,
but we have no idea which system it is. To set the stage and facilitate
discussion, recall that in a classical LDS, one might observe a sample
trajectory $\{\xt\}_{0\le t\le T}$ generated by an LDS obeying $\xtp=\bA\xt+\wt,$
where $\bA\in\R^{d\times d}$ determines the system dynamics in the
noiseless case, and $\{\wt\}_{t\geq0}$ denote independent zero-mean
noise vectors with covariance $\cov(\wt)=\bW\succ\mathbf{0}$.  The
mixed LDSs setting considered herein extends classical LDSs by allowing
for mixed measurements as described below; see Figure \ref{fig:setting}
for a visualization of the scenario. 
\begin{itemize}
\item \emph{Multiple linear systems}. Suppose that there are $K$ different
LDSs as represented by $\{(\Ak,\Wk)\}_{1\le k\le K}$, where $\Ak\in\mathbb{R}^{d\times d}$
and $\Wk\in\R^{d\times d}$ represent the state transition matrix
and noise covariance matrix of the $k$-th LDS, respectively. Here
and throughout, we shall refer to $(\Ak,\Wk)$ as the system matrix
for the $k$-th LDS. We only assume that $(\Ak,\Wk)\neq(\Al,\Wl)$ for
any $k\neq\ell$, whereas $\Ak=\Al$ or $\Wk=\Wl$ is allowed.
\item \emph{Mixed sample trajectories}. We collect a total number of $M$
unlabeled trajectories from these LDSs. More specifically,
the $m$-th sample trajectory is drawn from one of the LDSs in the
following manner: set $(\bm{A},\bm{W})$ to be $(\Ak,\Wk)$ for some
$1\leq k\leq K$, and generate a trajectory of length $T_{m}$ obeying
\begin{equation}
\xtp=\bA\xt+\wt,\quad\text{where the }\bm{w}_{t}\text{'s are i.i.d., }\E[\wt]=\boldsymbol{0},\ \cov(\wt)=\bW\succ\boldsymbol{0}.\label{eq:setting-1}
\end{equation}
Note, however, that the label $k$ associated with each sample trajectory
is a latent variable not revealed to us, resulting in a mixture of
unlabeled trajectories. The current paper focuses on the case where
the length of each trajectory is somewhat short, making it infeasible
to estimate the system matrix from a single trajectory. 
\item \emph{Goal}. The aim is to jointly learn the system matrices $\{(\Ak,\Wk)\}_{1\le k\le K}$
from the mixture of sample trajectories. In particular, we seek to
accomplish this task in a sample-efficient manner, where the total
sample size is defined to be the aggregate trajectory length $\sum_{m=1}^{M}T_{m}$.
\end{itemize}
The mixed LDSs setting described above is motivated by many real-world
scenarios where a single time-series model is insufficient to capture
the complex and heterogeneous patterns in temporal data. In what follows,
we single out a few potential applications, with a longer list provided
in Section~\ref{sec:related_work}. 
\begin{itemize}
\item In psychology \cite{bulteel2016clustering}, researchers collect multiple
time-series trajectories (e.g.~depression-related symptoms over a
period of time) from different patients. Fitting this data with multi-modal
LDSs (instead of a single model) not only achieves better fitting
performance, but also helps identify subgroups of the persons, which
further inspires interpretations of the results and tailored treatments
for patients from different subgroups. 
\item Another example concerns an automobile sensor dataset, which consists
of a single continuous (but possibly time-varying) trajectory of measurements
from the sensors \cite{hallac2017toeplitz}. Through segmentation
of the trajectory, clustering of the short pieces, and learning within
each cluster, one can discover, and obtain meaningful interpretations
for, a small number of key driving modes (such as ``driving straight'',
``slowing down'', ``turning''). 
\item In robotics, a robot might collect data samples while navigating a
complex environment with varying terrains (e.g.~grass, sand, carpets,
rocks) \cite{brunskill2009provably}. With mixture modeling in place,
one can hope to jointly learn different dynamical systems underlying
various terrains, from a continuous yet time-varying trajectory acquired
by the robot. 
\end{itemize}
\begin{figure}
\begin{centering}
\includegraphics[width=0.7\textwidth]{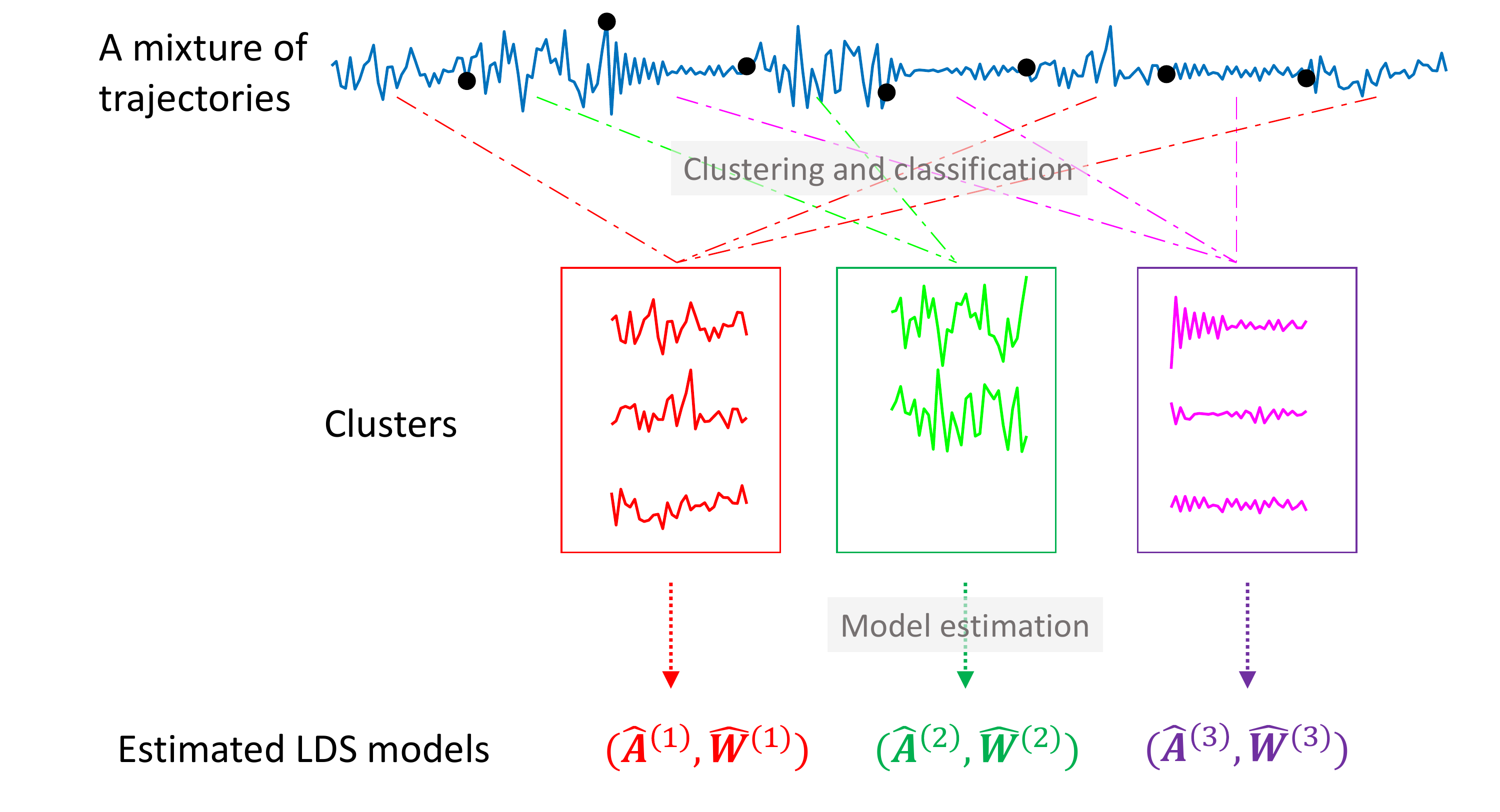}
\par\end{centering}
\caption{\label{fig:setting}A high-level visualization of the mixed LDSs formulation,
and the algorithmic idea of combining clustering, classification,
and model estimation. Here, we consider the special case where the
multiple short trajectories come from the segments of a single continuous
trajectory. The black dots within the continuous trajectory represent
the time steps when the latent variable (i.e.~the unknown label)
changes.}
\end{figure}

While there is no shortage of potential applications, a mixture of
linear dynamical systems is far more challenging to learn compared
to the classical setting with a single LDS. In particular, the presence
of the \emph{latent variables} --- namely, the unknown labels of
the sample trajectories --- significantly complicates matters. One
straightforward idea is to learn a coarse model for each trajectory,
followed by proper clustering of these coarse models (to be utilized
for refined model estimation); however, this idea becomes infeasible
in the high-dimensional setting unless all sample trajectories are
sufficiently long. Another popular approach is to alternate between
model estimation and clustering of trajectories, based on, say, the
expectation-maximization (EM) algorithm; unfortunately, there is no
theoretical support for such EM-type algorithms, and we cannot preclude
the possibilities that the algorithms get stuck at undesired local
optima. The lack of theoretical guarantees in prior literature motivates
us to come up with algorithms that enjoy provable performance guarantees.
Finally, the present paper is also inspired by the recent progress
in \emph{meta-learning for mixed linear regression} \cite{kong2020meta,kong2020robust},
where the goal is to learn multiple linear models from a mixture of
independent samples; note, however, that the temporal dependence underlying
time-series data in our case poses substantial challenges and calls
for the development of new algorithmic ideas. 

\subsection{Main contributions \label{subsec:contributions}}

In this work, we take an important step towards guaranteed learning
of mixed linear dynamical systems, focusing on algorithm design that
comes with end-to-end theoretical guarantees. In particular, we propose
a two-stage meta-algorithm to tackle the challenge of mixed LDSs:
\begin{enumerate}
\item \emph{Coarse estimation}: perform a coarse-level clustering of the
unlabeled sample trajectories (assisted by dimension reduction), and
compute initial model estimation for each cluster; 
\item \emph{Refinement}: classify additional trajectories (and add each
of them to the corresponding cluster) based on the above coarse model
estimates, followed by refined model estimation with the updated clusters. 
\end{enumerate}
This two-stage meta approach, as well as the specific methods for
individual steps, will be elucidated in Section~\ref{sec:algorithms}. 

Encouragingly, assuming that the noise vectors $\{\wmt\}$ are independent
Gaussian random vectors, the proposed two-stage algorithm is not only
computationally efficient, but also guaranteed to succeed in the presence
of a polynomial sample size. Informally, our algorithm achieves exact
clustering/classification of the sample trajectories as well as faithful
model estimation, under the following conditions (with a focus on
the dependency on the dimension $d$): (1) each short trajectory length
$T_{m}$ is allowed to be much smaller than $d$; (2) the total trajectory
length of each stage is linear in $d$ (up to logarithmic factors);
(3) to achieve a final model estimation error $\epsilon\rightarrow0$
(in the spectral norm), it suffices to have a total trajectory length
of order $d/\epsilon^{2}$ for each LDS model. See Section \ref{sec:main_results}
(in particular, Corollary \ref{cor:simplify}) for the precise statements
of our main results, which will also be validated numerically.

It is worth noting that, although we focus on mixed LDSs for concreteness,
we will make clear that the proposed modular algorithm is fairly flexible
and can be adapted to learning mixtures of other time-series models,
as long as certain technical conditions are satisfied; see Remark~\ref{rem:beyond_LDS}
at the end of Section \ref{sec:algorithms} for a detailed discussion.

\subsection{Notation}

Throughout this paper, vectors and matrices are represented by boldface
letters. Let $\Id$ be the $d\times d$ identity matrix. For a vector
$\bx\in\R^{d}$, we define $(\bx)_{i}\in\R$ as the $i$-th entry
of $\bx$, and $\|\bx\|_{2}$ as its $\ell_{2}$ norm; for a matrix
$\bX\in\R^{m\times n}$, we define $(\bX)_{i}\in\R^{n}$ as the transpose
of the $i$-th row of $\bX$, and $\|\bX\|$ (resp.~$\|\bX\|_{\Frm}$)
as its spectral (resp.~Frobenius) norm. For a square matrix $\bX\in\R^{d\times d}$,
we denote its eigenvalues as $\{\lambda_{i}(\bX)\}_{1\le i\le d}$,
and if $\bX$ is symmetric, we sort its eigenvalues (which are all
real) in descending order, with the maximum and minimum being $\lambmax(\bX)\coloneqq\lambda_{1}(\bX)$
and $\lambmin(\bX)\coloneqq\lambda_{d}(\bX)$, respectively; if in
addition $\bX$ is positive definite, we denote its condition number
as $\cond(\bX)\coloneqq\lambmax(\bX)/\lambmin(\bX)\ge1$. Let $\vc(\cdot)$
denote the vectorization of a matrix. For two matrices $\bA,\bB$,
let $\bA\otimes\bB$ denote their Kronecker product. If $\bA=[A_{ij}]_{1\leq i\leq m,1\leq j\leq n}$
and $\bB=[B_{ij}]_{1\leq i\leq m,1\leq j\leq n}$ are of the same
dimension, we denote by $\langle\bA,\bB\rangle\coloneqq\sum_{i=1}^{m}\sum_{j=1}^{n}A_{ij}B_{ij}$
their inner product. Given vectors $\bx_{1},\dots,\bx_{n}\in\R^{d}$
where $n<d$, let $\spn\{\bx_{i},1\le i\le N\}\in\R^{d\times n}$
represent the subspace spanned by these vectors.

Throughout this work, we always use the superscript ``$(k)$'' to
indicate ``the $k$-th model'', as in $\Ak$ and $\Wk$; this is
to be distinguished from the superscript ``$k$'' without the parentheses,
which simply means the power of $k$. For a discrete set $\Omega$,
we denote by $|\Omega|$ its cardinality. Define $\ind(\mathcal{E})$
to be the indicator function, which takes value $1$ if the event
$\mathcal{E}$ happens, and $0$ otherwise. Given a sequence of real
numbers $\{x_{i}\}_{1\le i\le N}$, we denote its median as $\median\{x_{i},1\le i\le N\}$.
Let $a_{n}\lesssim b_{n}$ and $a_{n}=O(b_{n})$ indicate that $a_{n}\le C_{0}b_{n}$
for all $n=1,2,\dots$, where $C_{0}>0$ is some universal constant;
moreover, $a_{n}\gtrsim b_{n}$ is equivalent to $b_{n}\lesssim a_{n}$,
and $a_{n}\asymp b_{n}$ indicates that $a_{n}\lesssim b_{n}$ and
$a_{n}\gtrsim b_{n}$ hold simultaneously. We shall also let $\poly(n)$
denote some polynomial in $n$ of a constant degree.
For a positive integer $n$, we denote $[n]\coloneqq \{1, 2, \dots, n\}$.

\section{Algorithms \label{sec:algorithms}}

We propose a two-stage paradigm for solving the mixed LDSs problem,
as summarized in Algorithm~\ref{alg:overall}. It consists of several
subroutines as described in Algorithms~\ref{alg:subspace}--\ref{alg:classification},
which will be explained in detail in the remainder of this section.
Note that Algorithm~\ref{alg:overall} is stated in a modular manner,
and one might replace certain subroutines by alternative schemes in
order to handle different settings and model assumptions. 

Before delving into the algorithmic details, we introduce some
additional notation and assumptions that will be useful for our presentation,
without much loss of generality. To begin with, we augment the notation
for each sample trajectory with its trajectory index; that is, for
each $1\le m\le M$, the $m$-th trajectory --- denoted by $\Xm\coloneqq\{\xmt\}_{0\le t\le T_{m}}$
--- starts with some initial state $\bx_{m,0}\in\R^{d}$, and evolves
according to the $k_{m}$-th LDS for some \emph{unknown} label $1\le k_{m}\le K$
such that
\begin{equation}
\xmtp=\bA^{(k_{m})}\xmt+\wmt,\quad\text{where the }\wmt\text{'s are i.i.d., }\E[\wmt]=\boldsymbol{0},\ \cov(\wmt)=\bW^{(k_{m})}\succ\boldsymbol{0}\label{eq:setting}
\end{equation}
for all $0\le t\le T_{m}-1$. Next, we divide the $M$ sample trajectories
$\{\Xm\}_{1\le m\le M}$ in hand into three disjoint subsets $\Msubspace,\Mclustering,\Mclassification$
satisfying 
\[
\Msubspace\cup\Mclustering\cup\Mclassification=\{1,2,\dots,M\},
\]
where each subset of samples will be employed to perform one subroutine.
We assume that all trajectories within each subset have the same length,
namely,
\begin{equation}
T_{m}=\begin{cases}
\Tssubspace & \text{if}\quad m\in\Msubspace,\\
\Tsclustering & \text{if}\quad m\in\Mclustering,\\
\Tsclassification & \text{if}\quad m\in\Mclassification.
\end{cases}\label{eq:defn-Tm}
\end{equation}
Finally, we assume that $K\le d$, so that performing subspace estimation
in Algorithm~\ref{alg:overall} will be helpful (otherwise one might
simply eliminate this step). The interested readers are referred to
Appendix \ref{sec:extensions} for discussions of some potential extensions
of our algorithms (e.g.~adapting to the case where the trajectories
within a subset have different lengths, or the case where certain
parameters are \emph{a priori} unknown). 

\begin{algorithm}[tbp]
\DontPrintSemicolon
\caption{A two-stage algorithm for mixed LDSs} \label{alg:overall}
{\bf Input:} {$M$ short trajectories $\{\Xm\}_{1\le m\le M}$ (where $\Xm=\{\xmt\}_{0\le t\le T_{m}}$); parameters $\tau, G$}. \\
{\color{blue}\tcp{Stage 1: coarse estimation.}} 
Run Algorithm~\ref{alg:subspace} with $\{\Xm\}_{m \in \Msubspace}$ to obtain subspaces $\{\Vi,\Ui\}_{1\le i\le d}$ . \label{line:subspace} \\
Run Algorithm~\ref{alg:clustering} with $\{\Xm\}_{m \in \Mclustering}$, $\{\Vi, \Ui\}_{1 \le i \le d}$, $\tau$, $G$, to obtain clusters $\{\Ck\}_{1\le k\le K}$. \label{line:clustering} \\
Run Algorithm~\ref{alg:LS} with $\{\Ck\}_{1\le k\le K}$ to obtain coarse models $\{\Akhat,\Wkhat\}_{1\le k \le K}$. \label{line:LS1} \\
{\color{blue}\tcp{Stage 2: refinement.}} 
Run Algorithm~\ref{alg:classification} with $\{\Xm\}_{m \in \Mclassification}, \{\Akhat,\Wkhat\}_{1\le k \le K}, \{\Ck\}_{1 \le k \le K}$, to update clusters $\{\Ck\}$. \label{line:classification} \\
Run Algorithm~\ref{alg:LS} with $\{\Ck\}_{1\le k\le K}$ to obtain refined models $\{\Akhat,\Wkhat\}_{1\le k \le K}$. \label{line:LS2} \\
{\bf Output:} {final model estimation $\{\Akhat,\Wkhat\}_{1 \le k \le K}$ and clusters $\{\Ck\}_{1 \le k \le K}$}.
\end{algorithm}

\subsection{Preliminary facts}

We first introduce some preliminary background on the autocovariance
structures and mixing property of linear dynamical systems, which
form the basis of our algorithms for subspace estimation and clustering
of trajectories.

\paragraph{Stationary covariance matrices.}

Consider first the basic LDS with $\xtp=\bA\xt+\wt$. If $\E[\xt]=\mathbf{0},\cov(\xt)=\boldsymbol{\Gamma}$
and $\E[\wt]=\boldsymbol{0},\cov(\wt)=\bW$, then it follows that
\[
\E[\xtp]=\mathbf{0}\qquad\text{and}\qquad\cov(\xtp)=\bA\cdot\cov(\xt)\cdot\bA^{\top}+\cov(\wt)=\bA\bGamma\bA^{\top}+\bW.
\]
Under certain assumption on stability, this leads to the \emph{order-0
stationary autocovariance matrix} $\bGamma(\bA,\bW)$ of the LDS model
$(\bA,\bW)$, defined as follows \cite{kailath2000linear}:
\begin{align}
\boldsymbol{\Gamma}(\bA,\bW) & \coloneqq\E\big[\xt\xt^{\top}|\bA,\bW\big]=\bA\cdot\boldsymbol{\Gamma}(\bA,\bW)\cdot\bA^{\top}+\bW=\sum_{t=0}^{\infty}\bA^{t}\bW(\bA^{t})^{\top}.\label{eq:def_Gamma}
\end{align}
Furthermore, define the \emph{order-1 stationary autocovariance matrix}
as
\begin{equation}
\bY(\bA,\bW)\coloneqq\E\big[\xtp\xt^{\top}|\bA,\bW\big]=\E\big[(\bA\xt+\wt)\xt^{\top}|\bA,\bW\big]=\bA\cdot\boldsymbol{\Gamma}(\bA,\bW).\label{eq:def_Y}
\end{equation}
For the mixed LDSs case (\ref{eq:setting}) with $K$ models, we abbreviate
\begin{equation}
\Gammak\coloneqq\boldsymbol{\Gamma}(\Ak,\Wk),\quad\Yk\coloneqq\bY(\Ak,\Wk),\quad1\le k\le K.\label{eq:def_Gammak_Yk}
\end{equation}
In turn, the definitions of $\{\Gammak,\Yk\}$ suggest that we can
recover $\Ak,\Wk$ from $\Gammak,\Yk$ as follows:
\begin{align}
\Ak & =\Yk\Gammak^{-1},\quad\Wk=\Gammak-\Ak\Gammak\Ak^{\top}.\label{eq:GY_to_AW}
\end{align}

\paragraph{Mixing.}

Expanding the LDS recursion $\xtp=\bA\xt+\wt$ with $\cov(\wt)=\bW$,
we have\texttt{
\begin{equation}
\bx_{t+s}=\bA^{s}\xt+\sum_{i=1}^{s}\bA^{i-1}\bw_{t+s-i}\label{eq:x-ts-expansion}
\end{equation}
}for all $s\ge1$. If $\bA$ satisfies certain assumptions regarding
stability and if $s$ is larger than
certain mixing time of the LDS, then the first term on the right-hand
side of \eqref{eq:x-ts-expansion} approaches zero, while the second
term is independent of the history up to time $t$, with covariance
close to the stationary autocovariance matrix $\bGamma(\bA,\bW)$.
This suggests that, for two samples within the same trajectory that
are sufficiently far apart, we can treat them as being (almost) independent
of each other; this simple fact shall inspire our algorithmic design
and streamline statistical analysis later on. It is noteworthy that
the proposed algorithms do not require prior knowledge about the mixing
times of the LDS models.

\subsection{Subspace estimation \label{subsec:subspace}}

\paragraph{Procedure.}

Recall that the notation $\Gammaki\in\R^{d}$ (resp. $\Yki$) represents
the transpose of the $i$-th row of $\Gammak$ (resp. $\Yk$) defined
in \eqref{eq:def_Gammak_Yk}. With this set of notation in place,
let us define the following subspaces:
\begin{equation}
\Vistar\coloneqq\spn\Big\{\Gammaki,1\le k\le K\Big\},\quad\Uistar\coloneqq\spn\Big\{\Yki,1\le k\le K\Big\},\quad1\le i\le d.\label{eq:defn-Vi-Ui-star}
\end{equation}
It is easily seen from the construction that each of these subspaces
has rank at most $K$. 

As it turns out, the collection of $2d$ subspaces defined in \eqref{eq:defn-Vi-Ui-star}
provides crucial low-dimensional information about the linear dynamical
systems of interest. This motivates us to develop a data-driven method
to estimate these subspaces, which will in turn allow for proper dimension
reduction in subsequent steps. Towards this end, we propose to employ
a spectral method for subspace estimation using sample trajectories
in $\Msubspace$: 

\begin{itemize}
\item[(i)] divide $\{0,1,\dots,\Tssubspace\}$ into
four segments of the same size, and denote the 2nd (resp.~4th) segment
as $\Omegaone$ (resp.~$\Omegatwo$); 
\item[(ii)] for each $m\in\Msubspace, 1 \le i \le d, j \in \{1,2\}$, compute 
\begin{equation}
\hmij \coloneqq \defhmij, \quad \gmij \coloneqq \defgmij, \label{eq:hg};
\end{equation}
\item[(iii)] for each $1\le i\le d$,
compute the following matrices
\begin{equation}
\Hihat \coloneqq\frac{1}{|\Msubspace|}\sum_{m\in\Msubspace} \hmione \hmitwo^{\top}, \quad
\Gihat \coloneqq\frac{1}{|\Msubspace|}\sum_{m\in\Msubspace} \gmione \gmitwo^{\top}, \label{eq:Hihat_Gihat}
\end{equation}
and let $\Vi\in\R^{d\times K}$ (resp.~$\Ui$) be the top-$K$ eigenspace
of $\Hihat+\Hihat^{\top}$ (resp.~$\Gihat+\Gihat^{\top}$). \end{itemize}
The output $\{\Vi,\Ui\}_{1\le i\le d}$ will then serve as our estimate
of $\{\Vistar,\Uistar\}_{1\le i\le d}$. This spectral approach is
summarized in Algorithm~\ref{alg:subspace}. 

\paragraph{Rationale.}

According to the mixing property of LDS, if $\Tssubspace$ is larger
than some appropriately defined mixing time, then each sample trajectory
in $\Msubspace$ will mix sufficiently and nearly reach stationarity
(when constrained to $t\in\Omegaone\cup\Omegatwo$). In this case,
it is easy to check that
\[
\E\big[\hmij]\approx(\bGamma^{(k_{m})})_{i},\quad\E\big[\gmij]\approx(\bY^{(k_{m})})_{i},\quad1\le i\le d,\quad j \in \{1,2\}.
\]
Moreover, the samples in $\Omegaone$ are nearly independent of those
in $\Omegatwo$ as long as the spacing between them exceeds the mixing
time, and therefore, 
\begin{subequations}
\label{eq:Hi_Gi}
\begin{align}
\E\Big[\Hihat\Big] & \approx\frac{1}{|\Msubspace|}\sum_{m\in\Msubspace}(\bGamma^{(k_{m})})_{i}(\bGamma^{(k_{m})})_{i}^{\top}=\sumk\pksubspace\Gammaki\Gammaki^{\top}\eqqcolon\Hi,\label{eq:Hi}\\
\E\Big[\Gihat\Big] & \approx\frac{1}{|\Msubspace|}\sum_{m\in\Msubspace}(\bY^{(k_{m})})_{i}(\bY^{(k_{m})})_{i}^{\top}=\sumk\pksubspace\Yki\Yki^{^{\top}}\eqqcolon\Gi,\label{eq:Gi}
\end{align}
\end{subequations}
where $\pksubspace$ denotes the fraction of sample trajectories generated
by the $k$-th model, namely, 
\begin{equation}
\pksubspace\coloneqq\frac{1}{|\Msubspace|}\sum_{m\in\Msubspace}\ind(k_{m}=k),\quad1\le k\le K.\label{eq:defn-p-subspace}
\end{equation}
As a consequence, if $\Tssubspace$ and $|\Msubspace|$ are both sufficiently
large, then one might expect $\Hihat$ (resp.~$\Gihat$) to be a
reasonably good approximation of $\Hi$ (resp.~$\Gi$), the latter
of which has rank at most $K$ and has $\Vistar$ (resp.~$\Uistar$)
as its eigenspace. All this motivates us to compute the rank-$K$
eigenspaces of $\Hihat+\Hihat^{\top}$ and $\Gihat+\Gihat^{\top}$
in Algorithm~\ref{alg:subspace}.

\begin{algorithm}[tbp]
\DontPrintSemicolon
\caption{Subspace estimation} \label{alg:subspace}
{\bf Input:} {short trajectories $\{\Xm\}_{m \in \Msubspace}$, where $\Xm=\{\xmt\}_{0\le t\le\Tssubspace}$.} \\
Let $N\gets\lfloor\Tssubspace/4\rfloor$, and $\Omegaone\gets\{N+j,1\le j\le N\},\Omegatwo\gets\{3N+j,1\le j\le N\}$. \\
\For{$(m,i,j)\in\Msubspace \times [d] \times [2]$}{
Compute
\begin{equation*}
\hmij \gets \defhmij, \quad \gmij \gets \defgmij.
\end{equation*} \\
}
\For{$i = 1, \dots, d$}{
Compute
\begin{equation*}
\Hihat \gets\frac{1}{|\Msubspace|}\sum_{m\in\Msubspace} \hmione \hmitwo^{\top}, \quad
\Gihat \gets\frac{1}{|\Msubspace|}\sum_{m\in\Msubspace} \gmione \gmitwo^{\top}, 
\end{equation*} \\
Let $\Vi \in \R^{d\times K}$ (resp.~$\Ui$) be the top-$K$ eigenspace of $\Hihat+\Hihat^{\top}$ (resp.~$\Gihat+\Gihat^{\top}$). \\
}
{\bf Output:} subspaces {$\{\Vi, \Ui\}_{1 \le i \le d}$}.
\end{algorithm}

\subsection{Clustering \label{subsec:clustering}}

This step seeks to divide the sample trajectories in $\Mclustering$
into $K$ clusters (albeit not perfectly), such that the trajectories
in each cluster are primarily generated by the same LDS model. We
intend to achieve this by performing pairwise comparisons of the autocovariance
matrices associated with the sample trajectories. 

\paragraph{Key observation.}

Even though $(\Ak,\Wk)\neq(\Al,\Wl)$, it is indeed possible that
$\Gammak=\Gammal$ or $\Yk=\Yl$. Therefore, in order to differentiate
sample trajectories generated by different systems based on $\boldsymbol{\Gamma}(\bA,\bW)$
and $\boldsymbol{Y}(\bA,\bW)$, it is important to ensure separation
of $(\Gammak,\Yk)$ and $(\Gammal,\Yl)$ when $k\neq l$, which can
be guaranteed by the following fact.
\begin{fact}
\label{fact:separation}If $(\Ak,\Wk)\neq(\Al,\Wl)$, then we have
either $\Gammak\neq\Gammal$ or $\Yk\neq\Yl$ (or both).
\end{fact}
\begin{proof}
We prove this by contradiction. Suppose instead that $\Gammak=\Gammal$
and $\Yk=\Yl$, then (\ref{eq:GY_to_AW}) yields
\begin{align*}
\Ak & =\Yk\Gammak^{-1}=\Yl\Gammal^{-1}=\Al,\quad\text{and}\\
\Wk & =\Gammak-\Ak\Gammak\Ak^{\top}=\Gammal-\Al\Gammal\Al^{\top}=\Wl,
\end{align*}
which is contradictory to the assumption that $(\Ak,\Wk)\neq(\Al,\Wl)$.
\end{proof}
\begin{algorithm}[tbp]
\DontPrintSemicolon
\caption{Clustering} \label{alg:clustering}
{\bf Input:} short trajectories $\{\Xm\}_{m \in \Mclustering}$, where $\Xm=\{\xmt\}_{0\le t\le\Tsclustering}$; subspaces $\{\Vi, \Ui\}_{1 \le i \le d}$; testing threshold $\tau$; number of copies $G$. \\
Let $N\gets\lfloor\Tsclustering/4G\rfloor$, and  
$$
\Omega_{g,1}\gets\Big\{(4g-3)N+j,1\le j\le N\Big\}, \quad \Omega_{g,2}\gets\Big\{(4g-1)N+j,1\le j\le N\Big\}, \quad 1 \le g \le G.
$$ \\
\For{$(m,i,g,j) \in \Mclustering \times [d] \times [G] \times [2]$}{
Compute
\begin{equation*}
\hmigj \gets \defhmigj, \quad \gmigj \gets \defgmigj.
\end{equation*} \\
}
{\color{blue}\tcp{Compute the similarity matrix  $\bS$:}} 
\For{$(m, n) \in \Mclustering \times \Mclustering$}{ 
\For{$g = 1, \dots, G$}{
Compute
\begin{subequations}
\label{eq:stat_g}
\begin{align}
\statGammag&\gets\sum_{i=1}^{d}\Big\langle\Vi^{\top}\big(\hmigone-\hnigone\big), \Vi^{\top}\big(\hmigtwo-\hnigtwo\big)\Big\rangle, \label{eq:statGammag} \\
\statYg&\gets\sum_{i=1}^{d}\Big\langle\Ui^{\top}\big(\gmigone-\gnigone\big), \Ui^{\top}\big(\gmigtwo-\gnigtwo\big)\Big\rangle,  \label{eq:statYg}
\end{align}
\end{subequations} \\
}
Set
$
S_{m,n} \gets \ind\Big(\mathsf{median}\big\{\statGammag, 1 \le g \le G\big\} + \mathsf{median}\big\{\statYg, 1 \le g \le G\big\} \le \tau \Big)
$. \\
}
Divide $\Mclustering$ into $K$ clusters $\{\Ck\}_{1\le k\le K}$ according to $\{S_{m,n}\}_{m,n\in \Mclustering}$. \\
{\bf Output:} clusters $\{\Ck\}_{1 \le k \le K}$.
\end{algorithm}

\paragraph{Idea.}

Let us compare a pair of sample trajectories $\{\xt\}_{0\le t\le\Tsclustering}$
and $\{\zt\}_{0\le t\le\Tsclustering}$, where $\{\xt\}$ is generated
by the system $(\Ak,\Wk)$ and $\{\zt\}$ by the system $(\Al,\Wl)$.
In order to determine whether $k=\ell$, we propose to estimate the
quantity $\|\Gammak-\Gammal\|_{\Frm}^{2}+\|\Yk-\Yl\|_{\Frm}^{2}$
using the data samples $\{\xt\}$ and $\{\zt\}$, which is expected
to be small (resp.~large) if $k=\ell$ (resp.~$k\neq\ell$). To
do so, let us divide $\{0,1,\dots,\Tsclustering\}$ evenly into four
segments, and denote by $\Omegaone$ (resp.~$\Omegatwo$) the 2nd
(resp.~4th) segment, akin to what we have done in the previous step.
Observe that
\[
\Ui^{\top}\E\Big[\big((\xtp)_{i}\xt-(\ztp)_{i}\zt\big)\Big]\approx\Ui^{\top}\big(\Yki-\Yli\big)
\]
for all $1\le i\le d$ and $t\in\Omegaone\cup\Omegatwo$. Assuming
sufficient mixing and utilizing (near) statistical independence due
to sample splitting, we might resort to the following statistic:
\begin{equation}
\stat_{Y}\coloneqq\sum_{i=1}^{d}\Big\langle\Ui^{\top}\frac{1}{|\Omegaone|}\sum_{t\in\Omegaone}\big((\xtp)_{i}\xt-(\ztp)_{i}\zt\big),\Ui^{\top}\frac{1}{|\Omegatwo|}\sum_{t\in\Omegatwo}\big((\xtp)_{i}\xt-(\ztp)_{i}\zt\big)\Big\rangle,\label{eq:def_statY}
\end{equation}
whose expectation is given by
\begin{align*}
\E[\stat_{Y}] & \approx\sum_{i=1}^{d}\Big\langle\Ui^{\top}\big(\Yki-\Yli\big),\Ui^{\top}\big(\Yki-\Yli\big)\Big\rangle\\
 & =\sum_{i=1}^{d}\Big\|\Ui^{\top}\big(\Yki-\Yli\big)\Big\|_{2}^{2}\approx\sum_{i=1}^{d}\big\|\Yki-\Yli\big\|_{2}^{2}=\big\|\Yk-\Yl\big\|_{\Frm}^{2};
\end{align*}
here, the first approximation is due to the near independence between
samples from $\Omegaone$ and those from $\Omegatwo$, whereas the
second approximation holds if each subspace $\Ui$ is sufficiently
close to $\Uistar=\spn\{(\bY^{(j)})_{i},1\le j\le K\}$. The purpose
of utilizing $\{\Ui\}$ is to reduce the variance of $\stat_{Y}$.
Similarly, we can compute another statistic based on $\{\bm{V}_{i}\}$
as follows:\texttt{
\begin{equation}
\stat_{\Gamma}\coloneqq\sum_{i=1}^{d}\Big\langle\Vi^{\top}\frac{1}{|\Omegaone|}\sum_{t\in\Omegaone}\big((\xt)_{i}\xt-(\zt)_{i}\zt\big),\Vi^{\top}\frac{1}{|\Omegatwo|}\sum_{t\in\Omegatwo}\big((\xt)_{i}\xt-(\zt)_{i}\zt\big)\Big\rangle,\label{eq:def_statGamma}
\end{equation}
}which has expectation
\[
\E[\stat_{\Gamma}]\approx\sumi\|\Vi^{\top}(\Gammaki-\Gammali)\|_{2}^{2}\approx\|\Gammak-\Gammal\|_{\Frm}^{2}.
\]
Consequently, we shall declare $k\neq\ell$ if $\stat_{\Gamma}+\stat_{Y}$
exceeds some appropriate threshold $\tau$.  

\paragraph{Procedure.}

We are now positioned to describe the proposed clustering procedure.
We first compute the statistics $\stat_{\Gamma}$ and $\stat_{Y}$
for each pair of sample trajectories in $\Mclustering$ by means of
the method described above, and then construct a similarity matrix
$\bS$, in a way that $S_{m,n}$ is set to $0$ if $\stat_{\Gamma}+\stat_{Y}$
(computed for the $m$-th and $n$-th trajectories) is larger
than a threshold $\tau$, or set to $1$ otherwise. In order to enhance
the robustness of these statistics, we divide $\{0,1,\dots,\Tsclustering\}$
into $4G$ (instead of $4$) segments, compute $G$ copies of $\stat_{\Gamma}$
and $\stat_{Y}$, and take the medians of these values. Next, we apply
a mainstream clustering algorithm (e.g.~spectral clustering \cite{chen2021spectral})
to the similarity matrix $\bS$, and divide $\Mclustering$ into $K$
disjoint clusters $\{\Ck\}_{1\le k\le K}$. The complete clustering
procedure is provided in Algorithm~\ref{alg:clustering}. The threshold
$\tau$ shall be chosen to be on the order of $\min_{1\le k<\ell\le K}\{\|\Gammak-\Gammal\|_{\Frm}^{2}+\|\Yk-\Yl\|_{\Frm}^{2}\}$
(which is strictly positive due to Fact~\ref{fact:separation}),
and it suffices to set the number of copies $G$ to be on some logarithmic
order. Our choice of these parameters will be specified in Theorem~\ref{thm:case0}
momentarily. 

\subsection{Model estimation }

Suppose that we have obtained reasonably good clustering accuracy
in the previous step, namely for each $1\le k\le K$, the cluster
$\Ck$ output by Algorithm~\ref{alg:clustering} contains mostly trajectories
generated by the same model $(\bA^{(\pi(k))},\bW^{(\pi(k))})$, with
$\pi$ representing some permutation function of $\{1,\dots,K\}$.
We propose to obtain a coarse model estimation and covariance estimation,
as summarized in Algorithm~\ref{alg:LS}. More specifically, for each
$k$, we use the samples $\{\{\xmt\}_{0\le t\le T_{m}}\}_{m\in\Ck}$
to obtain an estimate of $\bA^{(\pi(k))}$ by solving the following
least-squares problem \cite{simchowitz2018learning,sarkar2019near}
\begin{align}
\Akhat & \coloneqq\arg\min_{\bA}\sum_{m\in\Ck}\sum_{0\le t\le T_{m}-1}\|\xmtp-\bA\xmt\|_{2}^{2},
\end{align}
whose closed-form solution is given in (\ref{eq:refined_A}). Next,
we can use $\Akhat$ to estimate the noise vector
\[
\wmthat\coloneqq\xmtp-\Akhat\xmt\approx\wmt,
\]
and finally estimate the noise covariance $\bW^{(\pi(k))}$ with the
empirical covariance of $\{\wmthat\}$, as shown in (\ref{eq:refined_W}). 

\begin{algorithm}[tbp]
\DontPrintSemicolon
\caption{Least squares and covariance estimation} \label{alg:LS}
{\bf Input:} {clusters $\{\Ck\}_{1 \le k \le K}$}. \\
\For{$k = 1, \dots, K$}{ 
Compute
\begin{subequations}
\label{eq:refined_A_W}
\begin{align}
\Akhat &\gets\Big(\sum_{m\in\Ck}\sum_{0\le t\le T_{m}-1}\xmtp\xmt^{\top}\Big)\Big(\sum_{m\in\Ck}\sum_{0\le t\le T_{m}-1}\xmt\xmt^{\top}\Big)^{-1}, \quad \text{and} \label{eq:refined_A} \\
\Wkhat &\gets\frac{1}{\sum_{m\in\Ck}T_m}\sum_{m\in\Ck}\sum_{0\le t\le T_{m}-1}\wmthat\wmthat^{\top}, \quad \text{where} \quad \wmthat=\xmtp-\Akhat\xmt \label{eq:refined_W}.
\end{align}
\end{subequations}
}
{\bf Output:} {estimated models $\{\Akhat,\Wkhat\}_{1 \le k \le K}$}.
\end{algorithm}

\subsection{Classification \label{subsec:classification}}

\paragraph{Procedure.}

In the previous steps, we have obtained initial clusters $\{\Ck\}_{1\le k\le K}$
and coarse model estimates $\{\Akhat,\Wkhat\}_{1\le k\le K}$. With
the assistance of additional sample trajectories $\{\Xm\}_{m\in\Mclassification}$,
we can infer their latent labels and assign them to the corresponding
clusters; the procedure is stated in Algorithm~\ref{alg:classification}
and will be explained shortly. Once this is done, we can run Algorithm~\ref{alg:LS}
again with the updated clusters $\{\Ck\}$ to refine our model estimates,
which is exactly the last step of Algorithm~\ref{alg:overall}.

\paragraph{Rationale.}

The strategy of inferring labels in Algorithm~\ref{alg:classification}
can be derived from the maximum likelihood estimator, under the assumption
that the noise vectors $\{\wmt\}$ follow Gaussian distributions.
Note, however, that even in the absence of Gaussian assumptions, Algorithm~\ref{alg:classification}
remains effective in principle. To see this, consider a short trajectory
$\{\xt\}_{0\le t\le T}$ generated by model $(\Ak,\Wk)$, i.e.~$\xtp=\Ak\xt+\wt$
where $\E[\wt]=\mathbf{0},\cov(\wt)=\Wk$. Let us define the following
loss function 
\begin{equation}
L(\bA,\bW)\coloneqq T\cdot\log\det(\bW)+\sum_{t=0}^{T-1}(\xtp-\bA\xt)^{\top}\bW^{-1}(\xtp-\bA\xt).\label{eq:def_loss}
\end{equation}
With some elementary calculation, we can easily check that for any
incorrect label $\ell\neq k$, it holds that $\E[L(\Al,\Wl)-L(\Ak,\Wk)]>0$,
with the proviso that $(\Ak,\Wk)\neq(\Al,\Wl)$ and $\{\xt\}$ are
non-degenerate in some sense; in other words, the correct model $(\Ak,\Wk)$
achieves the minimal expected loss (which, due to the quadratic form
of the loss function, depends solely on the first and second moments
of the distributions of $\{\wt\}$, as well as the initial state $\bx_{0}$).
This justifies the proposed procedure for inferring unknown labels
in Algorithm~\ref{alg:classification}, provided that $T_{m}$ is
large enough and that the estimated LDS models are sufficiently reliable. 

\begin{algorithm}[tbp]
\DontPrintSemicolon
\caption{Classification} \label{alg:classification}
{\bf Input:} {short trajectories $\{\Xm\}_{m \in \Mclassification}$, where $\Xm=\{\xmt\}_{0\le t\le T_m}$; coarse models $\{\Akhat,\Wkhat\}_{1\le k \le K}$; clusters $\{\Ck\}_{1 \le k \le K}$}. \\
\For{$m \in \Mclassification$}{
Infer the label of the $m$-th trajectory by 
\begin{equation*}
\hat{k}_{m}\gets\arg\min_{\ell}\,\,\Big\{ T_m\cdot\log\det(\Wlhat)+\sum_{t=0}^{T_m-1}(\xmtp-\Alhat\xmt)^{\top}(\Wlhat)^{-1}(\xmtp-\Alhat\xmt)\Big\},
\end{equation*}
then add $m$ to cluster $\mathcal{C}_{\hat{k}_m}$. 
}
{\bf Output:} {updated clusters $\{\Ck\}_{1 \le k \le K}$}.
\end{algorithm}
\begin{rem}
\label{rem:beyond_LDS}While this work focuses on mixtures of LDSs,
we emphasize that the general principles of Algorithm~\ref{alg:overall}
are applicable under much weaker assumptions. For Algorithms~\ref{alg:subspace}
and~\ref{alg:clustering} to work, we essentially only require that
each sample trajectory satisfies a certain \emph{mixing} property,
and that the autocovariances of different models are \emph{sufficiently
separated} (and hence distinguishable). As for Algorithms~\ref{alg:LS}
and~\ref{alg:classification}, we essentially require a \emph{well-specified
parametric form} of the time-series models. This observation might
inspire future extensions (in theory or applications) of Algorithm~\ref{alg:overall}
to much broader settings.
\end{rem}

\section{Main results \label{sec:main_results}}

\subsection{Model assumptions \label{subsec:models_assumptions}}

To streamline the theoretical analysis, we focus on the case where
the trajectories are driven by \emph{Gaussian noise}; that is, for
each $1\le m\le M,0\le t\le T_{m}$, the noise vector $\wmt$ in (\ref{eq:setting})
is independently generated from the Gaussian distribution $\Ncal(\mathbf{0},\Wkm)$.
Next, we assume for simplicity that the labels $\{k_{m}\}_{1\le m\le M}$
of the trajectories are pre-determined and fixed, although one might
equivalently regard $\{k_{m}\}$ as being random and independent of
the noise vectors $\{\wmt\}$. Moreover, while our algorithms and
analysis are largely insensitive to the initial states $\{\bx_{m,0}\}_{1\le m\le M}$,
we focus on two canonical cases for concreteness: (i) the trajectories
start at \emph{zero} state, or (ii) they are segments of \emph{one}
continuous long trajectory. This is formalized as follows: 
\begin{subequations}
\label{eq:case01}
\begin{align}
\text{Case 0:}\quad & \bx_{m,0}=\mathbf{0},\quad1\le m\le M;\label{eq:case0}\\
\text{Case 1:}\quad & \bx_{1,0}=\mathbf{0},\quad\text{and}\quad\bx_{m+1,0}=\bx_{m,T_{m}},\quad1\le m\le M-1.\label{eq:case1}
\end{align}
\end{subequations}
We further introduce some notation for the total trajectory length:
define
\begin{align*}
\Tt & \coloneqq\sum_{1\le m\le M}T_{m}=\Ttsubspace+\Ttclustering+\Ttclassification,\\
\text{where}\quad & T_{\mathsf{total,o}}\coloneqq\Tso\cdot|\Mcal_{\mathsf{o}}|,\quad\mathsf{o}\in\{\mathsf{subspace},\mathsf{clustering},\mathsf{classification}\}.
\end{align*}
Additionally, we assume that each model occupies a non-degenerate
fraction of the data; in other words, there exists some $0<\pmin\le1/K$
such that
\[
\pmin\le\pko\coloneqq\frac{1}{|\Mo|}\sum_{m\in\Mo}\ind(k_{m}=k),\quad1\le k\le K,\quad\mathsf{o}\in\{\mathsf{subspace},\mathsf{clustering},\mathsf{classification}\}.
\]

Finally, we make the following assumptions about the ground-truth
LDS models, where we recall that the autocovariance matrices $\{\Gammak,\Yk\}$
have been defined in (\ref{eq:def_Gammak_Yk}). 
\begin{assumption}
\label{assu:models}The LDS models $\{\Ak,\Wk\}_{1\le k\le K}$ satisfy
the following conditions:
\begin{enumerate}
\item There exist $\ka\ge1$ and $0\le\rho<1$ such that for any $1\le k\le K$,
\begin{equation}
\|(\Ak)^{t}\|\le\ka\cdot\rho^{t},\quad t=1,2,\dots;\label{eq:stability}
\end{equation}
\item There exist $\Gmax\ge\lup\ge\llb>0$ and $\kwcross\ge\kwself\ge1$
such that for any $1\le k\le K$,
\begin{align*}
\lambmax(\Gammak)\le\Gmax, & \quad\llb\le\lambmin(\Wk)\le\lambmax(\Wk)\le\lup,\\
\frac{\lup}{\llb}\eqqcolon\kwcross, & \quad\kappa(\Wk)=\frac{\lambmax(\Wk)}{\lambmin(\Wk)}\le\kwself;
\end{align*}
\item There exist $\Dsep,\DAW>0$ such that for any $1\le k<\ell\le K$,
\begin{subequations}
\label{eq:def_separation}
\begin{align}
\|\Gammak-\Gammal\|_{\Frm}^{2}+\|\Yk-\Yl\|_{\Frm}^{2} & \ge\Dsep^{2},\label{eq:DGY}\\
\|\Ak-\Al\|_{\Frm}^{2}+\frac{\|\Wk-\Wl\|_{\Frm}^{2}}{\lup^{2}} & \ge\DAW^{2}.\label{eq:DAW}
\end{align}
\end{subequations}
\end{enumerate}
\end{assumption}
The first assumption states that each state transition matrix $\Ak$
is \emph{exponentially stable}, which is a quantified version of stability
and has appeared in various forms in the literature of LDS \cite{kailath2000linear,cohen2018online};
here, $\ka$ can be regarded as a condition number, while $\rho$
is a contraction rate. The second assumption ensures that the noise
covariance matrices $\{\Wk\}$ are well conditioned, and the autocovariance
matrices $\{\Gammak\}$ are bounded. The third assumption quantifies
the separation between different LDS models; in particular, the notion
of $\DGY$ (resp. $\DAW$) will play a crucial role in our analysis
of the clustering (resp.~classification) step in Algorithm~\ref{alg:overall}.
It is important that we consider the separation of $(\Gammak,\Yk)$
versus $(\Gammal,\Yl)$ jointly, which guarantees that $\DGY$ is
always strictly positive (thanks to Fact \ref{fact:separation}),
despite the possibility of $\Gammak=\Gammal$ or $\Yk=\Yl$; the reasoning
for our definition of $\DAW$ is similar. For the readers' convenience,
we include Table \ref{tab:list_params} for a quick reference to the
key notation and parameters used in our analysis.
\begin{rem}
\label{rem:sqrt_d}Since the separation parameters $\DAW,\DGY$ in
(\ref{eq:def_separation}) are defined with respect to the Frobenius
norm, we may naturally regard them as 
\begin{equation}
\DAW=\sqrt{d}\,\delAW,\quad\DGY=\sqrt{d}\,\delGY,\label{eq:separation_natural}
\end{equation}
where $\delAW,\delGY$ are the \emph{canonical separation parameter}
(in terms of the spectral norm). For example, consider the simple
setting where $K=2$, $\Wone=\Wtwo$, $\Aone=0.5\Id$ and $\Atwo=(0.5+\delta)\Id$
for some canonical separation parameter $\delta$; in
this case, we have $\DAW=\|\Aone-\Atwo\|_{\Frm}=\|\delta\Id\|_{\Frm}=\sqrt{d}\delta$.
This observation will be crucial for obtaining the correct dependence
on the dimension $d$ in our subsequent analysis.
\end{rem}
\begin{rem}
\label{rem:parameters} Most of the parameters in Assumption \ref{assu:models}
come directly from the ground-truth LDS models $\{\Ak,\Wk\}$, except
$\Gmax$ and $\DGY$. In fact, they can be bounded by $\Gmax\lesssim\Wmax$
and $\DGY\gtrsim\DAW$; see Fact \ref{fact:parameters} in the appendix
for the formal statements. However, the pre-factors in these bounds
can be pessimistic in general cases, therefore we choose to preserve
the parameters $\Gmax$ and $\DGY$ in our analysis.
\end{rem}
\begin{table}
\caption{A list of notation and parameters. In the subscripts of $\protect\Mo,\protect\Tso,\protect\Tto$,
the symbol $\protect\osf$ takes value in $\{\protect\subspace,\protect\clustering,\protect\classification\}$.
\label{tab:list_params}}

\begin{centering}
\begin{tabular}{cc}
\toprule 
Notation & Explanation\tabularnewline
\midrule
\midrule 
$d$ & State dimension\tabularnewline
\midrule 
$K$ & Number of LDS models\tabularnewline
\midrule 
$k_{m}$ & The unknown label (latent variable) of the $m$-th trajectory\tabularnewline
\midrule 
$\Mo$ & Subsets of $M$ trajectories, $\{1,\dots,M\}=\Msubspace\cup\Mclustering\cup\Mclassification$\tabularnewline
\midrule 
$\pmin$ & A lower bound for the fraction of trajectories generated by each model\tabularnewline
\midrule 
$\Tso$ & Short trajectory length for $\Mo$\tabularnewline
\midrule 
$\Tto$ & Total trajectory length, $\Tto=\Tso\cdot|\Mo|$\tabularnewline
\midrule 
$\Ak,\Wk$ & State transition matrix and noise covariance matrix of the $k$-th
LDS model\tabularnewline
\midrule 
$\Gammak,\Yk$ & Order-0 and order-1 stationary autocovariance matrices of the $k$-th
LDS model\tabularnewline
\midrule 
$\ka,\rho$ & $\|(\Ak)^{t}\|\le\ka\cdot\rho^{t},\quad t=1,2,\dots$\tabularnewline
\midrule 
$\DGY,\DAW$ & Model separation (\ref{eq:def_separation})\tabularnewline
\midrule 
$\llb,\lup$ & $\llb\le\lambmin(\Wk)\le\lambmax(\Wk)\le\Wmax,\quad1\le k\le K$\tabularnewline
\midrule 
$\kwcross,\kwself$ & $\kwcross=\lup/\llb;\quad\cond(\Wk)\le\kwself,\quad1\le k\le K$\tabularnewline
\midrule 
$\Gmax$ & $\|\Gammak\|\le\Gmax,\quad1\le k\le K$\tabularnewline
\bottomrule
\end{tabular}
\par\end{centering}
\end{table}

\subsection{Theoretical guarantees\label{subsec:main_theorems}}

We are now ready to present our end-to-end performance guarantees
for Algorithm~\ref{alg:overall}. Recall the definition of Cases 0
and 1 in (\ref{eq:case01}). Our main results for these two cases
are summarized in Theorems \ref{thm:case0} and \ref{thm:case1} below,
with proofs deferred to Section~\ref{sec:analysis}. This section
ends with some interpretations and discussions of the results.
\begin{thm}
[Main result for Case 0] \label{thm:case0} There exist positive
constants $\{C_{i}\}_{1\le i\le8}$ such that the following holds for
any fixed $0<\delta<1/2$. Consider the model~(\ref{eq:setting})
under the assumptions in Sections~\ref{sec:algorithms} and~\ref{subsec:models_assumptions},
with a focus on Case 0 (\ref{eq:case0}). Suppose that we run Algorithm~\ref{alg:overall}
with parameters $\tau,G$ that satisfy $1/8<\tau/\DGY^{2}<3/8$, $G\ge C_{1}\iota_{1}$,
and data $\{\Xm\}_{1\le m\le M}$ (where $\Xm=\{\xmt\}_{0\le t\le T_{m}}$)
that satisfies
\begin{subequations}
\label{eq:Ts_Tt_conditions}
\begin{align}
 & \Tssubspace\ge C_{2}\frac{\iota_{1}}{1-\rho},\quad\Ttsubspace\ge C_{3}\frac{d}{1-\rho}\Bigg(\bigg(\frac{\Gmax\sqrt{d}}{\DGY}\bigg)^{4}\frac{K^{2}}{\pmin^{2}}+1\Bigg)\cdot\iota_{1}^{4},\label{eq:Tssubspace}\\
 & \Tsclustering\ge C_{4}\frac{G}{1-\rho}\left(\frac{\Gmax^{2}\ka^{2}\sqrt{dK}}{\Dsep^{2}}+1\right)\iota_{2},\quad\Ttclustering\ge C_{5}\frac{d\kwcross^{2}}{\pmin}\bigg(\frac{d}{\DAW^{2}}\frac{\Gmax}{\Wmin}+1\bigg)\iota_{3}^{2},\label{eq:Tsclustering}\\
 & \Tsclassification\ge C_{6}\Big(\kwself^{2}+\frac{\kwcross^{6}}{\DAW^{2}}\Big)\iota_{1}^{2},\label{eq:Tsclassification}
\end{align}
\end{subequations}
where we define the logarithmic terms $\iota_{1}\coloneqq\log(\frac{d\ka\Tt}{\delta}),\iota_{2}\coloneqq\log\big((\frac{\Gmax}{\Dsep}+2)\frac{d\ka\Tt}{\delta}\big),\iota_{3}\coloneqq\log(\frac{\Gmax}{\Wmin}\frac{d\ka\Tt}{\delta})$.
Then, with probability at least $1-\delta$, Algorithm~\ref{alg:overall}
achieves exact clustering in Line~\ref{line:clustering} and exact
classification in Line~\ref{line:classification}; moreover, there
exists some permutation $\pi:\{1,\dots,K\}\rightarrow\{1,\dots,K\}$
such that the final model estimation $\{\Akhat,\Wkhat\}_{1\le k\le K}$
in Line~\ref{line:LS2} obeys 
\begin{align}
\|\Akhat-\Apik\|\le C_{7}\sqrt{\frac{d\kwself\iota_{3}}{\pmin T}}, & \quad\frac{\|\Wkhat-\Wpik\|}{\|\Wpik\|}\le C_{8}\sqrt{\frac{d\iota_{3}}{\pmin T}},\quad1\le k\le K,\label{eq:final_errors}
\end{align}
where $T\coloneqq\Ttclustering+\Ttclassification$.
\end{thm}
In fact, for a sharper result, one can replace the $\pmin T$ terms
in the final error bounds (\ref{eq:final_errors}) with the sample
size for the $\pi(k)$-th model, namely $T^{(\pi(k))}\coloneqq p_{\clustering}^{(\pi(k))}\Ttclustering+p_{\classification}^{(\pi(k))}\Ttclassification\ge\pmin T$.
\begin{thm}
[Main result for Case 1] \label{thm:case1} Consider the same setting
of Theorem \ref{thm:case0}, except that we focus on Case 1 (cf.~(\ref{eq:case1}))
instead. If we replace the conditions on $\Ttclustering$ and $\Tsclassification$
in (\ref{eq:Ts_Tt_conditions}) with
\begin{subequations}
\label{eq:Ts_Tt_case1}
\begin{align}
\Ttclustering & \ge C_{5}\frac{d\kwcross^{2}}{\pmin}\bigg(\frac{d}{\DAW^{2}}\frac{\Gmax}{\Wmin}\ka^{2}+1\bigg)\iota_{3}^{2},\label{eq:Ttclustering_case1}\\
\Tsclassification & \ge C_{6}\Big(\kwself^{2}+\frac{\kwcross^{6}}{\DAW^{2}}\Big)\iota_{1}^{2}+\frac{1}{1-\rho}\log(2\ka),\label{eq:Tsclassification_case1}
\end{align}
\end{subequations}
then the same high-probability performance guarantees in Theorem \ref{thm:case0}
continue to hold.
\end{thm}
While Theorems \ref{thm:case0} and \ref{thm:case1} guarantee that
Algorithm~\ref{alg:overall} successfully learns a mixture of LDS
models with a polynomial number of samples, these results involve
many parameters and may be somewhat difficult to interpret. In the
following corollary, we make some simplifications and focus on the
most important parameters.
\begin{cor}
\label{cor:simplify} Consider the same setting of Theorems \ref{thm:case0}
and \ref{thm:case1}. For simplicity, suppose that the condition numbers
$\ka,\kwself,\kwcross\asymp1$, and the fractions of data generated
by different LDS models are balanced, namely $\pmin\asymp1/K$. Moreover,
define the canonical separation parameters $\delAW\coloneqq\DAW/\sqrt{d}$
and $\delGY\coloneqq\DGY/\sqrt{d}$, as suggested by Remark~\ref{rem:sqrt_d}.
Finally, define the mixing time $\tmix\coloneqq1/(1-\rho)$. Then
we can rewrite the sample complexities in Theorems~\ref{thm:case0}
and~\ref{thm:case1} as follows (where we hide the logarithmic terms
$\{\iota_{i}\}_{1\le i\le3}$): if 
\begin{align*}
\Tssubspace\gtrsim\tmix, & \qquad\Ttsubspace\gtrsim\tmix d\Bigg(\bigg(\frac{\Gmax K}{\delGY}\bigg)^{4}+1\Bigg),\\
\Tsclustering\gtrsim\tmix\left(\Big(\frac{\Gmax}{\delGY}\Big)^{2}\sqrt{\frac{K}{d}}+1\right), & \qquad\Ttclustering\gtrsim Kd\bigg(\frac{1}{\delAW^{2}}\frac{\Gmax}{\Wmin}+1\bigg),\\
\Tsclassification\gtrsim\begin{cases}
\frac{1}{d\delAW^{2}}+1 & \text{for Case 0},\\
\frac{1}{d\delAW^{2}}+\tmix & \text{for Case 1},
\end{cases} & \qquad\Ttclustering+\Ttclassification\gtrsim\frac{Kd}{\epsilon^{2}},
\end{align*}
then with high probability, Algorithm~\ref{alg:overall} achieves
exact clustering in Line~\ref{line:clustering}, exact classification
in Line~\ref{line:classification}, and final model estimation errors
\begin{align*}
\|\Akhat-\Apik\| & \le\epsilon,\quad\frac{\|\Wkhat-\Wpik\|}{\|\Wpik\|}\le\epsilon,\quad1\le k\le K
\end{align*}
for some permutation $\pi:\{1,\dots,K\}\rightarrow\{1,\dots,K\}$.
\end{cor}
In what follows, we briefly remark on the key implications of Corollary
\ref{cor:simplify}. 
\begin{itemize}
\item \emph{Dimension $d$ and targeted error $\epsilon$}: (1) Our algorithms
allow the $\Tso$'s to be much smaller than (and even decrease with)
$d$. (2) Each $\Tto$ shall grow linearly with $d$, and it takes
an order of $Kd/\epsilon^{2}$ samples in total to learn $K$ models
in up to $\epsilon\rightarrow0$ errors (in the spectral norm), which
is just $K$ times the usual parametric rate ($d/\epsilon^{2}$) of
estimating a single model. 
\item \emph{Mixing time $\tmix$}: (1) $\Tssubspace$ and $\Tsclustering$
are linear in $\tmix$, which ensures sufficient mixing and thus facilitates
our algorithms for Stage 1. In contrast, $\Tsclassification$ depends
on $\tmix$ only for Case 1, with the sole and different purpose of
ensuring that the states $\{\xmt\}$ are bounded throughout (see Example~\ref{exa:switching}
in the appendix for an explanation of why this is necessary). (2)
While $\Ttsubspace$ needs to grow linearly with $\tmix$, this is
not required for $\Ttclustering$ and $\Ttclassification$, because
our methods for model estimation (i.e.~least squares and covariance
estimation) do not rely on mixing \cite{simchowitz2018learning,sarkar2019near}.
\item \emph{Canonical separation parameters $\delAW,\delGY$}: 
(1)~$\Tsclustering\gtrsim 1/\delGY^{2}$ guarantees exact clustering of the trajectories, while $\Tsclassification\gtrsim 1/\delAW^{2}$ guarantees exact classification. 
(2)~$\Ttsubspace\gtrsim 1/\delGY^{4}$ leads to sufficiently accurate subspaces, while $\Ttclustering\gtrsim 1/\delAW^{2}$ leads to accurate initial model estimation.\footnote{It is possible to improve
the $1/\delGY^{4}$ factor in $\Ttsubspace$ to $1/\delGY^{2}$, if
one is willing to pay for some extra factors of eigen-gaps; see Section~\ref{sec:analysis} for a detailed discussion.}
\end{itemize}
\begin{rem}
The step of subspace estimation is non-essential and optional; it
allows for a smaller $\Tsclustering$, but comes at the price of complicating
the overall algorithm. For practitioners who prefer a simpler algorithm,
they might simply remove this step (i.e.~Line~\ref{line:subspace}
of Algorithm~\ref{alg:overall}), and replace the rank-$K$ subspaces
$\{\Vi,\Ui\}$ with $\Id$ (i.e.~no dimensionality reduction for
the clustering step). The theoretical guarantees continue to hold
with minor modification: in Corollary \ref{cor:simplify}, one simply
needs to remove the conditions on $\Tssubspace,\Ttsubspace$, and
in the condition for $\Tsclustering$, replace the factor $\sqrt{K/d}$
(where $K$ is due to dimensionality reduction) with $\sqrt{d/d}=1$.
This is one example regarding how our modular algorithms and theoretical
analysis can be easily modified and adapted to accommodate different
situations. 
\end{rem}

\subsection{Numerical experiments }

We now validate our theoretical findings with a series of numerical
experiments, confirming that Algorithm~\ref{alg:overall} successfully
solves the mixed LDSs problem. In these experiments, we fix $d=80,K=4$;
moreover, let $\Tssubspace=20,$ $\Tsclustering=20$ and $\Tsclassification=5$,
all of which are much smaller than~$d$. We take $|\Msubspace|=30\,d,|\Mclustering|=10\,d$,
and vary $|\Mclassification|$ between $[0,5000\,d]$. Our experiments
focus on Case~1 as defined in (\ref{eq:case1}), and we generate
the labels of the sample trajectories uniformly at random. The ground-truth
LDS models are generated in the following manner: $\Ak=\rho\bR^{(k)}$,
where $\rho=0.5$ and $\bR^{(k)}\in\R^{d\times d}$ is a random orthogonal
matrix; $\Wk$ has eigendecomposition $\bU^{(k)}\mathbf{\Lambda}^{(k)}(\bU^{(k)})^{\top}$,
where $\bU^{(k)}$ is a random orthogonal matrix and the diagonal entries
of $\mathbf{\Lambda}^{(k)}$ are independently drawn from the uniform
distribution on $[1,2]$. 

Our experimental results are illustrated in Figure~\ref{fig:exp_result}.
Here, the horizontal axis represents the sample size $T=\Ttclustering+\Ttclassification$
for model estimation, and the vertical axis represents the estimation
errors, measured by $\max_{1\le k\le K}\|\Akhat-\Apik\|$ (plotted
in blue) and $\max_{1\le k\le K}\|\Wkhat-\Wpik\|/\|\Wpik\|$ (plotted
in orange). The results confirm our main theoretical prediction: Algorithm~\ref{alg:overall}
recovers the LDS models based on a mixture of short trajectories with
length $\Tso\ll d$, and achieves an error rate of $1/\sqrt{T}$.
In addition, we observe in our experiments that the final outputs
of Algorithm~\ref{alg:overall} are robust to a small number of mis-clustered
or mis-classified trajectories during the intermediate steps; this
is clearly an appealing property to have, especially when the $\Tso$'s
and $|\Mo|$'s in reality are slightly smaller than what our theory
requires. 

\begin{figure}
\begin{centering}
\includegraphics[width=0.5\textwidth]{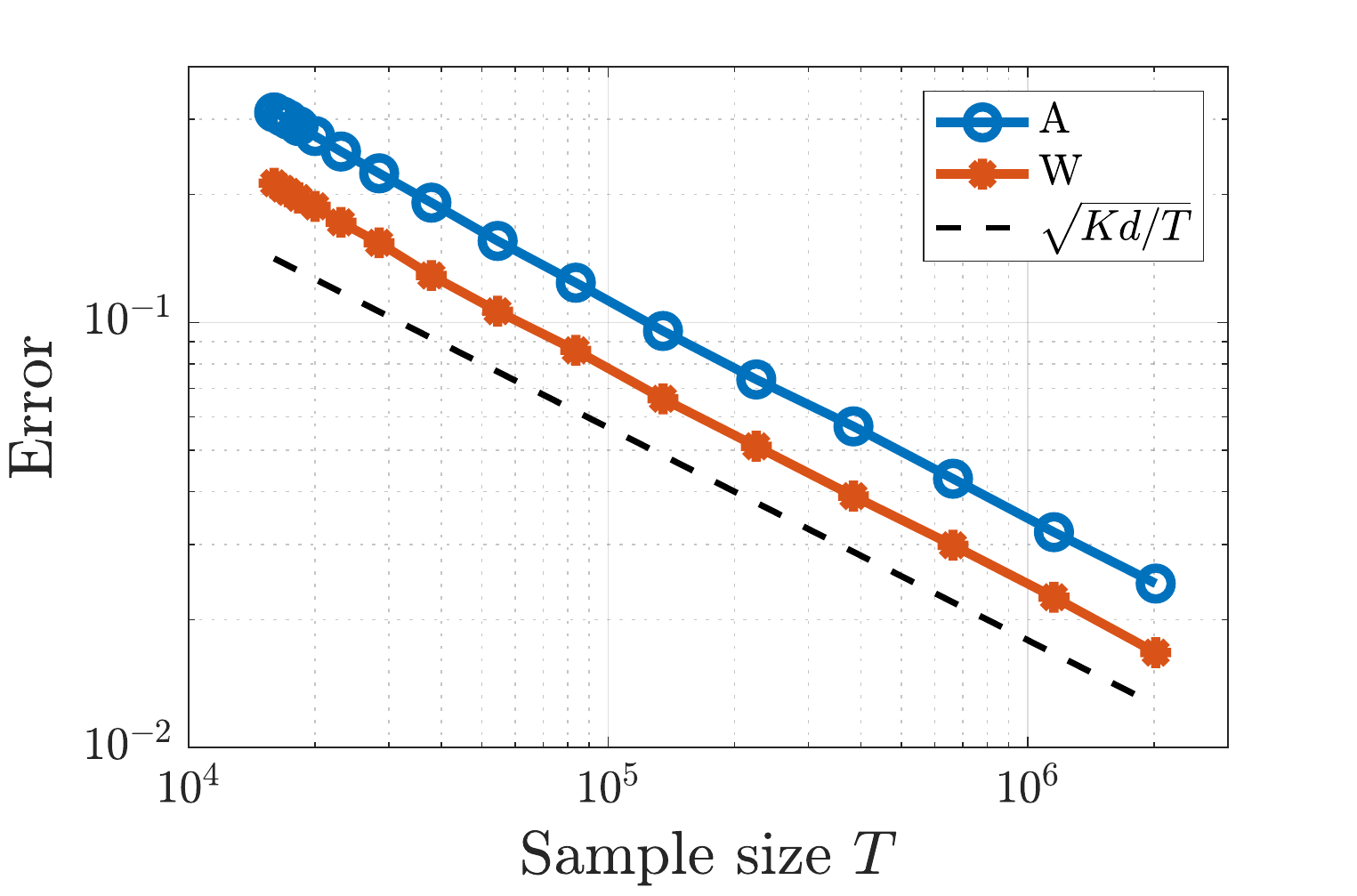}
\par\end{centering}
\caption{\label{fig:exp_result}The model estimation errors of Algorithm~\ref{alg:overall}
versus the total sample size (excluding $\protect\Msubspace$). Each
curve is an average over 12 independent trials.}

\end{figure}

\section{Related works\label{sec:related_work}}

\paragraph{Mixed linear regression and meta-learning.}

Our work is closely related to the recent papers \cite{kong2020meta,kong2020robust}
that bridge mixed linear regression and meta-learning. In mixed linear
regression \cite{quandt1978estimating}, it is assumed that there
exist a small number of linear models of interest; each independent
data point is generated by one of these models, although the label
(i.e., which model generates a data sample) is unknown. The past decade
has witnessed numerous theoretical advances in this problem, e.g.~\cite{yi2014alternating,chen2014convex,li2018learning,chen2020learning,kwon2020minimax,diakonikolas2020small,yin2018learning,mazumdar2020recovery,chen2021learning,diamandis2021wasserstein}.
In addition, meta-learning \cite{finn2017model,harrison2020continuous}
considers the problem of jointly solving multiple supervised learning
tasks, each of which has only a few samples. The key assumption is
that there is a common underlying structure (or inductive bias) in
these different tasks; the goal is to retrieve this structure, which
can be utilized later on to learn a new task with much fewer samples.
Some closely related problem formulations include few-shot learning
\cite{snell2017prototypical,du2020few}, transfer learning \cite{pan2009survey,TripuraneniJJ20},
multi-task learning \cite{baxter2000model,maurer2016benefit}, and
so on. Finally, in the setting of \emph{meta-learning for mixed
linear regression} \cite{kong2020meta,kong2020robust}, each task
contains a few independent and identically distributed samples generated
by a linear regression model, and the inductive bias is that there
is only a small discrete set of ground-truth linear models. Our work
extends this setting to time-series data (where each short trajectory
can be naturally regarded as a task); some parts of our algorithmic
designs have been inspired by the work of \cite{kong2020meta,kong2020robust},
and we also improve upon some of their analyses.\footnote{There is a concurrent preprint \cite{modi2021joint} that extends
multi-task learning \cite{du2020few,TripuraneniJJ20} to time-series
data, and the setting therein includes mixed LDSs as a special case.
However, the authors of \cite{modi2021joint} assume oracle access
to the global optimum of a non-convex optimization problem, without
providing a practical algorithm that can provably find it; moreover,
with the short trajectory length fixed, the estimation error bounds
in that work will remain bounded away from zero, even if the number
of trajectories grows to infinity. In comparison, we consider a simpler
problem setting, and propose computationally efficient algorithms
with better error bounds.}

\paragraph{Mixtures of time-series models and trajectories.}

Mixture models for time series have achieved empirical success in
the study of psychology \cite{bulteel2016clustering,takano2020clustering},
neuroscience \cite{albert1991two,mezer2009cluster}, biology \cite{wong2000mixture},
air pollution \cite{d2015time}, economics \cite{mcculloch1994statistical,maharaj2000cluster,kalliovirta2016gaussian},
automobile sensors \cite{hallac2017toeplitz}, and many other domains.
 Some specific aspects of mixture models include hypothesis testing
for a pair of trajectories \cite{maharaj2000cluster}, or clustering
of multiple trajectories \cite{liao2005clustering,aghabozorgi2015time,pathak2021cluster,huang2021coresets}.
In addition to mixture models, other related yet different models
in the literature include time-varying systems \cite{qu2021stable,minasyan2021online},
systems with random parameters \cite{du2020q}, switching systems
\cite{sun2006switched,sarkar2019nonparametric,ansari2021deep,mcculloch1994statistical},
switching state-space models \cite{ghahramani2000variational,linderman2017bayesian},
Markovian jump systems \cite{shi2015survey,zhao2019brief}, and event-triggered
systems \cite{sedghi2020machine,schluter2020event}, to name just
a few. There are even more related models in reinforcement learning
(RL), such as latent bandit \cite{maillard2014latent,hong2020revisited},
multi-task learning \cite{wilson2007multi,lazaric2010bayesian,brunskill2013multitask,liu2016pac,sodhani21a}/meta-learning
\cite{finn2017model}/transfer-learning \cite{taylor2009transfer,huang2021adarl,tirinzoni2020sequential}
for RL, typed parametric models \cite{brunskill2009provably}, latent
Markov decision processes \cite{kwon2021rl,hallak2015contextual,steimle2021multi,buchholz2019computation},
and so on. What distinguishes our work from this extensive literature is
that, we design algorithms and prove rigorous non-asymptotic sample
complexities for model estimation, in the specific setting of mixture
modeling that features (1)~a finite set of underlying time-series
models, and (2)~unknown labels of the trajectories, with no probabilistic
assumptions imposed on these latent variables.

\paragraph{Linear dynamical systems.}

Linear dynamical systems (also referred to as vector autoregressive
models in the statistics literature) is one of the most fundamental
models in system identification and optimal control \cite{ljung1998system,khalil1996robust}.
In the past few years, there has been a surge of studies, in both
control and machine learning communities, about non-asymptotic theoretical
analyses of various learning procedures for LDSs. This includes the
basic LDS model \cite{faradonbeh2018finite,simchowitz2018learning,sarkar2019near},
LDSs with control input and quadratic cost (i.e.~linear-quadratic
regulators) \cite{dean2020sample,cohen2018online,jedra2019sample,faradonbeh2020adaptive,mania2019certainty,simchowitz2020naive,fazel2018global,malik2019derivative},
and LDSs with partial observation $\yt\approx\bC\xt$ (e.g.~Kalman
filtering or linear-quadratic-Gaussian control) \cite{oymak2019non,simchowitz2019learning,Sarkar2021LTI,sun20regularization,tsiamis2020sample,Lale2020log,zheng2021sample}.
In particular, it was only until recently that the authors of \cite{simchowitz2018learning,sarkar2019near}
proved sharp error bounds of ordinary least squares for estimating
the state transition matrix of the basic LDS model, using a single
trajectory; our analysis of Algorithm~\ref{alg:LS} is largely inspired
by their techniques.

\section{Analysis\label{sec:analysis}}

This section provides detailed, modular theoretical results for the
performance of Algorithms~\ref{alg:subspace}--\ref{alg:classification},
and concludes with a proof for the main theorems in Section \ref{subsec:main_theorems}
(i.e.~the performance guarantees of Algorithm~\ref{alg:overall}).

\paragraph{Subspace estimation.}

The following theorem provides upper bounds on the errors of subspaces
$\{\Vi,\Ui\}$ output by Algorithm~\ref{alg:subspace}, assuming sufficient
mixing of the short trajectories. 
\begin{thm}
\label{thm:subspace} Consider the model~(\ref{eq:setting}) under
the assumptions in Sections~\ref{sec:algorithms} and \ref{subsec:models_assumptions}.
There exist some universal constants $C_{1},C_{2},C_{3}>0$ such that
the following holds. Suppose that we run Algorithm~\ref{alg:subspace}
with data $\{\Xm\}_{m\in\Msubspace}$ obeying
\[
\Tssubspace\ge C_{1}\cdot\tmix,\quad\Ttsubspace\ge C_{2}\cdot\tmix d\cdot\log\frac{\Tt d}{\delta},\quad\text{where}\quad\tmix\coloneqq\frac{1}{1-\rho}\cdot\log\bigg(\frac{d\ka\Tt}{\delta}\bigg).
\]
Then with probability at least $1-\delta$, Algorithm~\ref{alg:subspace}
outputs $\{\Vi,\Ui\}_{1\le i\le d}$ satisfying the following: for
all $1\le k\le K$ and $1\le i\le d$,
\begin{align}
 & \max\left\{ \|\Gammaki-\Vi\Vi^{\top}\Gammaki\|_{2},\|\Yki-\Ui\Ui^{\top}\Yki\|_{2}\right\} \nonumber \\
 & \qquad\le C_{3}\cdot\Gmax\bigg(\frac{K}{\pmin}\bigg)^{1/2}\bigg(\frac{\tmix d}{\Ttsubspace}\log^{3}\frac{\Tt d}{\delta}\bigg)^{1/4}.\label{eq:bound_subspace}
\end{align}
\end{thm}
Our proof (deferred to Appendix~\ref{subsec:proof_subspace}) includes
a novel perturbation analysis; the resulted error bound~(\ref{eq:bound_subspace})
has a $1/\Ttsubspace^{1/4}$ dependence and is gap-free (i.e.~independent
of the eigenvalue gaps of the ground-truth low-rank matrices, which
can be arbitrarily close to zero in the worst case). It is possible
to adapt the existing perturbation results in \cite{kong2020meta,kong2020robust}
to our setting (which we include in Lemma \ref{lem:step1_perturbation}
in the appendix for completeness); however, one of them is dependent
on the eigenvalue gaps, while the other one incurs a worse $1/\Ttsubspace^{1/6}$
dependence. It would be interesting future work to investigate whether
a gap-free bound with a $1/\Ttsubspace^{1/2}$ dependence is possible.

\paragraph{Clustering.}

Our next theorem shows that Algorithm~\ref{alg:clustering} achieves
exact clustering of $\Mclustering$, if $\Tsclustering$ is sufficiently
large and subspaces $\{\Vi,\Ui\}$ are accurate. The proof is deferred
to Appendix \ref{subsec:proof_clustering}. 
\begin{thm}
\label{thm:clustering} Consider the model~(\ref{eq:setting}) under
the assumptions in Sections~\ref{sec:algorithms} and \ref{subsec:models_assumptions}.
There exist universal constants $C_{1},C_{2},c_{3}>0$ such that the
following holds. Suppose that we run Algorithm~\ref{alg:clustering}
with data $\{\Xm\}_{m\in\Mclustering}$, independent subspaces $\{\Vi,\Ui\}_{1\le i\le d}$
and parameters $\tau$, $G$ that satisfy the following: 
\begin{itemize}
\item The threshold $\tau$ obeys $1/8<\tau/\DGY^{2}<3/8$;
\item The short trajectory length $\Tsclustering=NG$, where
\[
G\ge C_{1}\cdot\log\frac{|\Mclustering|}{\delta},\quad N\ge C_{2}\left(\frac{\Gmax^{2}\ka^{2}\sqrt{dK}}{\Dsep^{2}}+1\right)\frac{1}{1-\rho}\log\Bigg(\bigg(\frac{\Gmax}{\Dsep}+2\bigg)\frac{d\ka\Tt}{\delta}\Bigg);
\]
\item The subspaces $\{\Vi,\Ui\}_{1\le i\le d}$ satisfy that, for all $1\le i\le d$
and $1\le k\le K$, 
\begin{equation}
\max\bigg\{\|\Gammaki-\Vi\Vi^{\top}\Gammaki\|_{2},\|\Yki-\Ui\Ui^{\top}\Yki\|_{2}\bigg\}\le c_{3}\frac{\DGY}{\sqrt{d}}.\label{eq:2a_subspace_error}
\end{equation}
\end{itemize}
Then with probability at least $1-\delta$, Algorithm~\ref{alg:clustering}
achieves exact clustering: for all $m_{1},m_{2}\in\Mclustering$,
$S_{m_{1},m_{2}}=1$ if and only if the $m_{1}$-th and $m_{2}$-th
trajectories are generated by the same model, i.e.~they have the
same label $k_{m_{1}}=k_{m_{2}}$.
\end{thm}

\paragraph{Least squares and covariance estimation.}

The next result controls the model estimation errors of Algorithm~\ref{alg:LS},
under the assumption that every cluster is pure.
\begin{thm}
\label{thm:LS} Consider the model~(\ref{eq:setting}) under the
assumptions in Section~\ref{subsec:models_assumptions}. There exist
universal constants $C_{1},C_{2},C_{3}>0$ such that the following
holds. Let $\{\Ck\}_{1\le k\le K}$ be subsets of $\Mclustering\cup\Mclassification$
such that for all $1\le k\le K$, $\Ck$ contains only short trajectories
generated by model $(\Ak,\Wk)$, namely $k_{m}=k$ for all $m\in\Ck$.
Suppose that for all $m\in\Mclustering\cup\Mclassification$, $T_{m}\ge4$,
and for all $1\le k\le K$, 
\[
\Ttk\coloneqq\sum_{m\in\Ck}T_{m}\ge C_{1}d\kwself^{2}\iota,\quad\text{where}\quad\iota\coloneqq\log\Big(\frac{\Gmax}{\Wmin}\frac{d\ka\Tt}{\delta}\Big).
\]
Let $\Akhat,\Wkhat$ be computed by (\ref{eq:refined_A_W}) in Algorithm~\ref{alg:LS}.
Then with probability at least $1-\delta$, one has
\begin{align*}
\|\Akhat-\Ak\| & \le C_{2}\sqrt{\frac{d\kwself\iota}{\Ttk}},\quad\frac{\|\Wkhat-\Wk\|}{\|\Wk\|}\le C_{3}\sqrt{\frac{d\iota}{\Ttk}},\quad1\le k\le K.
\end{align*}
\end{thm}
Our proof (postponed to Appendix \ref{subsec:proof_LS}) is based
on the techniques of \cite{simchowitz2018learning,sarkar2019near},
but with two major differences. First, the authors of \cite{simchowitz2018learning,sarkar2019near}
consider the setting where ordinary least squares is applied to a
single continuous trajectory generated by a single LDS model; this
is not the case for our setting, and thus our proof and results are
different from theirs. Second, the noise covariance matrix $\bW$
is assumed to be $\sigma^{2}\Id$ in \cite{simchowitz2018learning,sarkar2019near},
while in our case, $\{\Wk\}_{1\le k\le K}$ are unknown and need to
be estimated.

\paragraph{Classification.}

Our last theorem shows that Algorithm~\ref{alg:classification} correctly
classifies all trajectories in $\Mclassification$, as long as the
coarse models are sufficiently accurate and the short trajectory lengths
are large enough; these conditions are slightly different for Cases
0 and 1 defined in (\ref{eq:case01}). See Appendix \ref{subsec:proof_classification}
for the proof.  
\begin{thm}
\label{thm:classification} Consider the model~(\ref{eq:setting})
under the assumptions in Section~\ref{subsec:models_assumptions}.
There exist universal constants $c_{1},c_{2},C_{3}>0$ such that
the following holds. Suppose that we run Algorithm~\ref{alg:classification}
with data $\{\Xm\}_{m\in\Mclassification}$ and independent coarse
models $\{\Akhat,\Wkhat\}_{1\le k\le K}$ satisfying $\|\Akhat-\Ak\|\le\epsA,\|\Wkhat-\Wk\|\le\epsW$
for all $k$. Then with probability at least $1-\delta$, Algorithm~\ref{alg:classification}
correctly classifies all trajectories in $\Mclassification$, provided
that 
\begin{subequations}
\label{eq:classification_conditions}
\begin{align}
\text{For Case 0:}\quad & \epsA\le c_{1}\DAW\sqrt{\frac{\llb}{\Gmax\kwcross(d+\iota)}},\quad\frac{\epsW}{\llb}\le c_{2}\cdot\min\Big\{1,\frac{\DAW}{\sqrt{\kwcross d}}\Big\},\label{eq:err_A_W_case0}\\
 & T_{m}\ge C_{3}\Big(\kwself^{2}+\frac{\kwcross^{6}}{\DAW^{2}}\Big)\iota^{2},\quad m\in\Mclassification;\label{eq:Tm_case0}\\
\text{For Case 1:}\quad & \epsA\le c_{1}\DAW\sqrt{\frac{\llb}{\Gmax\kwcross\ka^{2}(d+\iota)}},\quad\frac{\epsW}{\llb}\le c_{2}\cdot\min\Big\{1,\frac{\DAW}{\sqrt{\kwcross d}}\Big\},\label{eq:err_A_W_case1}\\
 & T_{m}\ge C_{3}\Big(\kwself^{2}+\frac{\kwcross^{6}}{\DAW^{2}}\Big)\iota^{2}+\frac{1}{1-\rho}\log(2\ka),\quad m\in\Mclassification,\label{eq:Tm_case1}
\end{align}
\end{subequations}
where $\iota\coloneqq\log\frac{\Tt}{\delta}$ is a logarithmic term.
\end{thm}

\paragraph{Proof of Theorems \ref{thm:case0} and \ref{thm:case1}.}

Our main theorems are direct implications of the above guarantees
for the individual steps. 
\begin{itemize}
\item \emph{Stage 1}: To begin with, according to Theorem \ref{thm:subspace},
if the condition (\ref{eq:Tssubspace}) on $\Tssubspace$ and $\Ttsubspace$
hold, then the subspaces $\{\Vi,\Ui\}$ output by Line~\ref{line:subspace}
of Algorithm~\ref{alg:overall} are guaranteed to satisfy the error
bounds (\ref{eq:2a_subspace_error}) required by Theorem \ref{thm:clustering}.
This together with the condition (\ref{eq:Tsclustering}) on $\Tsclustering$
ensures exact clustering in Line~\ref{line:clustering}. 
\item \emph{Stage 2}: Based on this, if we further know that $\Ttclustering$
obeys condition (\ref{eq:Tsclustering}) for Case 0 or (\ref{eq:Ttclustering_case1})
for Case 1, then Theorem \ref{thm:LS} tells us that the coarse model
estimation in Line~\ref{line:LS1} satisfies the error bounds (\ref{eq:err_A_W_case0})
or (\ref{eq:err_A_W_case1}) required by Theorem \ref{thm:classification}.
Together with the assumption (\ref{eq:Tsclassification}) or (\ref{eq:Tsclassification_case1})
on $\Tsclassification$, this guarantees exact classification in Line~\ref{line:classification}
of Algorithm\ref{alg:overall}, for either Case 0 or 1. At the end,
the final model estimation errors (\ref{eq:final_errors}) follow
immediately from Theorem \ref{thm:LS}. 
\end{itemize}
Note that all the statements above are high-probability guarantees;
it suffices to take the union bound over these steps, so that the
performance guarantees of Theorems \ref{thm:case0} and \ref{thm:case1}
hold with probability at least $1-\delta$. This finishes the proof
of our main theorems.

\section{Discussion \label{sec:Discussions}}

This paper has developed a theoretical and algorithmic framework for
learning multiple LDS models from a mixture of short, unlabeled sample
trajectories. Our key contributions include a modular two-stage meta-algorithm,
as well as theoretical analysis demonstrating its computational and
statistical efficiency in solving the mixed LDSs problem. We would
like to invite the readers to contribute to this important topic by,
say, further strengthening the\emph{ }theoretical analysis and algorithmic
design.  For example, in certain cases $\Tsclustering$ can be a
bottleneck compared with $\Tssubspace$ and $\Tsclassification$,
and thus one might hope to achieve a better dependence on $\Ttclustering$
(e.g.~allowing $\Ttclustering\ll d$), by replacing Line~\ref{line:LS1}
in Algorithm~\ref{alg:overall} with a different method (e.g.~adapting
the method of coarse model estimation in \cite[Algorithm 3]{kong2020meta}
to our setting). As another example, from the practical perspective,
it is possible that the data is a single continuous trajectory and
the time steps when the underlying model changes are \emph{unknown}
\cite{hallac2017toeplitz,harrison2020continuous}; in order to accommodate
such a case, one might need to incorporate change-point detection
into the learning process. 

Moving beyond the current setting of mixed LDSs, we remark that there
are plenty of opportunities for future studies. For instance, while
our methods in Stage 1 rely on the mixing property of the LDS models,
it is worth exploring whether it is feasible to handle the \emph{non-mixing}
case \cite{simchowitz2018learning,sarkar2019near}. Another potential
direction is to consider the robustness against outliers and adversarial
noise \cite{chen2021kalman,kong2020robust}. One might even go further
and extend the ideas (e.g.~the two-stage meta-algorithm) to learning
mixtures of \emph{other time-series models }(potentially with model
selection \cite{wong2000mixture}), such as LDS with partial observations
or nonlinear observations \cite{mhammedi2020learning}, autoregressive--moving-average
(ARMA) models, nonlinear dynamical systems \cite{mania2020active,kakade2020information,foster2020learning},
to name a few. Ultimately, it would be of great importance to consider
the case \emph{with controlled inputs}, such as learning mixtures
of linear-quadratic regulators, or latent Markov decision processes
\cite{kwon2021rl} that arises in reinforcement learning.

\section*{Acknowledgements}

Y.~Chen is supported in part by the ARO grant W911NF-20-1-0097, the NSF grants CCF-1907661 and IIS-1900140, and the AFOSR grant FA9550-19-1-0030.
H.~V.~Poor is supported in part by the NSF under Grant CCF-1908308.
We would like to thank Yuxin Chen and Gen Li for numerous helpful discussions.

\appendix

\section{Proofs for Section \ref{sec:analysis}} 

This section starts with some preliminaries about linear dynamical
systems that will be helpful later. Then, it provides the main proofs
for the theorems in Section \ref{sec:analysis}. 

\subsection{Preliminaries}

\paragraph{Truncation of autocovariance.}

Recall the notation $\Gammak=\sum_{i=0}^{\infty}\Ak^{i}\Wk(\Ak^{i})^{\top}$
from (\ref{eq:def_Gamma}) and (\ref{eq:def_Gammak_Yk}). We add a
subscript $t$ to represent its $t$-step truncation:
\begin{equation}
\boldsymbol{\Gamma}_{t}^{(k)}\coloneqq\sum_{i=0}^{t-1}\Ak^{i}\Wk(\Ak^{i})^{\top}.\label{eq:def_Gammakt}
\end{equation}
Also recall the assumption of exponential stability in (\ref{eq:stability}),
namely $\|(\Ak)^{t}\|\le\ka\rho^{t}$. As a result, $\bGamma_{t}^{(k)}$
is close to $\Gammak$:
\begin{align}
\boldsymbol{0}\preccurlyeq\Gammak-\bGamma_{t}^{(k)} & =\sum_{i=t}^{\infty}\Ak^{i}\Wk(\Ak^{i})^{\top}=\Ak^{t}\Gammak(\Ak^{t})^{\top},\nonumber \\
\|\Gammak-\bGamma_{t}^{(k)}\| & \le\|\Gammak\|\|\Ak^{t}\|^{2}\le\|\Gammak\|\ka^{2}\rho^{2t}\le\Gmax\ka^{2}\rho^{2t}.\label{eq:Gamma_truncated_approx}
\end{align}
Moreover, let $\bY_{t}^{(k)}\coloneqq\Ak\bGamma_{t}^{(k)}$, then
$\bY_{t}^{(k)}$ is also close to $\Yk$:
\begin{equation}
\|\Yk-\bY_{t}^{(k)}\|\le\|\Ak\|\|\Gammak-\bGamma_{t}^{(k)}\|\le\Gmax\ka^{3}\rho^{2t}.\label{eq:Y_truncated_approx}
\end{equation}

\paragraph{\textquotedblleft Independent version\textquotedblright{} of states.}

Given some mixing time $\tmix$, we define $\xtilmt(\tmix)$ as the
$(\tmix-1)$-step approximation of $\xmt$ :
\begin{equation}
\xtilmt=\xtilmt(\tmix)\coloneqq\sum_{i=0}^{\tmix-2}(\Akm)^{i}\bw_{m,t-i-1}\sim\Ncal(\boldsymbol{0},\bGamma_{\tmix-1}^{(k_{m})}),\quad\tmix\le t\le T_{m},\quad1\le m\le M.\label{eq:def_xtilde}
\end{equation}
Since $\xtilmt$ consists of only the most recent noise vectors, it
is independent of the history up to $\bx_{m,t-\tmix+1}$. Our proofs
for Stage 1 will rely on this ``independent version'' $\{\xtilmt\}$
of states $\{\xmt\}$; the basic idea is that, for an appropriately
chosen $\tmix$, a trajectory of length $T$ can be regarded as a
collection of $T/\tmix$ independent samples. We will often use the
notation $\tilde{\cdot}$ to represent the ``independent version''
of other variables as well.

\paragraph{Boundedness of states.}

The following lemma provides upper bounds for $\{\|\xtilmt\|_{2}\}$
and $\{\|\xmt\|_{2}\}$. This will help to control the effects of
mixing errors and model estimation errors in the analyses later.
\begin{lem}
[Bounded states] \label{lem:bounded_state} Consider the model~(\ref{eq:setting})
under the assumptions in Sections~\ref{sec:algorithms} and \ref{subsec:models_assumptions}.
Fix any $\tmix\ge3$. Then with probability at least $1-\delta$,
we have $\|\xtilmt\|_{2}\le C_{0}\sqrt{\Gmax(d+\log(\Tt/\delta))}$
for all $1\le m\le M,\tmix\le t\le T_{m}$, where $C_{0}>0$ is some
universal constant; moreover, for both cases of initial states defined
in (\ref{eq:case01}), all states $\{\{\bx_{m,t}\}_{0\le t\le T_{m}}\}_{1\le m\le M}$
are bounded throughout:
\begin{itemize}
\item Case 0: with probability at least $1-\delta$, we have $\|\bx_{m,t}\|_{2}\le C_{0}\sqrt{\Gmax(d+\log(\Tt/\delta))}$
for all $m,t$.
\item Case 1: suppose that $\tmix\ge\frac{1}{1-\rho}\log(\sqrt{2}\ka)$,
and $T_{m}\ge\tmix$ for all $m$, then with probability at least
$1-\delta$, we have $\|\bx_{m,t}\|_{2}\le3C_{0}\ka\sqrt{\Gmax(d+\log(\Tt/\delta))}$
for all $m,t$, and $\|\bx_{m,t}\|_{2}\le2C_{0}\sqrt{\Gmax(d+\log(\Tt/\delta))}$
for all $t\ge\tmix$ or $t=0$.
\end{itemize}
\end{lem}
\begin{proof}
First, recall from \cite[Corollary 7.3.3]{vershynin2018high} that,
if random vector $\ba\sim\Ncal(\boldsymbol{0},\Id)$, then for all
$u\ge0$, we have $\Pr(\|\ba\|_{2}\ge2\sqrt{d}+u)\le2\exp(-cu^{2})$.
Since $\xtilmt\sim\Ncal(\boldsymbol{0},\bGamma_{\tmix-1}^{(k_{m})})$,
where $\bGamma_{\tmix-1}^{(k_{m})}\preccurlyeq\bGamma^{(k_{m})}\preccurlyeq\Gmax\Id$,
we have $\Pr(\|\xtilmt\|_{2}\ge\sqrt{\Gmax}(2\sqrt{d}+u))\le2\exp(-cu^{2})$.
Taking the union bound, we have
\begin{align*}
 & \Pr\Big(\text{there exists }m,t\text{ such that }\|\xtilmt\|_{2}\ge\sqrt{\Gmax}(2\sqrt{d}+u)\Big)\\
 & \qquad\le\sum_{m=1}^{M}\sum_{t=\tmix}^{T_{m}}\Pr\Big(\|\xtilmt\|_{2}\ge\sqrt{\Gmax}(2\sqrt{d}+u)\Big)\le2\Tt\exp(-cu^{2})\le\delta,
\end{align*}
where the last inequality holds if we pick $u\ge\sqrt{\frac{1}{c}\log\frac{2\Tt}{\delta}}$.
This finishes the proof of the first claim in the lemma. Next, we
prove the boundedness of $\{\|\xmt\|_{2}\}$. 
\begin{itemize}
\item \emph{Case 0}: It is easy to check that $\bx_{m,t}\sim\Ncal(\boldsymbol{0},\bGamma_{t}^{(k_{m})})$,
where $\bGamma_{t}^{(k_{m})}\preccurlyeq\bGamma^{(k_{m})}\preccurlyeq\Gmax\Id$.
The boundedness of $\{\|\xmt\|_{2}\}$ can be proved by a similar
argument as before, which we omit for brevity.
\item \emph{Case 1}: Define $\ximt\coloneqq\xmt-(\Akm)^{t}\bx_{m,0}\sim\Ncal(\mathbf{0},\bGamma_{t}^{(k_{m})})$.
By a similar argument as before, we have with probability at least
$1-\delta$, $\|\ximt\|_{2}\le C_{0}\sqrt{\Gmax(d+\log(\Tt/\delta))}$
for all $m,t$. Morever, for any $t\ge\tmix\ge\frac{1}{1-\rho}\log(\sqrt{2}\ka)$,
we have $\|(\Akm)^{t}\|\le\ka\rho^{t}\le1/2$. With this in place,
we have
\begin{align*}
\bx_{m+1,0} & =\bx_{m,T_{m}}=(\Akm)^{T_{m}}\bx_{m,0}+\mathbf{\xi}_{m,T_{m}},\quad\text{and thus}\\
\|\bx_{m+1,0}\| & \le\|(\Akm)^{T_{m}}\|\|\bx_{m,0}\|_{2}+\|\mathbf{\xi}_{m,T_{m}}\|_{2}\le\frac{1}{2}\|\bx_{m,0}\|_{2}+C_{0}\sqrt{\Gmax(d+\log\frac{\Tt}{\delta})}.
\end{align*}
Recall the assumption that $\bx_{1,0}=\mathbf{0}$; by induction,
we have $\|\bx_{m,0}\|_{2}\le2C_{0}\sqrt{\Gmax(d+\log(\Tt/\delta))}$
for all $1\le m\le M$. Now, we have for all $m,t$, 
\begin{align*}
\|\xmt\|_{2} & \le\|(\Akm)^{t}\bx_{m,0}\|_{2}+\|\ximt\|_{2}\\
 & \le\ka\|\bx_{m,0}\|_{2}+C_{0}\sqrt{\Gmax(d+\log\frac{\Tt}{\delta})}\le3C_{0}\ka\sqrt{\Gmax(d+\log\frac{\Tt}{\delta})};
\end{align*}
moreover, for $t\ge\tmix$, since $\|(\Akm)^{t}\|\le1/2$, we obtain
a better bound
\begin{align*}
\|\xmt\|_{2} & \le\|(\Akm)^{t}\bx_{m,0}\|_{2}+\|\ximt\|_{2}\\
 & \le\frac{1}{2}\|\bx_{m,0}\|_{2}+C_{0}\sqrt{\Gmax(d+\log\frac{\Tt}{\delta})}\le2C_{0}\sqrt{\Gmax(d+\log\frac{\Tt}{\delta})}.
\end{align*}
\end{itemize}
This shows the boundedness of $\{\|\xmt\|_{2}\}$ and completes our
proof of the lemma.
\end{proof}

\subsection{\label{subsec:proof_subspace}Proof of Theorem~\ref{thm:subspace}}

Theorem~\ref{thm:subspace} is an immediate consequence of Lemmas~\ref{lem:step1_concentration}
and~\ref{lem:step1_perturbation} below. The former shows the concentration
of $\Hihat,\Gihat$ around the targeted low-rank matrices $\Hi,\Gi$,
while the latter is a result of perturbation analysis.
\begin{lem}
\label{lem:step1_concentration} Under the setting of Theorem~\ref{thm:subspace},
with probability at least $1-\delta$, we have for all $1\le i\le d$,
\[
\max\Big\{\|\Hihat-\Hi\|,\|\Gihat-\Gi\|\Big\}\lesssim\Gmax^{2}\sqrt{\frac{\tmix d}{\Ttsubspace}\log^{3}\frac{\Tt d}{\delta}}.
\]
\end{lem}
\begin{lem}
\label{lem:step1_perturbation}Consider the matrix $\Mstar=\sumk\pk\yk\yk^{\top}\in\R^{d\times d}$,
where $0<\pk<1,\sumk\pk=1$, and $\yk\in\R^{d}$. Let $\bM$ be symmetric
and satisfy $\|\bM-\Mstar\|\le\epsilon$, and $\bU\in\R^{d\times K}$
be the top-$K$ eigenspace of $\bM$. Then we have
\begin{equation}
\sumk\pk\|\yk-\bU\bU^{\top}\yk\|_{2}^{2}\le2K\epsilon,\label{eq:eigen_result1}
\end{equation}
and for all $1\le k\le K$, it holds that 
\begin{equation}
\|\yk-\bU\bU^{\top}\yk\|_{2}\le\min\left\{ \bigg(\frac{2K\epsilon}{\pk}\bigg)^{1/2},\frac{2\epsilon}{\lambmin(\bM)}\|\yk\|_{2},\sqrt{2}\bigg(\frac{\epsilon}{\pk}\|\yk\|_{2}\bigg)^{1/3}\right\} .\label{eq:eigen_result2}
\end{equation}
\end{lem}
For our main analyses in Sections \ref{subsec:main_theorems} and
\ref{sec:analysis}, we choose the first term on the right-hand side
of (\ref{eq:eigen_result2}). 

\subsubsection{Proof of Lemma \ref{lem:step1_concentration}}

We first analyze the idealized case with i.i.d. samples; then we make
a connection between this i.i.d. case and the actual case of mixed
LDSs, by utilizing the mixing property of linear dynamical systems.
We prove the result of Lemma \ref{lem:step1_concentration} for $\|\Gihat-\Gi\|$
only, since the analysis for $\|\Hihat-\Hi\|$ is mostly the same
(and simpler in fact). 

\paragraph{Step 1: the idealized i.i.d. case.}

With some abuse of notation, suppose that for all $1\le m\le M$,
we have for some $k_{m}\in\{1,\dots,K\}$, 
\begin{align*}
\xmt,\zmt\iid\Ncal(\boldsymbol{0},\tilde{\bGamma}^{(k_{m})}), & \quad\wmt,\vmt\iid\Ncal(\boldsymbol{0},\bW^{(k_{m})}),\\
\xmt'=\Akm\xmt+\wmt, & \quad\zmt'=\Akm\zmt+\vmt,\quad1\le t\le N,
\end{align*}
where for all $1\le k\le K$, it holds that $\Wk,\tilde{\bGamma}^{(k)}\preccurlyeq\Gammak\preccurlyeq\Gmax\Id$.
Notice that $\cov(\xmt')=\Akm\tilde{\bGamma}^{(k_{m})}(\Akm)^{\top}+\Wkm\preccurlyeq\Akm\bGamma^{(k_{m})}(\Akm)^{\top}+\Wkm=\bGamma^{(k_{m})}\preccurlyeq\Gmax\Id$.
Consider the i.i.d. version of matrix $\Gihat$ and its expectation
$\Gi$ defined as follows:
\[
\Gihat\coloneqq\frac{1}{MN}\sum_{m=1}^{M}\sum_{t=1}^{N}\Big(\big(\xmt'\big)_{i}\xmt\Big)\Big(\big(\zmt'\big)_{i}\zmt\Big)^{\top},\quad\Gi=\sumk\pk(\Ytilk)_{i}(\Ytilk)_{i}^{\top},
\]
where $(\Ytilk)_{i}$ is the transpose of the $i$-th row of $\tilde{\bY}^{(k)}\coloneqq\Ak\Gtilk$.
For this i.i.d. setting, we claim that, if the i.i.d. sample size
$MN$ satisfies $MN\gtrsim d\cdot\log(MNd/\delta)$, then with probability
at least $1-\delta$,
\begin{equation}
\|\Gihat-\Gi\|\lesssim\Gmax^{2}\sqrt{\frac{d}{MN}\log^{3}\frac{MNd}{\delta}},\quad1\le i\le d.\label{eq:step1_iid}
\end{equation}
This claim can be proved by a standard covering argument with truncation;
we will provide a proof later for completeness.

\paragraph{Step 2: back to the actual case of mixed LDSs.}

Now we turn to the analysis of $\Gihat$ defined in (\ref{eq:Hihat_Gihat})
versus its expectation $\Gi$ defined in (\ref{eq:Gi}), for the mixed
LDSs setting. We first show that $\Gihat$ can be writte as a \emph{weighted
average} of some matrices, each of which can be further decomposed
into an i.i.d.\emph{ part} (as in Step 1) plus a \emph{negligible
mixing error term}. Then we analyze each term in the decomposition,
and finally put pieces together to show that $\Gihat\approx\Gi$.

\paragraph{Step 2.1: decomposition of the index set $\protect\Omegaone\times\protect\Omegatwo$.}

Recall the definition of index sets $\Omegaone,\Omegatwo$ in Algorithm~\ref{alg:subspace}.
Denote the first index of $\Omegaone$ (resp.~$\Omegatwo$) as $\tau_{1}+1$
(resp.~$\tau_{2}+1$), and let $\Delta\coloneqq\tau_{2}-\tau_{1}$
be their distance. Also denote $N\coloneqq|\Omegaone|=|\Omegatwo|\asymp\Tssubspace$.
For any $t\in\Omegaone$ and $1\le j\le N$, define
\[
s_{j}(t)\coloneqq\mathsf{Cycle}(t+\Delta+j;\Omegatwo)=\begin{cases}
t+\Delta+j & \text{if}\quad t+\Delta+j\le\tau_{2}+N,\\
t+\Delta+j-N & \text{otherwise,}
\end{cases}
\]
where $\mathsf{Cycle}(i;\Omega)$ represents the cyclic indexing of
value $i$ on set $\Omega$. Then we have
\[
\Omegaone\times\Omegatwo=\Big\{(t_{1},t_{2}),t_{1}\in\Omegaone,t_{2}\in\Omegatwo\Big\}=\cup_{j=1}^{N}\Big\{\big(t,s_{j}(t)\big),t\in\Omegaone\Big\}.
\]
We further define
\[
\Stau\coloneqq\{\tau_{1}+\tau+f\cdot\tmix:f\ge0,\tau+f\cdot\tmix\le N\},\quad1\le\tau\le\tmix,
\]
so that $\Omegaone=\cup_{\tau=1}^{\tmix}\mathcal{S}_{\tau}$. Notice
that for each $\tau$, the elements of $\Stau$ are at least $\tmix$
far apart.  Putting together, we have
\begin{equation}
\Omegaone\times\Omegatwo=\cup_{j=1}^{N}\cup_{\tau=1}^{\tmix}\Big\{\big(t,s_{j}(t)\big),t\in\Stau\Big\}.\label{eq:decomp_Oone_Otwo}
\end{equation}

\paragraph{Step 2.2: decomposition of $\protect\Gihat$.}

In the remaining proof, we denote $\xmt'\coloneqq\xmtp$ for notational
consistency. Using the decomposition~(\ref{eq:decomp_Oone_Otwo})
of $\Omegaone\times\Omegatwo$, we can rewrite $\Gihat$ defined in
(\ref{eq:Hihat_Gihat}) as a weighted average of $N\tmix$ matrices:
\[
\Gihat=\frac{1}{|\Msubspace|}\sum_{m\in\Msubspace}\frac{1}{N^{2}}\sum_{(t_{1},t_{2})\in\Omegaone\times\Omegatwo}(\bx_{m,t_{1}}')_{i}\,\bx_{m,t_{1}}\cdot(\bx_{m,t_{2}}')_{i}\,\bx_{m,t_{2}}{}^{\top}=\sum_{j=1}^{N}\sum_{\tau=1}^{\tmix}\frac{|\Stau|}{N^{2}}\cdot\Fijtau,
\]
where
\begin{align}
\Fijtau & \coloneqq\frac{1}{|\Msubspace|\cdot|\Stau|}\sum_{m\in\Msubspace}\sum_{t\in\Stau}(\xmt')_{i}\,\xmt\cdot(\bx_{m,s_{j}(t)}')_{i}\,\bx_{m,s_{j}(t)}{}^{\top}.\label{eq:def_Fijtau}
\end{align}
Recalling the definition of $\xtilmt$ in (\ref{eq:def_xtilde}) and
$\bGamma_{t}^{(k)}$ in (\ref{eq:def_Gammakt}), we have
\[
\xmt=\tilde{\bx}_{m,t}+\underset{\eqqcolon\dmt}{\underbrace{(\Akm)^{\tmix-1}\bx_{m,t-\tmix+1}}}=\tilde{\bx}_{m,t}+\boldsymbol{\delta}_{m,t},
\]
where
\begin{equation}
\|\boldsymbol{\delta}_{m,t}\|_{2}\le\|(\Akm)^{\tmix-1}\|\cdot\|\bx_{m,t-\tmix+1}\|_{2}\le\ka\rho^{\tmix-1}\|\bx_{m,t-\tmix+1}\|_{2},\label{eq:delta_bound}
\end{equation}
and $\tilde{\bx}_{m,t}\sim\Ncal(\boldsymbol{0},\Gammaktmix)$ is independent
of $\dmt$. Moreover,
\[
\xmt'=\Akm\xmt+\wmt=\tilde{\bx}_{m,t}'+\Akm\boldsymbol{\delta}_{m,t},\quad\text{where}\quad\tilde{\bx}_{m,t}'\coloneqq\Akm\tilde{\bx}_{m,t}+\wmt.
\]
We can rewrite $\bx_{m,s_{j}(t)}=\tilde{\bx}_{m,s_{j}(t)}+\boldsymbol{\delta}_{m,s_{j}(t)}$
and $\bx_{m,s_{j}(t)}'=\tilde{\bx}_{m,s_{j}(t)}'+\Akm\boldsymbol{\delta}_{m,s_{j}(t)}$
 in a similar manner. Putting this back to (\ref{eq:def_Fijtau}),
one has 
\begin{align}
\Fijtau & =\frac{1}{|\Msubspace|\cdot|\Stau|}\sum_{m\in\Msubspace}\sum_{t\in\Stau}(\xmt')_{i}\,\xmt\cdot(\xmsjt')_{i}\,\xmsjt{}^{\top}\nonumber \\
 & =\frac{1}{|\Msubspace|\cdot|\Stau|}\sum_{m\in\Msubspace}\sum_{t\in\Stau}\nonumber \\
 & \qquad(\tilde{\bx}_{m,t}'+\Akm\boldsymbol{\delta}_{m,t})_{i}\,(\tilde{\bx}_{m,t}+\boldsymbol{\delta}_{m,t})\cdot(\xtilmsjt'+\Akm\boldsymbol{\delta}_{m,s_{j}(t)})_{i}\,(\xtilmsjt+\boldsymbol{\delta}_{m,s_{j}(t)}){}^{\top}\nonumber \\
 & =\underset{\eqqcolon\Ftilijtau}{\underbrace{\frac{1}{|\Msubspace|\cdot|\Stau|}\sum_{m\in\Msubspace}\sum_{t\in\Stau}(\tilde{\bx}_{m,t}')_{i}\,\tilde{\bx}_{m,t}\cdot(\xtilmsjt')_{i}\,\xtilmsjt{}^{\top}}}+\Dijtau,\label{eq:def_Ftilde_Delta}
\end{align}
where $\Dijtau$ contains all the $\{\dmt\}$ terms in the expansion.
The key observation here is that, by our definition of index set $\Stau$,
the $\{\xtilmt\}$ terms in $\Ftilijtau$ are independent, and thus
we can utilize our earlier analysis of the i.i.d. case in Step 1 to
study $\Ftilijtau$. 

\paragraph{Step 2.3: analysis for each term of the decomposition.}

Towards showing $\Gihat\approx\Gi$, we prove in the following that
$\Ftilijtau$ concentrates around its expectation $\tilde{\bG}_{i}$,
which in term is close to $\Gi$; moreover, the error term $\Dijtau$
becomes exponentially small as $\tmix$ grows. More specifically,
we have the following:
\begin{itemize}
\item Recall the notation $\Yktmix=\Ak\Gammaktmix$. It holds that
\[
\E[\Ftilijtau]=\tilde{\bG}_{i}\coloneqq\sumk\pk\Yktmixi\Yktmixi^{\top}.
\]
According to our result~(\ref{eq:step1_iid}) for the i.i.d. case,
for fixed $j,\tau$, we have with probability at least $1-\delta$,
\begin{align}
1\le i\le d,\quad\|\Ftilijtau-\tilde{\bG}_{i}\| & \lesssim\Gmax^{2}\sqrt{\frac{d}{|\Msubspace|\cdot|\Stau|}\log^{3}\frac{|\Msubspace|\cdot|\Stau|d}{\delta}}\nonumber \\
 & \lesssim\Gmax^{2}\sqrt{\frac{\tmix d}{\Ttsubspace}\log^{3}\frac{\Ttsubspace d}{\delta}}.\label{eq:Ftilde_concentration}
\end{align}
\item Recall from (\ref{eq:Y_truncated_approx}) that $\|\Yk-\Yktmix\|\le\Gmax\ka^{3}\rho^{2(\tmix-1)}$.
Therefore,
\begin{align*}
 & \big\|\Yki\Yki^{\top}-\Yktmixi\Yktmixi^{\top}\big\|\\
 & \qquad\le\big(\|\Yki\|_{2}+\|\Yktmixi\|_{2}\big)\,\|\Yki-\Yktmixi\|_{2}\\
 & \qquad\le\big(2\|\Yk\|+\|\Yk-\Yktmix\|\big)\,\|\Yk-\Yktmix\|\\
 & \qquad\le\big(2\Gmax\ka+\Gmax\ka^{3}\rho^{2(\tmix-1)}\big)\cdot\Gmax\ka^{3}\rho^{2(\tmix-1)}\\
 & \qquad=\big(2+\ka^{2}\rho^{2(\tmix-1)}\big)\cdot\Gmax^{2}\ka^{4}\rho^{2(\tmix-1)}\le3\Gmax^{2}\ka^{4}\rho^{2(\tmix-1)},
\end{align*}
where we use the mild assumption that $\tmix\ge1+\frac{\log\ka}{1-\rho}$,
and the fact that $\|\Yk\|,\|\Yktmix\|\le\Gmax\ka$. Consequently,
\begin{equation}
\|\Gi-\tilde{\bG}_{i}\|\le\sumk\pk\Big\|\Yki\Yki^{\top}-\Yktmixi\Yktmixi^{\top}\Big\|\le3\Gmax^{2}\ka^{4}\rho^{2(\tmix-1)}.\label{eq:Gtilde_close}
\end{equation}
\item By Lemma~\ref{lem:bounded_state}, if $\tmix\gtrsim\frac{1}{1-\rho}\log(2\ka)$,
then we have with probability at least $1-\delta$, all the $\xmt$'s
and $\tilde{\bx}_{m,t}$'s involved in the definition of $\Dijtau$
in~(\ref{eq:def_Ftilde_Delta}) have $\ell_{2}$ norm bounded by
$\sqrt{\Gmax}\poly(d,\ka,\log(\Tt/\delta))$. This together with the
upper bound on $\|\boldsymbol{\delta}_{m,t}\|_{2}$ in~(\ref{eq:delta_bound})
implies that for all $i,j,\tau$, it holds that
\begin{equation}
\|\Dijtau\|\le\Gmax^{2}\cdot\poly\Big(d,\ka,\log\frac{\Tt}{\delta}\Big)\cdot\rho^{\tmix-1}.\label{eq:bound_Delta}
\end{equation}
\end{itemize}

\paragraph{Step 2.4: putting pieces together.}

With (\ref{eq:Ftilde_concentration}), (\ref{eq:Gtilde_close}) and
(\ref{eq:bound_Delta}) in place and taking the union bound, we have
with probability at least $1-\delta$, for all $1\le i\le d,$
\begin{align*}
\|\Gihat-\Gi\| & =\bigg\|\sum_{j=1}^{N}\sum_{\tau=1}^{\tmix}\frac{|\Stau|}{N^{2}}\cdot\Fijtau-\Gi\bigg\|\le\max_{j,\tau}\|\Ftilijtau-\tilde{\bG}_{i}\|+\|\tilde{\bG}_{i}-\Gi\|+\max_{j,\tau}\|\Dijtau\|\\
 & \lesssim\Gmax^{2}\sqrt{\frac{\tmix d}{\Ttsubspace}\log^{3}\frac{\Ttsubspace d}{\delta}}+\Gmax^{2}\ka^{4}\rho^{2(\tmix-1)}+\Gmax^{2}\cdot\poly\Big(d,\ka,\log\frac{d\Tt}{\delta}\Big)\cdot\rho^{\tmix-1}\\
 & \lesssim\Gmax^{2}\sqrt{\frac{\tmix d}{\Ttsubspace}\log^{3}\frac{\Tt d}{\delta}},
\end{align*}
where the last inequality holds if $\tmix\gtrsim\frac{1}{1-\rho}\log\left(\frac{d\ka\Tt}{\delta}\right)$.
This finishes the proof of Lemma \ref{lem:step1_concentration}.
\begin{proof}
[Proof of (\ref{eq:step1_iid})] Define the truncating operator
\[
\trunc(x;D)\coloneqq x\cdot\ind(|x|\le D),\qquad x\in\R,\quad D\ge0.
\]
Consider the following truncated version of $\Gihat$:
\[
\Gihat^{\trunc}\coloneqq\frac{1}{MN}\sum_{m=1}^{M}\sum_{t=1}^{N}\bigg(\trunc\Big(\big(\xmt'\big)_{i};D_{0}\Big)\xmt\bigg)\bigg(\trunc\Big(\big(\zmt'\big)_{i};D_{0}\Big)\zmt\bigg)^{\top}
\]
(the truncating level $D_{0}$ will be specified later), and let $\bE_{i}^{\trunc}\coloneqq\E\big[\Gihat^{\trunc}\big]$
be its expectation. In the following, we first show that $\Gihat^{\trunc}$
concentrates around $\bE_{i}^{\trunc}$, and then prove that $\bE_{i}^{\trunc}\approx\Gi$.
\begin{itemize}
\item By a standard covering argument, we have
\begin{align*}
\|\Gihat^{\trunc}-\bE_{i}^{\trunc}\| & =\sup_{\bu,\bv\in\mathcal{S}^{d-1}}\bu^{\top}\Big(\Gihat^{\trunc}-\bE_{i}^{\trunc}\Big)\bv\le4\sup_{\bu,\bv\in\Ncal_{1/8}}\bu^{\top}\Big(\Gihat^{\trunc}-\bE_{i}^{\trunc}\Big)\bv,
\end{align*}
where $\Ncal_{1/8}$ denotes the $1/8$-covering of the unit sphere
$\mathcal{S}^{d-1}$ and has cardinality $|\Ncal_{1/8}|\le32^{d}$.
For fixed $\bu,\bv\in\Ncal_{1/8}$, one has 
\[
\bu^{\top}\Gihat^{\trunc}\bv=\frac{1}{MN}\sum_{m=1}^{M}\sum_{t=1}^{N}\trunc\Big(\big(\xmt'\big)_{i};D_{0}\Big)\bu^{\top}\xmt\cdot\trunc\Big(\big(\zmt'\big)_{i};D_{0}\Big)\bv^{\top}\zmt,
\]
where (cf. \cite[Chapter 2]{vershynin2018high} for the definitions
of subgaussian norm $\|\cdot\|_{\psi_{2}}$ and subexponential norm
$\|\cdot\|_{\psi_{1}}$)
\[
\bigg|\trunc\Big(\big(\xmt'\big)_{i};D_{0}\Big)\bigg|,\bigg|\trunc\Big(\big(\zmt'\big)_{i};D_{0}\Big)\bigg|\le D_{0},\quad\|\bu^{\top}\xmt\|_{\psi_{2}},\|\bv^{\top}\zmt\|_{\psi_{2}}\lesssim\sqrt{\Gmax}.
\]
Hence
\[
\bigg\|\trunc\Big(\big(\xmt'\big)_{i};D_{0}\Big)\bu^{\top}\xmt\cdot\trunc\Big(\big(\zmt'\big)_{i};D_{0}\Big)\bv^{\top}\zmt\bigg\|_{\psi_{1}}\lesssim D_{0}^{2}\Gmax,
\]
and by Bernstein's inequality \cite[Corollary 2.8.3]{vershynin2018high},
we have
\[
\Pr\bigg(\Big|\bu^{\top}\Big(\Gihat^{\trunc}-\bE_{i}^{\trunc}\Big)\bv\Big|\ge\tau\bigg)\le2\exp\left(-c_{0}MN\Big(\frac{\tau}{D_{0}^{2}\Gmax}\Big)^{2}\right)
\]
for all $0\le\tau\le D_{0}^{2}\Gmax$. Taking the union bound over
$\bu,\bv\in\Ncal_{1/8}$, we have with probability at least $1-\delta/2$,
\[
\|\Gihat^{\trunc}-\bE_{i}^{\trunc}\|\lesssim D_{0}^{2}\Gmax\sqrt{\frac{d+\log\frac{1}{\delta}}{MN}},\quad\text{provided that}\quad MN\gtrsim d+\log\frac{1}{\delta}.
\]
\item Note that
\[
\|\Gi-\bE_{i}^{\trunc}\|=\Bigg\|\sumk\pk\E_{\xt,\zt\sim\Ncal(0,\Gtilk)}\bigg[\Big((\xt')_{i}(\zt')_{i}-\trunc\big((\xt')_{i}\big)\trunc\big((\zt')_{i}\big)\Big)\xt\zt^{\top}\bigg]\Bigg\|,
\]
where for each $k$,
\begin{align*}
 & \Bigg\|\E_{\xt,\zt\sim\Ncal(0,\Gtilk)}\bigg[\Big((\xt')_{i}(\zt')_{i}-\trunc\big((\xt')_{i}\big)\trunc\big((\zt')_{i}\big)\Big)\xt\zt^{\top}\bigg]\Bigg\|\\
 & \qquad=\sup_{\bu,\bv\in\mathcal{S}^{d-1}}\Bigg|\E_{\xt,\zt\sim\Ncal(0,\Gtilk)}\bigg[(\xt')_{i}(\zt')_{i}\bu^{\top}\xt\bv^{\top}\zt\cdot\Big(1-\ind\big(|(\xt')_{i}|\le D_{0},|(\zt')_{i}|\le D_{0}\big)\Big)\bigg]\Bigg|\\
 & \qquad\le\sup_{\bu,\bv\in\mathcal{S}^{d-1}}\sqrt{\E\Big((\xt')_{i}(\zt')_{i}\bu^{\top}\xt\bv^{\top}\zt\Big)^{2}}\sqrt{\E\Big(1-\ind\big(|(\xt')_{i}|\le D_{0},|(\zt')_{i}|\le D_{0}\big)\Big)^{2}}\\
 & \qquad\lesssim\Gmax^{2}\sqrt{\Pr(|\Ncal(0,\Gmax)|>D_{0})}\\
 & \qquad\lesssim\Gmax^{2}\exp\left(-c_{1}\frac{D_{0}^{2}}{\Gmax}\right).
\end{align*}
\end{itemize}
Finally, notice that if the truncating level $D_{0}$ is sufficiently
large, then we have $\Gihat=\Gihat^{\trunc}$ with high probability.
More formally, we have shown that for fixed $1\le i\le d$, if $MN\gtrsim d+\log(1/\delta)$,
then 
\begin{align*}
\Pr\bigg(\|\Gihat-\Gi\| & \lesssim D_{0}^{2}\Gmax\sqrt{\frac{d+\log\frac{1}{\delta}}{MN}}+\Gmax^{2}\exp\left(-c_{1}\frac{D_{0}^{2}}{\Gmax}\right)\bigg)\\
 & \qquad\ge1-\frac{\delta}{2}-\sum_{m,t}\Pr\Big(|(\xmt')_{i}|>D_{0}\,\,\text{or}\,\,|(\zmt')_{i}|>D_{0}\Big).
\end{align*}
If we pick the truncating level $D_{0}\asymp\sqrt{\Gmax\log(MN/\delta)},$
then it is easy to check that with probability at least $1-\delta$,
\[
\|\Gihat-\Gi\|\lesssim D_{0}^{2}\Gmax\sqrt{\frac{d+\log\frac{1}{\delta}}{MN}}\lesssim\Gmax^{2}\sqrt{\frac{d}{MN}\log^{3}\frac{MN}{\delta}}.
\]
Taking the union bound over $1\le i\le d$ finishes our proof of~(\ref{eq:step1_iid}).
\end{proof}

\subsubsection{Proof of Lemma \ref{lem:step1_perturbation}}

Define $\boldsymbol{\Delta}\coloneqq\bM-\Mstar$, and denote the eigendecomposition
of $\Mstar$ and $\bM$ as $\Mstar=\bU_{\star}\boldsymbol{\Lambda}_{\star}\bU_{\star}^{\top}$
and $\bM=\bU\boldsymbol{\Lambda}\bU^{\top}+\bU_{\perp}\boldsymbol{\Lambda}_{\perp}\bU_{\perp}^{\top}$,
where diagonal matrix $\mathbf{\Lambda}$ (resp.~$\mathbf{\Lambda}_{\star}$)
contains the top-$K$ eigenvalues of $\bM$ (resp.~$\Mstar$). Then
we have
\begin{align*}
\boldsymbol{\Lambda} & =\bU^{\top}\bM\bU=\bU^{\top}\Mstar\bU+\bU^{\top}\boldsymbol{\Delta}\bU=\sumk\pk\bU^{\top}\yk\yk^{\top}\bU+\bU^{\top}\boldsymbol{\Delta}\bU,\\
\boldsymbol{\Lambda}_{\star} & =\bU_{\star}^{\top}\Mstar\bU_{\star}=\sumk\pk\bU_{\star}^{\top}\yk\yk^{\top}\bU_{\star}.
\end{align*}
Substracting these two equations gives
\[
\sumk\pk\Big(\bU_{\star}^{\top}\yk\yk^{\top}\bU_{\star}-\bU^{\top}\yk\yk^{\top}\bU\Big)=\boldsymbol{\Lambda}_{\star}-\boldsymbol{\Lambda}+\bU^{\top}\boldsymbol{\Delta}\bU.
\]
Taking the trace of both sides, we get
\[
\sumk\pk\Big(\|\bU_{\star}^{\top}\yk\|_{2}^{2}-\|\bU^{\top}\yk\|_{2}^{2}\Big)=\Tr\big(\boldsymbol{\Lambda}_{\star}-\boldsymbol{\Lambda}\big)+\Tr\big(\bU^{\top}\boldsymbol{\Delta}\bU\big).
\]
On the left-hand side, 
\[
\|\bU_{\star}^{\top}\yk\|_{2}^{2}-\|\bU^{\top}\yk\|_{2}^{2}=\|\yk\|_{2}^{2}-\|\bU^{\top}\yk\|_{2}^{2}=\|\yk-\bU\bU^{\top}\yk\|_{2}^{2}\ge0,
\]
while on the right-hand side,
\[
\Tr(\boldsymbol{\Lambda}_{\star}-\boldsymbol{\Lambda})=\sum_{k=1}^{K}\Big(\lambda_{k}(\boldsymbol{\Lambda}_{\star})-\lambda_{k}(\boldsymbol{\Lambda})\Big)\overset{{\rm (i)}}{\le}K\|\boldsymbol{\Delta}\|\le K\epsilon,\quad\Tr(\bU^{\top}\boldsymbol{\Delta}\bU)\le\|\boldsymbol{\Delta}\|\cdot\Tr(\bU^{\top}\bU)=\|\bDel\|\cdot\Tr(\bI_{K})\le K\epsilon,
\]
where (i) follows from Weyl's inequality. Putting things together,
we have proved~(\ref{eq:eigen_result1}), which immediately leads
to the first upper bound in~(\ref{eq:eigen_result2}). The second
upper bound in~(\ref{eq:eigen_result2}) follows from a simple application
of Davis-Kahan's $\sin\Theta$ theorem \cite{davis1970rotation},
and the third term is due to \cite[Lemma A.11]{kong2020meta}; we
skip the details for brevity.

\subsection{\label{subsec:proof_clustering}Proof of Theorem \ref{thm:clustering}}

Our proof follows the three steps below:
\begin{enumerate}
\item Consider the idealized i.i.d. case, and characterize the expectations
and variances of the testing statistics computed by Algorithm~\ref{alg:clustering};
\item Go back to the actual case of mixed LDSs, and analyze one copy of
$\statGammag$ or $\statYg$ defined in (\ref{eq:stat_g}) for some
fixed $1\le g\le G$, by decomposing it into an i.i.d. part plus a
negligible mixing error term;
\item Analyze $\median\{\statGammag,1\le g\le G\}$ and $\median\{\statYg,1\le g\le G\}$,
and prove the correct testing for each pair of trajectories, which
implies that Algorithm~\ref{alg:clustering} achieves exact clustering.
\end{enumerate}

\paragraph{Step 1: the idealized i.i.d. case.}

Recall the definition of $\stat_{Y}$ in (\ref{eq:def_statY}) when
we first introduce our method for clustering. For notational convenience,
we drop the subscript in $\stat_{Y}$, and replace $\xtp,\ztp$ with
$\xt',\zt'$; then, with some elementary linear algebra, (\ref{eq:def_statY})
can be rewritten as
\begin{align*}
\stat & =\sum_{i=1}^{d}\Big\langle\Ui^{\top}\frac{1}{|\Omegaone|}\sum_{t\in\Omegaone}\big((\xt')_{i}\xt-(\zt')_{i}\zt\big),\Ui^{\top}\frac{1}{|\Omegatwo|}\sum_{t\in\Omegatwo}\big((\xt')_{i}\xt-(\zt')_{i}\zt\big)\Big\rangle\\
 & =\Big\langle\bU^{\top}\frac{1}{|\Omegaone|}\sum_{t\in\Omegaone}\vc\big(\xt(\xt')^{\top}-\zt(\zt')^{\top}\big),\bU^{\top}\frac{1}{|\Omegatwo|}\sum_{t\in\Omegatwo}\vc\big(\xt(\xt')^{\top}-\zt(\zt')^{\top}\big)\Big\rangle,
\end{align*}
where we define a large orthonormal matrix
\begin{equation}
\bU\coloneqq\begin{bmatrix}\bU_{1} & \boldsymbol{0} & \dots & \boldsymbol{0}\\
\boldsymbol{0} & \bU_{2} & \ddots & \vdots\\
\vdots & \ddots & \ddots & \boldsymbol{0}\\
\boldsymbol{0} & \dots & \boldsymbol{0} & \bU_{d}
\end{bmatrix}\in\R^{d^{2}\times dK},\label{eq:def_U}
\end{equation}
In this step, we consider the idealized i.i.d. case:
\begin{align*}
t\in\Omegaone\cup\Omegatwo,\quad\xt\iid\Ncal(\boldsymbol{0},\Gtilk),\quad\wt\iid & \Ncal(\boldsymbol{0},\Wtilk),\quad\xt'=\Atilk\xt+\wt,\\
\zt\iid\Ncal(\boldsymbol{0},\Gtill),\quad\vt\iid & \Ncal(\boldsymbol{0},\Wtill),\quad\zt'=\Atill\zt+\vt,
\end{align*}
where $\Gtilk,\Gtill,\Wtilk,\Wtill$ are $d\times d$ covariance matrices,
and $\Atilk,\Atill$ are $d\times d$ state transition matrix. Our
goal is to characterize the expectation and variance of $\stat$ in
this i.i.d. case. 

Before we present the results, we need some additional notation. First,
let $\{\be_{i}\}_{1\le i\le d}$ be the canonical basis of $\R^{d}$,
and define
\begin{align*}
\Ytilk & \coloneqq\Atilk\Gtilk,\\
\Sigk & \coloneqq\Big(\Atilk\otimes\Id\Big)\Big((\Gtilk)^{1/2}\otimes(\Gtilk)^{1/2}\Big)\Big(\bI_{d^{2}}+\bP\Big)\Big((\Gtilk)^{1/2}\otimes(\Gtilk)^{1/2}\Big)\Big((\Atilk)^{\top}\otimes\Id\Big),
\end{align*}
where $\bP\in\R^{d^{2}\times d^{2}}$ is a symmetric permutation matrix,
whose $(i,j)$-th block is $\be_{j}\be_{i}^{\top}\in\R^{d\times d},1\le i,j\le d$.
Let $\Ytill,\Sigl$ be defined similarly, with $\Atilk,\Gtilk$ replaced
by $\Atill,\Gtill$. Moreover, define
\begin{subequations}
\label{eq:mu_Sig}
\begin{align}
\mukl & \coloneqq\bU^{\top}\vc\Big((\Ytilk-\Ytill)^{\top}\Big)\in\R^{dK},\label{eq:def_mu}\\
\Sigkl & \coloneqq\bU^{\top}\big(\Sigk+\Wtilk\otimes\Gtilk+\Sigl+\Wtill\otimes\Gtill\big)\bU\in\R^{dK\times dK}.\label{eq:def_Sig}
\end{align}
\end{subequations}

Now we are ready to present our results for the i.i.d. case. The first
lemma below gives a precise characterization of $\E[\stat]$ and $\var(\stat)$,
in terms of $\mukl$ and $\Sigkl$; the second lemma provides some
upper and lower bounds, which will be handy for our later analyses.
\begin{lem}
\label{lem:stat_E_var} Denote $N=\min\{|\Omegaone|,|\Omegatwo|\}$.
For the i.i.d. case just described, one has 
\begin{align*}
\E[\stat] & =\|\mukl\|_{2}^{2},\quad\var(\stat)\le\frac{1}{N^{2}}\Tr(\Sigkl^{2})+\frac{2}{N}\mukl^{\top}\Sigkl\mukl,
\end{align*}
and the inequality becomes an equality if $|\Omegaone|=|\Omegatwo|=N$.
\end{lem}
\begin{lem}
\label{lem:mu_Sigma_simp} Consider the same setting of Lemma~\ref{lem:stat_E_var}.
Furthermore, suppose that
\[
\Gtilk,\Gtill\preccurlyeq\Gmax\Id,\quad\Wtilk,\Wtill\preccurlyeq\lup\Id,\quad\|\Atilk\|,\|\Atill\|\le\ka
\]
 for some $0<\lup\le\Gmax$ and $\ka\ge1$. Then the following holds
true.
\begin{itemize}
\item (Upper bound on expectation) It holds that
\begin{equation}
\E[\stat]=\|\mukl\|_{2}^{2}\le\|\Ytilk-\Ytill\|_{\Frm}^{2}.\label{eq:Estat_upper}
\end{equation}
\item (Lower bound on expectation) If $\Ytilk\neq\Ytill$ and subspaces
$\{\Ui\}_{1\le i\le d}$ satisfy 
\begin{equation}
1\le i\le d,\quad\max\left\{ \|(\Ytilk)_{i}-\Ui\Ui^{\top}(\Ytilk)_{i}\|_{2},\|(\Ytill)_{i}-\Ui\Ui^{\top}(\Ytill)_{i}\|_{2}\right\} \le\epsilon,\label{eq:lem4_subspace_requirement}
\end{equation}
for some $\epsilon\ge0$, then we have
\[
\E[\stat]=\|\mukl\|_{2}^{2}\ge\big\|\Ytilk-\Ytill\big\|_{\Frm}^{2}-4\epsilon\sumi\big\|(\Ytilk)_{i}+(\Ytill)_{i}\big\|_{2}.
\]
\item (Upper bound on variance) The matrix $\Sigkl$ is symmetric and satisfies
\[
\boldsymbol{0}\preccurlyeq\Sigkl\preccurlyeq6\Gmax^{2}\ka^{2}\bI_{dK};
\]
this, together with the earlier upper bound (\ref{eq:Estat_upper})
on $\|\mukl\|_{2}^{2}$, implies that
\[
\var(\stat)\le\frac{1}{N^{2}}\Tr(\Sigkl^{2})+\frac{2}{N}\mukl^{\top}\Sigkl\mukl\lesssim\left(\frac{\Gmax^{2}\ka^{2}}{N}\right)^{2}dK+\frac{\Gmax^{2}\ka^{2}}{N}\|\Ytilk-\Ytill\|_{\Frm}^{2}.
\]
\end{itemize}
\end{lem}

\paragraph{Step 2: one copy of $\protect\statYg,\protect\statGammag$ for a
fixed $g$.}

Now we turn back to the mixed LDSs setting and prove Theorem~\ref{thm:clustering}.
Let us focus on the testing of one pair of short trajectories $\{\bx_{m_{1},t}\},\{\bx_{m_{2},t}\}$
for some $m_{1},m_{2}\in\Mclustering$, $m_{1}\neq m_{2}$. For notational
consistency, in this proof we rewrite these two trajectories as $\{\xt\}$
and $\{\zt\}$, their labels $k_{m_{1}},k_{m_{2}}$ as $k,\ell$,
and the trajectory length $\Tsclustering$ as $\Ts$, respectively.
Also denote $\xt'\coloneqq\xtp$ and $\zt'\coloneqq\ztp$. Recall
the definition of $\{\statYg\}_{1\le g\le G}$ in (\ref{eq:statYg});
for now, we consider one specific element and ignore the subscript
$g$. Recalling the definition of $\bU\in\R^{d^{2}\times dK}$ in
(\ref{eq:def_U}), we have
\begin{align*}
\stat_{Y} & =\sum_{i=1}^{d}\Big\langle\frac{1}{N}\sum_{t\in\Omegaone}\Ui^{\top}\big((\xt')_{i}\xt-(\zt')_{i}\zt\big),\frac{1}{N}\sum_{t\in\Omegatwo}\Ui^{\top}\big((\xt')_{i}\xt-(\zt')_{i}\zt\big)\Big\rangle\\
 & =\Big\langle\bU^{\top}\frac{1}{N}\sum_{t\in\Omegaone}\vc\big(\xt(\xt')^{\top}-\zt(\zt')^{\top}\big),\bU^{\top}\frac{1}{N}\sum_{t\in\Omegatwo}\vc\big(\xt(\xt')^{\top}-\zt(\zt')^{\top}\big)\Big\rangle,
\end{align*}
where $N=\lfloor\Ts/4G\rfloor=|\Omegaone|=|\Omegatwo|$. 

In the following, we show how to decompose $\stat_{Y}$ into an i.i.d.
term plus a negligible mixing error term, and then analyze each component
of this decomposition; finally, we put pieces together to give a characterization
of $\stat_{Y}$, or $\{\statYg\}_{1\le g\le G}$ when we put the subscript
$g$ back in at the end of this step. 

\paragraph{Step 2.1: decomposition of $\protect\stat_{Y}$.}

Define $\Sonetau\coloneqq\{t_{1}+\tau_{1}+f\cdot\tmix:f\ge0,\tau_{1}+f\cdot\tmix\le|\Omegaone|\}$,
where $t_{1}+1$ is the first index of $\Omegaone$, and the mixing
time $\tmix$ will be specified later; define $\Stwotau$ similarly.
Note that for each $\tau_{1}$, the elements of $\Sonetau$ are at
least $\tmix$ far apart; moreover, we have $\Omegaone=\cup_{\tau_{1}=1}^{\tmix}\Sonetau,\Omegatwo=\cup_{\tau_{2}=1}^{\tmix}\Stwotau$,
and thus

\begin{align*}
\frac{1}{N}\sum_{t\in\Omegaone}\vc\big(\xt(\xt')^{\top}-\zt(\zt')^{\top}\big) & =\sum_{\tau_{1}=1}^{\tmix}\frac{|\Sonetau|}{N}\cdot\frac{1}{|\Sonetau|}\sum_{t\in\Sonetau}\vc\big(\xt(\xt')^{\top}-\zt(\zt')^{\top}\big),\\
\frac{1}{N}\sum_{t\in\Omegatwo}\vc\big(\xt(\xt')^{\top}-\zt(\zt')^{\top}\big) & =\sum_{\tau_{2}=1}^{\tmix}\frac{|\Stwotau|}{N}\cdot\frac{1}{|\Stwotau|}\sum_{t\in\Stwotau}\vc\big(\xt(\xt')^{\top}-\zt(\zt')^{\top}\big).
\end{align*}
Therefore, we can rewrite $\stat_{Y}$ as a weighted average
\begin{align}
\stat_{Y} & =\sum_{\tau_{1}=1}^{\tmix}\sum_{\tau_{2}=1}^{\tmix}\wtau\cdot\stat_{Y}^{\tau_{1},\tau_{2}},\quad\text{where}\nonumber \\
\wtau & \coloneqq\frac{|\Sonetau||\Stwotau|}{N^{2}},\quad\sum_{\tau_{1}=1}^{\tmix}\sum_{\tau_{2}=1}^{\tmix}\wtau=1,\quad\text{and}\nonumber \\
\stat_{Y}^{\tau_{1},\tau_{2}} & \coloneqq\Big\langle\bU^{\top}\frac{1}{|\Sonetau|}\sum_{t\in\Sonetau}\vc\big(\xt(\xt')^{\top}-\zt(\zt')^{\top}\big),\bU^{\top}\frac{1}{|\Stwotau|}\sum_{t\in\Stwotau}\vc\big(\xt(\xt')^{\top}-\zt(\zt')^{\top}\big)\Big\rangle.\label{eq:def_statYtau}
\end{align}
We can further decompose $\stat_{Y}^{\tau_{1},\tau_{2}}$ into an
i.i.d. term plus a small error term. To do this, recalling the definition
of $\xtilmt$ in (\ref{eq:def_xtilde}) and dropping the subscript
$m$, we have 
\begin{align*}
\xt & =\underset{\eqqcolon\dxt}{\underbrace{\Ak^{\tmix-1}\bx_{t-\tmix+1}}}+\xtilt=\dxt+\xtilt,\\
\xt' & =\Ak\xt+\wt=\Ak\dxt+\underset{\eqqcolon\xtilt'}{\underbrace{(\Ak\xtilt+\wt)}}=\Ak\dxt+\xtilt',
\end{align*}
where $\xtilt\sim\Ncal(\boldsymbol{0},\Gammaktmix)$. Similarly, we
rewrite $\zt=\dzt+\ztilt,\zt'=\Al\dzt+\ztilt'$. Plugging these into
the right-hand side of (\ref{eq:def_statYtau}) and expanding it,
one has
\begin{align*}
\stattau_{Y} & =\stattautil_{Y}+\Dtau_{Y},\quad\text{where}\\
\stattautil_{Y} & \coloneqq\Big\langle\bU^{\top}\frac{1}{|\Sonetau|}\sum_{t\in\Sonetau}\vc\big(\xtilt(\xtilt')^{\top}-\ztilt(\ztilt')^{\top}\big),\bU^{\top}\frac{1}{|\Stwotau|}\sum_{t\in\Stwotau}\vc\big(\xtilt(\xtilt')^{\top}-\ztilt(\ztilt')^{\top}\big)\Big\rangle,
\end{align*}
and $\Dtau_{Y}$ involves $\{\dxt,\dzt\}$ terms.

\paragraph{Step 2.2: analysis of each component.}

First, notice that the $\{\xtilt,\ztilt\}$ terms in the definition
of $\stattautil_{Y}$ are independent; this suggests that we can characterize
$\stattautil_{Y}$ by applying our earlier analysis for the i.i.d.
case in Step 1. Second, $\Dtau_{Y}$ involves $\{\dxt,\dzt\}$ terms,
which in turn involve $\Ak^{\tmix-1},\Al^{\tmix-1}$ and thus will
be exponentially small as $\tmix$ increases, thanks to Assumption
\ref{assu:models}. More formally, we have the following:
\begin{itemize}
\item Applying Lemmas~\ref{lem:stat_E_var} and \ref{lem:mu_Sigma_simp}
with $(\Gtilk,\Gtill)=(\Gammaktmix,\Gammaltmix)$, $(\Wtilk,\Wtill)=(\Wk,\Wl)$
and $(\Atilk,\Atill)=(\Ak,\Al)$, we have
\begin{align*}
\E\big[\stattautil_{Y}\big] & =\big\|\bU^{\top}\vc((\Yktmix-\Yltmix)^{\top})\big\|_{2}^{2},\\
\var\big(\stattautil_{Y}\big) & \lesssim\left(\frac{\Gmax^{2}\ka^{2}}{\Ntil}\right)^{2}dK+\frac{\Gmax^{2}\ka^{2}}{\Ntil}\big\|\Yktmix-\Yltmix\big\|_{\Frm}^{2},
\end{align*}
where $\tilde{N}\coloneqq\min\{|\Sonetau|,|\Stwotau|\}\asymp N/\tmix$.
\item By Lemma~\ref{lem:bounded_state}, with probability at least $1-\delta$,
all $\{\xt,\xtilt,\zt,\ztilt\}$ terms involved in the definition
of $\Dtau_{Y}$ has $\ell_{2}$ norm bounded by $\sqrt{\Gmax}\poly(d,\ka,\log(\Tt/\delta))$.
This implies that
\[
\big|\Dtau_{Y}\big|\le\Gmax^{2}\cdot\poly\Big(d,\ka,\log\frac{\Tt}{\delta}\Big)\cdot\rho^{\tmix-1}.
\]
\end{itemize}

\paragraph{Step 2.3: putting pieces together.}

Putting the subscript $g$ back in, we have already shown that
\begin{align*}
\statYg & =\sum_{\tau_{1}=1}^{\tmix}\sum_{\tau_{2}=1}^{\tmix}\wtau\cdot\stattau_{Y,g}\\
 & =\underset{\eqqcolon\stattil_{Y,g}}{\underbrace{\sum_{\tau_{1}=1}^{\tmix}\sum_{\tau_{2}=1}^{\tmix}\wtau\cdot\stattautil_{Y,g}}}+\underset{\eqqcolon\Delta_{Y,g}}{\underbrace{\sum_{\tau_{1}=1}^{\tmix}\sum_{\tau_{2}=1}^{\tmix}\wtau\cdot\Dtau_{Y,g}}}=\stattil_{Y,g}+\Delta_{Y,g},
\end{align*}
where
\begin{align*}
\E\big[\stattil_{Y,g}\big] & =\big\|\bU^{\top}\vc((\Yktmix-\Yltmix)^{\top})\big\|_{2}^{2},\\
\var\big(\stattil_{Y,g}\big) & \le\max_{\tau_{1},\tau_{2}}\var(\stattautil_{Y,g})\lesssim\left(\frac{\Gmax^{2}\ka^{2}}{\tilde{N}}\right)^{2}dK+\frac{\Gmax^{2}\ka^{2}}{\tilde{N}}\big\|\Yktmix-\Yltmix\big\|_{\Frm}^{2},\\
\big|\Delta_{Y,g}\big| & \le\max_{\tau_{1},\tau_{2}}|\Dtau_{Y,g}|\le\Gmax^{2}\cdot\poly\Big(d,\ka,\log\frac{\Tt}{\delta}\Big)\cdot\rho^{\tmix-1}.
\end{align*}

So far in Step 2, we have focused on the analysis of $\statYg$. We
can easily adapt the argument to study $\statGammag$ as well; the
major difference is that, in Step 2.2, we should apply Lemmas~\ref{lem:stat_E_var}
and \ref{lem:mu_Sigma_simp} with $(\Gtilk,\Gtill)=(\Gammaktmix,\Gammaltmix)$,
$(\Wtilk,\Wtill)=(\boldsymbol{0},\boldsymbol{0})$, $(\Atilk,\Atill)=(\Id,\Id)$,
and subspaces $\{\Ui\}$ replaced by $\{\Vi\}$ instead. The final
result is that, for all $1\le g\le G$,
\begin{align*}
\statGammag & =\stattil_{\Gamma,g}+\Delta_{\Gamma,g},
\end{align*}
where
\begin{align*}
\E\big[\stattil_{\Gamma,g}\big] & =\big\|\bV^{\top}\vc((\Gammaktmix-\Gammaltmix)^{\top})\big\|_{2}^{2},\\
\var\big(\stattil_{\Gamma,g}\big) & \lesssim\left(\frac{\Gmax^{2}\ka^{2}}{\tilde{N}}\right)^{2}dK+\frac{\Gmax^{2}\ka^{2}}{\tilde{N}}\big\|\Gammaktmix-\Gammaltmix\big\|_{\Frm}^{2},\\
\big|\Delta_{\Gamma,g}\big| & \le\Gmax^{2}\cdot\poly\Big(d,\ka,\log\frac{\Tt}{\delta}\Big)\cdot\rho^{\tmix-1}.
\end{align*}

\paragraph{Step 3: analysis of the medians, and final results.}

Recall the following standard result on the concentration of medians
(or median-of-means in general) \cite[Theorem 2]{lugosi2019mean}. 
\begin{prop}
[Concentration of medians] \label{prop:median} Let $X_{1},\dots X_{G}$
be i.i.d. random variables with mean $\mu$ and variance $\sigma^{2}$.
Then we have $|\mathsf{median}\{X_{g},1\le g\le G\}-\mu|\le2\sigma$
with probability at least $1-e^{-c_{0}G}$ for some constant $c_{0}$.
\end{prop}
Notice that by construction, $\{\stattil_{Y,g}\}_{1\le g\le G}$ are
i.i.d. (and so are $\{\stattil_{\Gamma,g}\}_{1\le g\le G}$). Applying
Proposition~\ref{prop:median} to our case, we know that if $G\gtrsim\log(1/\delta)$,
then with probability at least $1-\delta$, the following holds:
\begin{itemize}
\item If $k=\ell$, i.e.~the two trajectories are generated by the same
LDS model, then 
\begin{align}
 & \median\{\statGammag,1\le g\le G\}+\median\{\statYg,1\le g\le G\}\nonumber \\
 & \qquad\le\median\{\stattil_{\Gamma,g}\}+\median\{\stattil_{Y,g}\}+\max_{g}|\Delta_{\Gamma,g}|+\max_{g}|\Delta_{Y,g}|\nonumber \\
 & \qquad\le2\sqrt{\var(\stattil_{\Gamma,g})}+2\sqrt{\var(\stattil_{Y,g})}+\Gmax^{2}\cdot\poly\Big(d,\ka,\log\frac{\Tt}{\delta}\Big)\cdot\rho^{\tmix-1}\nonumber \\
 & \qquad\lesssim\frac{\Gmax^{2}\ka^{2}\sqrt{dK}}{\Ntil}+\Gmax^{2}\cdot\poly\Big(d,\ka,\log\frac{\Tt}{\delta}\Big)\cdot\rho^{\tmix-1};\label{eq:median_H0}
\end{align}
\item On the other hand, if $k\neq\ell$, then
\begin{align}
 & \median\{\statGammag,1\le g\le G\}+\median\{\statYg,1\le g\le G\}\nonumber \\
 & \qquad\ge\median\{\stattil_{\Gamma,g}\}+\median\{\stattil_{Y,g}\}-\Big(\max_{g}|\Delta_{\Gamma,g}|+\max_{g}|\Delta_{Y,g}|\Big)\nonumber \\
 & \qquad\ge\E[\stattil_{\Gamma,g}]+\E[\stattil_{Y,g}]-2\Big(\sqrt{\var(\stattil_{\Gamma,g})}+\sqrt{\var(\stattil_{Y,g})}\Big)-\Gmax^{2}\cdot\poly\Big(d,\ka,\log\frac{\Tt}{\delta}\Big)\cdot\rho^{\tmix-1}\nonumber \\
 & \qquad\ge\Big\|\bV^{\top}\vc((\Gammaktmix-\Gammaltmix)^{\top})\Big\|_{2}^{2}+\Big\|\bU^{\top}\vc((\Yktmix-\Yltmix)^{\top})\Big\|_{2}^{2}\nonumber \\
 & \qquad\qquad-C_{0}\bigg(\frac{\Gmax^{2}\ka^{2}\sqrt{dK}}{\Ntil}+\sqrt{\frac{\Gmax^{2}\ka^{2}}{\Ntil}}\Big(\|\Gammaktmix-\Gammaltmix\|_{\Frm}+\|\Yktmix-\Yltmix\|_{\Frm}\Big)\bigg)\nonumber \\
 & \qquad\qquad-\Gmax^{2}\cdot\poly\Big(d,\ka,\log\frac{\Tt}{\delta}\Big)\cdot\rho^{\tmix-1}.\label{eq:median_H1_old}
\end{align}
\end{itemize}
We need to further simplify the result (\ref{eq:median_H1_old}) for
the $k\neq\ell$ case. According to (\ref{eq:Gamma_truncated_approx})
and (\ref{eq:Y_truncated_approx}), we have
\begin{align*}
\|\Yk-\bY_{\tmix-1}^{(k)}\|_{\Frm} & \le\Gmax\sqrt{d}\ka^{3}\rho^{2(\tmix-1)}\eqqcolon\epsmix,\\
\|\Gammak-\bGamma_{\tmix-1}^{(k)}\|_{\Frm} & \le\Gmax\sqrt{d}\ka^{2}\rho^{2(\tmix-1)}\le\epsmix,
\end{align*}
which implies that
\begin{align*}
\|\Gammaktmix-\Gammaltmix\|_{\Frm} & \le\|\Gammak-\Gammal\|_{\Frm}+2\epsmix,\\
\Big\|\bV^{\top}\vc\big((\Gammaktmix-\Gammaltmix)^{\top}\big)\Big\|_{2}^{2} & \ge\max\Big\{\|\bV^{\top}\vc\big((\Gammak-\Gammal)^{\top}\big)\|_{2}-2\epsmix,0\Big\}^{2}\\
 & \ge\Big\|\bV^{\top}\vc\big((\Gammak-\Gammal)^{\top}\big)\Big\|_{2}^{2}-4\epsmix\Big\|\bV^{\top}\vc((\Gammak-\Gammal)^{\top})\Big\|_{2}\\
 & \ge\Big\|\bV^{\top}\vc\big((\Gammak-\Gammal)^{\top}\big)\Big\|_{2}^{2}-4\epsmix\|\Gammak-\Gammal\|_{\Frm}.
\end{align*}
We can do a similar analysis for $\|\Yktmix-\Yltmix\|_{\Frm}$ and
$\|\bU^{\top}\vc((\Yktmix-\Yltmix)^{\top})\|_{2}^{2}$. Moreover,
we claim (and prove later) that if the subspaces $\{\Vi,\Ui\}$ satisfy
the condition (\ref{eq:2a_subspace_error}) in the theorem, then
\begin{equation}
\Big\|\bV^{\top}\vc((\Gammak-\Gammal)^{\top})\Big\|_{2}^{2}+\Big\|\bU^{\top}\vc((\Yk-\Yl)^{\top})\Big\|_{2}^{2}\ge\frac{1}{2}\Big(\|\Gammak-\Gammal\|_{\Frm}^{2}+\|\Yk-\Yl\|_{\Frm}^{2}\Big).\label{eq:subspace_implication}
\end{equation}
Putting these back to~(\ref{eq:median_H1_old}), we have for the
$k\neq\ell$ case, 
\begin{align}
 & \median\{\statGammag,1\le g\le G\}+\median\{\statYg,1\le g\le G\}\nonumber \\
 & \qquad\ge\Big\|\bV^{\top}\vc\big((\Gammak-\Gammal)^{\top}\big)\Big\|_{2}^{2}+\Big\|\bU^{\top}\vc\big((\Yk-\Yl)^{\top}\big)\Big\|_{2}^{2}\nonumber \\
 & \qquad\qquad-4\epsmix\Big(\|\Gammak-\Gammal\|_{\Frm}+\|\Yk-\Yl\|_{\Frm}\Big)\nonumber \\
 & \qquad\qquad-C_{0}\bigg(\frac{\Gmax^{2}\ka^{2}\sqrt{dK}}{\Ntil}+\sqrt{\frac{\Gmax^{2}\ka^{2}}{\Ntil}}\Big(\|\Gammak-\Gammal\|_{\Frm}+\|\Yk-\Yl\|_{\Frm}+4\epsmix\Big)\bigg)\nonumber \\
 & \qquad\qquad-\Gmax^{2}\cdot\poly\Big(d,\ka,\log\frac{\Tt}{\delta}\Big)\cdot\rho^{\tmix-1}\nonumber \\
 & \qquad\overset{{\rm (i)}}{\ge}\frac{1}{2}\Big(\|\Gammak-\Gammal\|_{\Frm}^{2}+\|\Yk-\Yl\|_{\Frm}^{2}\Big)-0.01\Big(\|\Gammak-\Gammal\|_{\Frm}^{2}+\|\Yk-\Yl\|_{\Frm}^{2}\Big)\nonumber \\
 & \qquad\qquad-C_{1}\bigg(\frac{\Gmax^{2}\ka^{2}\sqrt{dK}}{\Ntil}+\sqrt{\frac{\Gmax^{2}\ka^{2}}{\Ntil}}\sqrt{\|\Gammak-\Gammal\|_{\Frm}^{2}+\|\Yk-\Yl\|_{\Frm}^{2}}\bigg)\nonumber \\
 & \qquad\qquad-\Gmax^{2}\cdot\poly\Big(d,\ka,\log\frac{\Tt}{\delta}\Big)\cdot\rho^{\tmix-1}\nonumber \\
 & \qquad\overset{{\rm (ii)}}{\ge}0.48\Big(\|\Gammak-\Gammal\|_{\Frm}^{2}+\|\Yk-\Yl\|_{\Frm}^{2}\Big)\nonumber \\
 & \qquad\qquad-C_{1}\bigg(\frac{\Gmax^{2}\ka^{2}\sqrt{dK}}{\Ntil}+\sqrt{\frac{\Gmax^{2}\ka^{2}}{\Ntil}}\sqrt{\|\Gammak-\Gammal\|_{\Frm}^{2}+\|\Yk-\Yl\|_{\Frm}^{2}}\bigg),\label{eq:median_H1}
\end{align}
where (i) holds if $\tmix\gtrsim\frac{1}{1-\rho}\log((\frac{\Gmax}{\Dsep}+2)d\ka)$
so that $\epsmix\le10^{-3}\Dsep$, and (ii) holds if $\tmix\gtrsim\frac{1}{1-\rho}\log((\frac{\Gmax}{\Dsep}+2)\frac{d\ka\Tt}{\delta})$
so that $\Gmax^{2}\cdot\poly(d,\ka,\log\frac{\Tt}{\delta})\cdot\rho^{\tmix-1}\le10^{-3}\Dsep^{2}$.

Putting (\ref{eq:median_H0}) and (\ref{eq:median_H1}) together,
we can finally check that, if it further holds that $\Ntil\asymp N/\tmix\gtrsim\frac{\Gmax^{2}\ka^{2}\sqrt{dK}}{\Dsep^{2}}+1$,
then we have with probability at least $1-\delta$,
\[
\median\{\statGammag\}+\median\{\statYg\}\begin{cases}
\le\frac{1}{8}\Dsep^{2} & \text{if}\quad k=\ell,\\
\ge\frac{3}{8}\big(\|\Gammak-\Gammal\|_{\Frm}^{2}+\|\Yk-\Yl\|_{\Frm}^{2}\big)\ge\frac{3}{8}\Dsep^{2} & \text{if}\quad k\neq\ell.
\end{cases}
\]
This together with our choice of testing threshold $\tau\in[\DGY^{2}/8,3\DGY^{2}/8]$
in Algorithm~\ref{alg:clustering} implies correct testing of the
two trajectories $\{\xt\},\{\zt\}$. Finally, taking the union bound
over all pairs of trajectories in $\Mclustering$ leads to correct
pairwise testing, which in turn implies exact clustering of $\Mclustering$;
this completes our proof of Theorem~(\ref{thm:clustering}).

\subsubsection{Proof of Lemma \ref{lem:stat_E_var}}

We first assume $|\Omegaone|=|\Omegatwo|=N$ for simplicity. Recall
that 
\begin{align*}
\stat & =\Big\langle\underset{\eqqcolon\ba}{\underbrace{\bU^{\top}\frac{1}{|\Omegaone|}\sum_{t\in\Omegaone}\vc\big(\xt(\xt')^{\top}-\zt(\zt')^{\top}\big)}},\underset{\eqqcolon\bb}{\underbrace{\bU^{\top}\frac{1}{|\Omegatwo|}\sum_{t\in\Omegatwo}\vc\big(\xt(\xt')^{\top}-\zt(\zt')^{\top}\big)}}\Big\rangle=\langle\ba,\bb\rangle,
\end{align*}
where $\ba,\bb\in\R^{dK}$ are i.i.d., and 
\begin{align*}
\E[\ba] & =\E\Big[\bU^{\top}\vc\big(\xt(\xt')^{\top}-\zt(\zt')^{\top}\big)\Big]=\bU^{\top}\vc\Big((\Ytilk-\Ytill)^{\top}\Big)=\mukl.
\end{align*}
Therefore, we have the expectation
\[
\E[\stat]=\big\langle\E[\ba],\E[\bb]\big\rangle=\|\mukl\|_{2}^{2}.
\]
It remains to compute the variance $\var(\stat)=\E[\stat^{2}]-\E[\stat]^{2}$,
where
\begin{equation}
\E[\stat^{2}]=\E\big[(\ba^{\top}\bb)^{2}\big]=\Tr\big(\E[\bb\bb^{\top}]\E[\ba\ba^{\top}]\big)=\Tr\big(\E[\ba\ba^{\top}]^{2}\big).\label{eq:Estat2}
\end{equation}
Here $\E[\ba\ba^{\top}]=\E[\ba]\E[\ba]^{\top}+\cov(\ba)$, and since
$\ba$ is an empirical average of $N$ i.i.d. random vectors, we have
\[
\cov(\ba)=\frac{1}{N}\cov(\boldsymbol{f}),\quad\text{where}\quad\boldsymbol{f}\coloneqq\bU^{\top}\vc\big(\xt(\xt')^{\top}-\zt(\zt')^{\top}\big)\in\R^{dK}.
\]
For now, we claim that 
\begin{equation}
\cov(\boldsymbol{f})=\bU^{\top}(\Sigk+\Wtilk\otimes\Gtilk+\Sigl+\Wtill\otimes\Gtill)\bU=\Sigkl,\label{eq:claim_covf}
\end{equation}
which will be proved soon later. Putting these back to~(\ref{eq:Estat2}),
one has
\begin{align*}
\E[\ba\ba^{\top}] & =\cov(\ba)+\E[\ba]\E[\ba]^{\top}=\frac{1}{N}\Sigkl+\mukl\mukl^{\top},\\
\E[\ba\ba^{\top}]^{2} & =\frac{1}{N^{2}}\Sigkl^{2}+\frac{1}{N}(\Sigkl\mukl\mukl^{\top}+\mukl\mukl^{\top}\Sigkl)+\|\mukl\|_{2}^{2}\mukl\mukl^{\top},
\end{align*}
and finally
\begin{align*}
\var(\stat) & =\E[\stat^{2}]-\E[\stat]^{2}=\Tr(\E[\ba\ba^{\top}]^{2})-\|\mukl\|_{2}^{4}=\frac{1}{N^{2}}\Tr(\Sigkl^{2})+\frac{2}{N}\mukl^{\top}\Sigkl\mukl,
\end{align*}
which completes our calculation of the variance for the case of $|\Omegaone|=|\Omegatwo|=N$.
For the more general case where (without loss of generality) $|\Omegaone|=N\le|\Omegatwo|$,
we simply need to modify the equation~(\ref{eq:Estat2}) to an inequality
$\E[\stat^{2}]=\Tr(\E[\bb\bb^{\top}]\E[\ba\ba^{\top}])\le\Tr(\E[\ba\ba^{\top}]^{2})$,
and the remaining analysis is the same. This finishes the proof of
Lemma~\ref{lem:stat_E_var}.
\begin{proof}
[Proof of (\ref{eq:claim_covf})] For notational simplicity, we drop
the subscripts $t$ in the definition of $\boldsymbol{f}$. Then we
have $\boldsymbol{f}=\bU^{\top}\vc\big(\bx(\bx')^{\top}-\bz(\bz')^{\top}\big)$,
and hence
\begin{align}
\cov(\boldsymbol{f}) & =\bU^{\top}\cov\Big(\vc\big(\bx(\bx')-\bz(\bz')^{\top}\big)\Big)\bU=\bU^{\top}\Big(\cov\big(\vc(\bx(\bx')^{\top})\big)+\cov\big(\vc(\bz(\bz')^{\top})\big)\Big)\bU,\label{eq:covf}
\end{align}
where the second equality uses the independence between $(\bx,\bx')$
and $(\bz,\bz')$. 

Let us focus on $\cov\big(\vc(\bx(\bx')^{\top})\big)$. For notational
simplicity, for now we rewrite $\Atilk,\Wtilk,\Gtilk$ as $\bA,\bW,\boldsymbol{\Gamma}$.
Define $\by\coloneqq\bGamma^{-1/2}\bx\sim\Ncal(\boldsymbol{0},\Id)$,
and recall that $\bx'=\bA\bx+\bw$ where $\bw\sim\Ncal(\mathbf{0},\bW)$.
Using the fact that $\vc(\bA\bB\bC)=(\bC^{\top}\otimes\bA)\vc(\bB)$
for any matrices of compatible shape, we have
\begin{align*}
\bg & \coloneqq\vc\Big(\bx\big(\bx'\big)^{\top}\Big)=\vc(\bx\bx^{\top}\bA^{\top})+\vc(\bx\bw^{\top})\\
 & =(\bA\otimes\Id)\vc(\bx\bx^{\top})+\vc(\bx\bw^{\top})=(\bA\otimes\Id)\vc(\bGamma^{1/2}\by\by^{\top}\bGamma^{1/2})+\vc(\bx\bw^{\top})\\
 & =\underset{\eqqcolon\bg_{1}}{\underbrace{(\bA\otimes\Id)(\bGamma^{1/2}\otimes\bGamma^{1/2})\vc(\by\by^{\top})}}+\underset{\eqqcolon\bg_{2}}{\underbrace{\vc(\bx\bw^{\top})}}=\bg_{1}+\bg_{2}.
\end{align*}
Note that $\E[\bg_{2}]=\boldsymbol{0}$, $\E[\bg]=\E[\bg_{1}]$, and
$\E[\bg_{1}\bg_{2}^{\top}]=\boldsymbol{0}$. Hence
\begin{equation}
\cov(\bg)=\E[\bg\bg^{\top}]-\E[\bg]\E[\bg]^{\top}=\cov(\bg_{1})+\E[\bg_{2}\bg_{2}^{\top}].\label{eq:covg}
\end{equation}
For the second term on the right-hand side, we have
\[
\bg_{2}=\vc(\bx\bw^{\top})=\begin{bmatrix}w_{1}\bx\\
\vdots\\
w_{d}\bx
\end{bmatrix},\quad\E[\bg_{2}\bg_{2}^{\top}]=\E\Big[[w_{i}w_{j}\bx\bx^{\top}]_{1\le i,j\le d}\Big]=\bW\otimes\bGamma;
\]
as for the first term, we claim (and prove soon later) that 
\begin{equation}
\cov\Big(\vc(\by\by^{\top})\Big)=\bI_{d^{2}}+\bP,\label{eq:cov_vec}
\end{equation}
which implies that
\begin{align*}
\cov(\bg_{1}) & =\cov\Big((\bA\otimes\Id)(\bGamma^{1/2}\otimes\bGamma^{1/2})\vc(\by\by^{\top})\Big)\\
 & =(\bA\otimes\Id)(\bGamma^{1/2}\otimes\bGamma^{1/2})\cov\Big(\vc(\by\by^{\top})\Big)(\bGamma^{1/2}\otimes\bGamma^{1/2})(\bA^{\top}\otimes\Id)\\
 & =(\bA\otimes\Id)(\bGamma^{1/2}\otimes\bGamma^{1/2})(\bI_{d^{2}}+\bP)(\bGamma^{1/2}\otimes\bGamma^{1/2})(\bA^{\top}\otimes\Id).
\end{align*}
Putting these back to (\ref{eq:covg}), one has
\[
\cov\big(\vc(\bx(\bx')^{\top})\big)=\cov(\bg)=(\bA\otimes\Id)(\bGamma^{1/2}\otimes\bGamma^{1/2})(\bI_{d^{2}}+\bP)(\bGamma^{1/2}\otimes\bGamma^{1/2})(\bA^{\top}\otimes\Id)+\bW\otimes\bGamma,
\]
which is equal to $\Sigk+\Wtilk\otimes\Gtilk$ if we return to the
original notation of $\Atilk,\Wtilk,\Gtilk$. By a similar analysis,
we can show that $\cov\big(\vc(\bz(\bz')^{\top})\big)=\Sigl+\Wtill\otimes\Gtill$.
Putting these back to (\ref{eq:covf}) finishes our calculation of
$\cov(\boldsymbol{f})$.

Finally, it remains to prove (\ref{eq:cov_vec}). Denote
\[
\bu\coloneqq\vc(\by\by^{\top})=\begin{bmatrix}y_{1}\by\\
\vdots\\
y_{d}\by
\end{bmatrix}.
\]
Then $\E[\bu]=\vc(\Id)$, and thus $\E[\bu]\E[\bu]^{\top}=[\be_{i}\be_{j}^{\top}]_{1\le i,j\le d}$.
Next, consider
\[
\E[\bu\bu^{\top}]=\E\Big[[y_{i}y_{j}\by\by^{\top}]_{1\le i,j\le d}\Big].
\]
\begin{itemize}
\item For $i=j$,
\[
\E\big[y_{i}^{2}y_{k}y_{\ell}\big]=\begin{cases}
3 & \text{if}\quad k=\ell=i,\\
1 & \text{if}\quad k=\ell\neq i,\\
0 & \text{if}\quad k\neq\ell,
\end{cases}
\]
and hence $\E[y_{i}^{2}\by\by^{\top}]=\Id+2\be_{i}\be_{i}^{\top}.$ 
\item For $i\neq j$,
\[
\E\big[y_{i}y_{j}y_{k}y_{\ell}\big]=\begin{cases}
1 & \text{if}\quad k=i,\ell=j\quad\text{or}\quad k=j,\ell=i,\\
0 & \text{otherwise},
\end{cases}
\]
and hence $\E[y_{i}y_{j}\by\by^{\top}]=\be_{i}\be_{j}^{\top}+\be_{j}\be_{i}^{\top}$. 
\end{itemize}
Putting together, the $(i,j)$-th $d\times d$ block of $\cov(\bu)=\E[\bu\bu^{\top}]-\E[\bu]\E[\bu]^{\top}$
is equal to $\Id+\be_{i}\be_{i}^{\top}$ if $i=j$, and $\be_{j}\be_{i}^{\top}$
if $i\neq j$. In other words, $\cov(\vc(\by\by^{\top}))=\cov(\bu)=\bI_{d^{2}}+\bP$,
where $\bP=[\be_{j}\be_{i}^{\top}]_{1\le i,j\le d}$ is a symmetric
permutation matrix; this completes our proof of~(\ref{eq:cov_vec}).
\end{proof}

\subsubsection{Proof of Lemma \ref{lem:mu_Sigma_simp}}

First, it holds that
\[
\|\mukl\|_{2}^{2}=\sumi\Big\|\Ui^{\top}\big((\Ytilk)_{i}-(\Ytill)_{i}\big)\Big\|_{2}^{2}\le\sumi\Big\|(\Ytilk)_{i}-(\Ytill)_{i}\Big\|_{2}^{2}=\|\Ytilk-\Ytill\|_{\Frm}^{2},
\]
which gives the upper bound on $\E[\stat]=\|\mukl\|_{2}^{2}$. For
the lower bound, the triangle inequality tells us that
\begin{align*}
\Big\|\Ui^{\top}\big((\Ytilk)_{i}-(\Ytill)_{i}\big)\Big\|_{2} & =\Big\|\Ui\Ui^{\top}(\Ytilk)_{i}-\Ui\Ui^{\top}(\Ytill)_{i}\Big\|_{2}\ge\max\Big\{\|(\Ytilk)_{i}-(\Ytill)_{i}\|_{2}-2\epsilon,0\Big\},
\end{align*}
which implies that
\[
\Big\|\Ui^{\top}\big((\Ytilk)_{i}-(\Ytill)_{i}\big)\Big\|_{2}^{2}\ge\|(\Ytilk)_{i}-(\Ytill)_{i}\|_{2}^{2}-4\epsilon\|(\Ytilk)_{i}-(\Ytill)_{i}\|_{2},
\]
and hence
\begin{align*}
\|\mukl\|_{2}^{2} & =\sumi\Big\|\Ui^{\top}\big((\Ytilk)_{i}-(\Ytill)_{i}\big)\Big\|_{2}^{2}\ge\sumi\|(\Ytilk)_{i}-(\Ytill)_{i}\|_{2}^{2}-4\epsilon\sumi\|(\Ytilk)_{i}-(\Ytill)_{i}\|_{2}\\
 & =\|\Ytilk-\Ytill\|_{\Frm}^{2}-4\epsilon\sumi\|(\Ytilk)_{i}-(\Ytill)_{i}\|_{2}.
\end{align*}

It remains to upper bound $\Sigkl$. Recall the definition 
\[
\Sigkl=\bU^{\top}\big(\Sigk+\Wtilk\otimes\Gtilk+\Sigl+\Wtill\otimes\Gtill\big)\bU.
\]
We will utilize the following basic facts: (1) for square matrices
$\bA$ and $\bB$ with eigenvalues $\{\lambda_{i}\}$ and $\{\mu_{j}\}$
respectively, their Kronecker product $A\otimes B$ has eigenvalues
$\{\lambda_{i}\mu_{j}\}$; (2) For matrices $\bA,\bB,\bC,\bD$ of
compatible shapes, it holds that $(\bA\otimes\bB)(\bC\otimes\bD)=(\bA\bC)\otimes(\bB\bD)$.
These imply that
\[
\boldsymbol{0}\preccurlyeq\Wtilk\otimes\Gtilk\preccurlyeq\|\Wtilk\|\|\Gtilk\|\bI_{d^{2}}\preccurlyeq\lup\Gmax\bI_{d^{2}},
\]
and
\begin{align*}
\boldsymbol{0}\preccurlyeq\Sigk & =\Big(\Atilk\otimes\Id\Big)\Big((\Gtilk)^{1/2}\otimes(\Gtilk)^{1/2}\Big)\Big(\bI_{d^{2}}+\bP\Big)\Big((\Gtilk)^{1/2}\otimes(\Gtilk)^{1/2}\Big)\Big((\Atilk)^{\top}\otimes\Id\Big)\\
 & \preccurlyeq2\Big(\Atilk\otimes\Id\Big)\Big((\Gtilk)^{1/2}\otimes(\Gtilk)^{1/2}\Big)\Big((\Gtilk)^{1/2}\otimes(\Gtilk)^{1/2}\Big)\Big((\Atilk)^{\top}\otimes\Id\Big)\\
 & =2\Big(\Atilk\otimes\Id\Big)\Big(\Gtilk\otimes\Gtilk\Big)\Big((\Atilk)^{\top}\otimes\Id\Big)\\
 & \preccurlyeq2\Gmax^{2}\Big(\Atilk\otimes\Id\Big)\Big((\Atilk)^{\top}\otimes\Id\Big)\\
 & =2\Gmax^{2}\Big(\Atilk(\Atilk)^{\top}\Big)\otimes\Id\preccurlyeq2\Gmax^{2}\ka^{2}\bI_{d^{2}}.
\end{align*}
Using the conditions $\lup\le\Gmax$ and $\ka\ge1$, we have
\[
\Sigk+\Wtilk\otimes\Gtilk\preccurlyeq(\lup\Gmax+2\Gmax^{2}\ka^{2})\bI_{d^{2}}\preccurlyeq3\Gmax^{2}\ka^{2}\bI_{d^{2}}.
\]
We can upper bound $\Sigl+\Wtill\otimes\Gtill$ by the same analysis.
As a result,
\[
\Sigkl\preccurlyeq\bU^{\top}(6\Gmax^{2}\ka^{2}\bI_{d^{2}})\bU=6\Gmax^{2}\ka^{2}\bU^{\top}\bU=6\Gmax^{2}\ka^{2}\bI_{dK},
\]
which finishes the proof of Lemma \ref{lem:mu_Sigma_simp}.

\subsubsection{Proof of (\ref{eq:subspace_implication})}

Let $\epsilon$ be the right-hand side of the condition (\ref{eq:2a_subspace_error})
on the subspaces. Then, applying the second point of Lemma \ref{lem:mu_Sigma_simp}
(with $\Ytilk=\Yk,\Ytill=\Yl$) tells us that 
\begin{align*}
\Big\|\bU^{\top}\vc\big((\Yk-\Yl)^{\top}\big)\Big\|_{2}^{2} & \ge\big\|\Yk-\Yl\big\|_{\Frm}^{2}-4\epsilon\sumi\big\|(\Yk)_{i}-(\Yl)_{i}\big\|_{2}\\
 & \ge\|\Yk-\Yl\|_{\Frm}^{2}-4\epsilon\sqrt{d}\|\Yk-\Yl\|_{\Frm},
\end{align*}
where the last line follows from the Cauchy-Schwarz inequality:
\[
\sumi\big\|(\Yk)_{i}-(\Yl)_{i}\big\|_{2}\le\sqrt{d\sumi\big\|(\Yk)_{i}-(\Yl)_{i}\big\|_{2}^{2}}=\sqrt{d}\big\|\Yk-\Yl\big\|_{\Frm}.
\]
We can lower bound $\|\bV^{\top}\vc((\Gammak-\Gammal)^{\top})\|_{2}^{2}$
similarly. Putting pieces together, we have
\begin{align*}
 & \Big\|\bV^{\top}\vc\big((\Gammak-\Gammal)^{\top}\big)\Big\|_{2}^{2}+\Big\|\bU^{\top}\vc\big((\Yk-\Yl)^{\top}\big)\Big\|_{2}^{2}\\
 & \quad\ge\|\Gammak-\Gammal\|_{\Frm}^{2}+\|\Yk-\Yl\|_{\Frm}^{2}-4\epsilon\sqrt{d}\big(\|\Gammak-\Gammal\|_{\Frm}+\|\Yk-\Yl\|_{\Frm}\big)\\
 & \quad\ge\frac{1}{2}\Big(\|\Gammak-\Gammal\|_{\Frm}^{2}+\|\Yk-\Yl\|_{\Frm}^{2}\Big),
\end{align*}
where the last inequality is due to the assumption $\epsilon\lesssim\DGY/\sqrt{d}$.
This completes our proof of (\ref{eq:subspace_implication}).

\subsection{\label{subsec:proof_LS}Proof of Theorem~\ref{thm:LS}}

It suffices to prove the error bounds for one specific value of $k$,
and then take the union bound over $1\le k\le K$. For notational
convenience, in this proof we rewrite $\Ttk,\Ak,\Akhat,\Wk,\Wkhat$
as $T,\bA,\Ahat,\bW,\What$, respectively. We will investigate the
close-form solution $\Ahat$, and prepare ourselves with a self-normalized
concentration bound; this will be helpful in finally proving the error
bounds for $\|\Ahat-\bA\|$ and $\|\What-\bW\|$.

\paragraph{Step 1: preparation.}

Recall the least-squares solution
\[
\Ahat=\Big(\sum_{m\in\Ck}\sum_{0\le t\le T_{m}-1}\xmtp\xmt^{\top}\Big)\Big(\sum_{m\in\Ck}\sum_{0\le t\le T_{m}-1}\xmt\xmt^{\top}\Big)^{-1}.
\]
Using the notation 
\[
\bX\coloneqq\begin{bmatrix}\vdots\\
\xmt^{\top}\\
\vdots
\end{bmatrix}_{0\le t\le T_{m}-1,m\in\Ck}\in\R^{T\times d},\quad\bX_{+}\coloneqq\begin{bmatrix}\vdots\\
\xmtp^{\top}\\
\vdots
\end{bmatrix},\quad\bN\coloneqq\begin{bmatrix}\vdots\\
\wmt^{\top}\\
\vdots
\end{bmatrix},
\]
we have
\[
\bX_{+}^{\top}=\bA\bX^{\top}+\bN^{\top},\quad\Ahat=\bX_{+}^{\top}\bX(\XTX)^{-1}=\bA+\NTX(\XTX)^{-1},
\]
namely 
\begin{equation}
\DelA\coloneqq\Ahat-\bA=\NTX(\XTX)^{-1}.\label{eq:Ahat_decomp}
\end{equation}
We will utilize the following matrix form of self-normalized concentration
\cite[Lemma C.4]{kakade2020information}.
\begin{lem}
[Self-normalized concentration, matrix form] \label{lem:self_norm_matrix}
Consider filtrations $\{\Ft\},$ and random vectors $\{\xt,\zt\}$
satisfying $\xt\in\Ftm$ and $\zt|\Ftm\sim\Ncal(\boldsymbol{0},\Id)$.
Let $\bV\in\R^{d\times d}$ be a fixed, symmetric, positive definite
matrix, and denote $\VTbar\coloneqq\sumt\xt\xt^{\top}+\bV$. Then
with probability at least $1-\delta$,
\[
\bigg\|(\VTbar)^{-1/2}\Big(\sumt\xt\zt^{\top}\Big)\bigg\|\lesssim\sqrt{d+\log\frac{1}{\delta}+\log\frac{\det(\VTbar)}{\det(\bV)}}.
\]
\end{lem}

\paragraph{Step 2: estimation error of $\protect\Ahat$.}

Let us rewrite (\ref{eq:Ahat_decomp}) as
\begin{equation}
\DelA^{\top}=(\XTX)^{-1/2}\cdot(\XTX)^{-1/2}\XTN.\label{eq:Ahat_error_decomp}
\end{equation}
It is obvious that $\XTX$ plays a crucial role. Recall from Lemma
\ref{lem:bounded_state} that with probability at least $1-\delta$,
$\|\xmt\|_{2}\le\Dvec$ for some $\Dvec\lesssim\sqrt{\Gmax}\cdot\poly(d,\ka,\log(\Tt/\delta))$.
Then trivially we have the upper bound
\[
\XTX\preccurlyeq\Dvec^{2}T\cdot\Id\eqqcolon\Vup.
\]
For a lower bound, we claim (and prove later) that with probability
at least $1-\delta$, 
\begin{equation}
\XTX\succcurlyeq\frac{1}{5}T\cdot\bW\eqqcolon\Vlb,\quad\text{provided that}\quad T\gtrsim\kwself^{2}d\cdot\log\Big(\frac{\Gmax}{\Wmin}\frac{\ka d\Tt}{\delta}\Big).\label{eq:VT_lower_bound}
\end{equation}

Now we are ready to control $\|\Ahat-\bA\|=\|\DelA\|$. From (\ref{eq:Ahat_error_decomp}),
we have
\[
\|\DelA\|\le\|(\XTX)^{-1/2}\|\cdot\|(\XTX)^{-1/2}\XTN\|.
\]
First, the lower bound (\ref{eq:VT_lower_bound}) on $\XTX$ tells
us that $\|(\XTX)^{-1/2}\|\lesssim1/\sqrt{T\cdot\lambmin(\bW)}$.
Moreover, applying Lemma \ref{lem:self_norm_matrix} with $\bV=\Vlb$,
one has with probability at least $1-\delta$,
\begin{align}
\|(\XTX)^{-1/2}\XTN\| & \lesssim\|(\XTX+\bV)^{-1/2}\XTN\|\lesssim\|\bW^{1/2}\|\sqrt{d+\log\frac{1}{\delta}+\log\frac{\det(\Vup+\Vlb)}{\det(\Vlb)}}\nonumber \\
 & \lesssim\sqrt{\|\bW\|}\sqrt{d\cdot\log\big(\frac{\Gmax}{\Wmin}\frac{d\ka\Tt}{\delta}\big)}.\label{eq:self_norm}
\end{align}
Putting these together, we have
\[
\|\DelA\|\lesssim\frac{1}{\sqrt{T\cdot\lambmin(\bW)}}\cdot\sqrt{\|\bW\|}\sqrt{d\cdot\log\big(\frac{\Gmax}{\Wmin}\frac{d\ka\Tt}{\delta}\big)}\lesssim\sqrt{\frac{d\cdot\kwself}{T}\log\big(\frac{\Gmax}{\Wmin}\frac{d\ka\Tt}{\delta}\big)},
\]
which proves our upper bound for $\|\Ahat-\bA\|$ in the theorem.

\paragraph{Step 3: estimation error of $\protect\What$.}

By the definition of $\wmthat=\xmtp-\Ahat\xmt=\wmt-\DelA\xmt$, we
have
\begin{align*}
\Nhat & \coloneqq\begin{bmatrix}\vdots\\
\wmthat^{\top}\\
\vdots
\end{bmatrix}\in\R^{T\times d},\\
\Nhat^{\top} & =\bN^{\top}-\DelA\bX^{\top}=\bN^{\top}-\NTX(\XTX)^{-1}\bX^{\top}=\bN^{\top}\big(\bI_{T}-\bX(\XTX)^{-1}\bX^{\top}\big).
\end{align*}
Notice that $\bI_{T}-\bX(\XTX)^{-1}\bX^{\top}$ is a symmetric projection
matrix. Therefore,
\[
\What=\frac{1}{T}\Nhat^{\top}\Nhat=\frac{1}{T}\bN^{\top}\big(\bI_{T}-\bX(\XTX)^{-1}\bX^{\top}\big)\bN,
\]
and thus
\begin{equation}
\What-\bW=\Big(\frac{1}{T}\bN^{\top}\bN-\bW\Big)-\frac{1}{T}\NTX(\XTX)^{-1}\XTN.\label{eq:What_W}
\end{equation}
The first term on the right-hand side can be controlled by a standard
result for covariance estimation (see Proposition \ref{prop:covariance_concentration}):
with probability at least $1-\delta$,
\[
\Big\|\frac{1}{T}\bN^{\top}\bN-\bW\Big\|\lesssim\|\bW\|\sqrt{\frac{d+\log\frac{1}{\delta}}{T}}.
\]
As for the second term, we rewrite it as 
\[
\frac{1}{T}\NTX(\XTX)^{-1}\XTN=\frac{1}{T}\Big((\XTX)^{-1/2}\XTN\Big)^{\top}\Big((\XTX)^{-1/2}\XTN\Big).
\]
Applying our earlier self-normalized concentration bound (\ref{eq:self_norm}),
we have 
\[
\Big\|\frac{1}{T}\NTX(\XTX)^{-1}\XTN\Big\|\lesssim\frac{\|\bW\|d\cdot\log(\frac{\Gmax}{\Wmin}\frac{d\ka\Tt}{\delta})}{T}.
\]
Putting these back to (\ref{eq:What_W}), we have
\begin{align*}
\|\What-\bW\| & \le\|\frac{1}{T}\bN^{\top}\bN-\bW\|+\|\frac{1}{T}\NTX(\XTX)^{-1}\XTN\|\\
 & \lesssim\|\bW\|\sqrt{\frac{d+\log\frac{1}{\delta}}{T}}+\|\bW\|\frac{d\cdot\log(\frac{\Gmax}{\Wmin}\frac{d\ka\Tt}{\delta})}{T}\lesssim\|\bW\|\sqrt{\frac{d\cdot\log(\frac{\Gmax}{\Wmin}\frac{d\ka\Tt}{\delta})}{T}},
\end{align*}
where the last inequality uses $T\gtrsim d\cdot\log(\frac{\Gmax}{\Wmin}\frac{d\ka\Tt}{\delta})$.
This finishes our proof of Theorem \ref{thm:LS}.

\subsubsection{Proof of (\ref{eq:VT_lower_bound})}

We start with the following decomposition:
\begin{align}
\XTX & =\sum_{m\in\Ck}\sum_{0\le t\le T_{m}-1}\xmt\xmt^{\top}\succcurlyeq\sum_{m\in\Ck}\sum_{1\le t\le T_{m}-1}\xmt\xmt^{\top}\nonumber \\
 & =\sum_{m\in\Ck}\sum_{0\le t\le T_{m}-2}\xmtp\xmtp^{\top}=\sum_{m\in\Ck}\sum_{0\le t\le T_{m}-2}(\bA\xmt+\wmt)(\bA\xmt+\wmt)^{\top}\nonumber \\
 & =\underset{\coloneqq\bP}{\underbrace{\sum_{m\in\Ck}\sum_{0\le t\le T_{m}-2}\wmt\wmt^{\top}}}\nonumber \\
 & \qquad+\underset{\eqqcolon\bQ}{\underbrace{\sum_{m\in\Ck}\sum_{0\le t\le T_{m}-2}\Big(\bA\xmt\xmt^{\top}\bA^{\top}+\bA\xmt\wmt^{\top}+\wmt\xmt^{\top}\bA^{\top}\Big)}}\nonumber \\
 & =\bP+\bQ.\label{eq:VT_decomp}
\end{align}

\paragraph{Lower bound for $\bP$.}

By Proposition \ref{prop:covariance_concentration}, we have with
probability at least $1-\delta$,
\[
\|\frac{1}{T-|\Ck|}\bP-\bW\|\lesssim\|\bW\|\sqrt{\frac{d+\log\frac{1}{\delta}}{T}}\lesssim\lambmin(\bW)\cdot\Id,\quad\text{provided that}\quad T\gtrsim\kwself^{2}(d+\log\frac{1}{\delta}).
\]
As a result, we have $\frac{1}{T-|\Ck|}\bP\succcurlyeq\frac{1}{2}\bW$,
which implies $\bP\succcurlyeq\frac{1}{4}T\cdot\bW$.

\paragraph{Lower bound for $\bQ$.}

Let $\Ncal_{\epsilon}$ be an $\epsilon$-net of the unit sphere $\Scal^{d-1}$
(where the value of $0<\epsilon<1$ will be specified later), and
let $\pi$ be the projection onto $\Ncal_{\epsilon}$. Recall the
standard result that $|\Ncal_{\epsilon}|\le(9/\epsilon)^{d}$. Moreover,
for any $\bv\in\Scal^{d-1}$, denote $\bDel_{\bv}\coloneqq\bv-\pi(\bv)$,
which satisfies $\|\bDel_{\bv}\|_{2}\le\epsilon$. Then we have
\begin{align}
\lambmin(\bQ) & =\inf_{\bv\in\Scal^{d-1}}\bv^{\top}\bQ\bv=\inf_{\bv\in\Scal^{d-1}}(\pi(\bv)+\bDel_{\bv})^{\top}\bQ(\pi(\bv)+\bDel_{\bv})\nonumber \\
 & \ge\inf_{\bv\in\Scal^{d-1}}\pi(\bv)^{\top}\bQ\pi(\bv)-(2\epsilon+\epsilon^{2})\|\bQ\|\ge\inf_{\bv\in\Ncal_{\epsilon}}\bv^{\top}\bQ\bv-3\epsilon\|\bQ\|.\label{eq:lambmin_Q}
\end{align}
For $\|\bQ\|$, we simply use a crude upper bound, based on the boundedness
of $\{\|\xmt\|_{2}\}$ (Lemma \ref{lem:bounded_state}): with probability
at least $1-\delta$, one has
\begin{equation}
\|\bQ\|\le\sum_{m\in\Ck}\sum_{0\le t\le T_{m}-2}\Big(\|\bA\xmt\|_{2}^{2}+2\|\bA\xmt\|_{2}\|\wmt\|_{2}\Big)\lesssim\Gmax\cdot\poly\Big(\ka,d,\Tt,\frac{1}{\delta}\Big).\label{eq:Q_upper_bound}
\end{equation}

Next, we lower bound $\inf_{\bv\in\Ncal_{\epsilon}}\bv^{\top}\bQ\bv$.
First, consider a fixed $\bv\in\Ncal_{\epsilon}$; denoting $\ymt\coloneqq\bA\xmt$
and $\umt\coloneqq\bW^{-1/2}\wmt\sim\Ncal(\mathbf{0},\Id)$, we have
\[
\bv^{\top}\bQ\bv=\sum_{m\in\Ck}\sum_{0\le t\le T_{m}-2}(\bv^{\top}\ymt)^{2}+2\sum_{m\in\Ck}\sum_{0\le t\le T_{m}-2}\bv^{\top}\ymt\cdot\umt^{\top}\bW^{1/2}\bv,
\]
where $\umt^{\top}\bW^{1/2}\bv\sim\Ncal(\mathbf{0},\bv^{\top}\bW\bv)$.
Lemma \ref{lem:self_norm_scalar} (the scalar version of self-normalized
concentration) tells us that, for any fixed $\lambda>0$, with probability
at least $1-\delta$,
\begin{align*}
\sum_{m\in\Ck}\sum_{0\le t\le T_{m}-2}\bv^{\top}\ymt\cdot\umt^{\top}\bW^{1/2}\bv & \ge-\sqrt{\bv^{\top}\bW\bv}\,\Big(\frac{\lambda}{2}\sum_{m\in\Ck}\sum_{0\le t\le T_{m}-2}(\bv^{\top}\ymt)^{2}+\frac{1}{\lambda}\log\frac{1}{\delta}\Big)\\
 & \ge-\sqrt{\|\bW\|}\,\Big(\frac{\lambda}{2}\sum_{m\in\Ck}\sum_{0\le t\le T_{m}-2}(\bv^{\top}\ymt)^{2}+\frac{1}{\lambda}\log\frac{1}{\delta}\Big).
\end{align*}
Replacing $\delta$ with $\delta/(9/\epsilon)^{d}$ and taking the
union bound, we have with probability at least $1-\delta$, for any
$\bv\in\Ncal_{\epsilon}$,
\begin{align*}
\bv^{\top}\bQ\bv & \ge\sum_{m\in\Ck}\sum_{0\le t\le T_{m}-2}(\bv^{\top}\ymt)^{2}-\sqrt{\|\bW\|}\,\Big(\lambda\sum_{m\in\Ck}\sum_{0\le t\le T_{m}-2}(\bv^{\top}\ymt)^{2}+\frac{2}{\lambda}\big(d\cdot\log\frac{9}{\epsilon}+\log\frac{1}{\delta}\big)\Big)\\
 & =\big(1-\sqrt{\|\bW\|}\lambda\big)\sum_{m\in\Ck}\sum_{0\le t\le T_{m}-2}(\bv^{\top}\ymt)^{2}-\sqrt{\|\bW\|}\,\frac{2}{\lambda}\Big(d\cdot\log\frac{9}{\epsilon}+\log\frac{1}{\delta}\Big).
\end{align*}
With the choice of $\lambda=1/\sqrt{\|\bW\|}$, this implies 
\begin{equation}
\bv^{\top}\bQ\bv\ge-2\|\bW\|(d\cdot\log\frac{9}{\epsilon}+\log\frac{1}{\delta}),\quad\text{for all}\quad\bv\in\Ncal_{\epsilon}.\label{eq:VQV_lower_bound}
\end{equation}

Putting (\ref{eq:Q_upper_bound}) and (\ref{eq:VQV_lower_bound})
back to (\ref{eq:lambmin_Q}), we have with probability at least $1-\delta$,
\begin{align*}
\lambmin(\bQ) & \ge\inf_{\bv\in\Ncal_{\epsilon}}\bv^{\top}\bQ\bv-3\epsilon\|\bQ\|\ge-C_{0}\bigg(\|\bW\|\big(d\cdot\log\frac{9}{\epsilon}+\log\frac{1}{\delta}\big)+\epsilon\,\Gmax\cdot\poly\big(\ka,d,\Tt,\frac{1}{\delta}\big)\bigg)
\end{align*}
for some universal constant $C_{0}>0$.

\paragraph{Putting things together.}

Recall the decomposition $\XTX\succcurlyeq\bP+\bQ$ in (\ref{eq:VT_decomp}).
We have already shown that if $T\gtrsim\kwself^{2}(d+\log\frac{1}{\delta})$,
then with probability at least $1-\delta$,
\begin{align*}
\XTX & \succcurlyeq\bP+\bQ\succcurlyeq\frac{1}{4}T\cdot\bW-C_{0}\bigg(\|\bW\|\big(d\cdot\log\frac{9}{\epsilon}+\log\frac{1}{\delta}\big)+\epsilon\,\Gmax\cdot\poly\big(\ka,d,\Tt,\frac{1}{\delta}\big)\bigg)\Id.
\end{align*}
It is easy to check that, if we further choose
\[
\epsilon\asymp\frac{1}{\poly(\ka,d,\Tt,\frac{1}{\delta},\frac{\Gmax}{\Wmin})},\quad T\gtrsim\kwself^{2}d\cdot\log\Big(\frac{\Gmax}{\Wmin}\frac{\ka d\Tt}{\delta}\Big),
\]
then we have $\XTX\succcurlyeq\frac{1}{5}T\cdot\bW$, which finishes
the proof of (\ref{eq:VT_lower_bound}).

\subsection{\label{subsec:proof_classification}Proof of Theorem~\ref{thm:classification}}

In this proof, we show the correct classification of one short trajectory
$\{\xmt\}_{0\le t\le T_{m}}$ (with true label $k_{m}$) for some
$m\in\Mclassification$; then it suffices to take the union bound
to prove the correct classification of all trajectories in $\Mclassification$.
For notational simplicity, we drop the subscript $m$ and rewrite
$\{\xmt\}$, $T_{m}$, $k_{m}$ as $\{\xt\}$, $T$, $k$, respectively.
The basic idea of this proof is to show that 
\begin{equation}
L(\Alhat,\Wlhat)>L(\Akhat,\Wkhat)\label{eq:L_compare}
\end{equation}
for any incorrect label $\ell\neq k$, where $L$ is the loss function
defined in (\ref{eq:def_loss}) and used by Algorithm~\ref{alg:LS}
for classification. Our proof below is simply a sequence of arguments
that finally transform (\ref{eq:L_compare}) into a sufficient condition
in terms of the coarse model errors $\epsA,\epsW$ and short trajectory
length $T$.

Before we proceed, we record a few basic facts that will be useful
later. First, the assumption $\|\Wkhat-\Wk\|\le\epsW\le0.1\llb$ implies
that $\lambmin(\Wkhat)\ge0.9\lambmin(\Wk)$, and $\Wkhat$ is well
conditioned with $\kappa(\Wkhat)\lesssim\kappa(\Wk)\le\kwself$. Morever,
by Lemma~\ref{lem:bounded_state}, with probability at least $1-\delta$,
we have for all $0\le t\le T$, 
\begin{itemize}
\item In Case 0, $\|\xt\|_{2}\le\Dx\lesssim\sqrt{\Gmax(d+\log\frac{\Tt}{\delta})}$,
provided that $T\gtrsim1$;
\item In Case 1, $\|\xt\|_{2}\le\Dx\lesssim\ka\sqrt{\Gmax(d+\log\frac{\Tt}{\delta})}$,
provided that $T\gtrsim\frac{1}{1-\rho}\log(2\ka)$.
\end{itemize}
Now we are ready to state our proof. Throughout our analyses, we will
make some intermediate claims, whose proofs will be deferred to the
end of this subsection.

\paragraph{Step 1: a sufficient condition for correct classification.}

In the following, we prove that for a fixed $\ell\neq k$, the condition
(\ref{eq:L_compare}) holds with high probability; at the end of the
proof, we simply take the union bound over $\ell\neq k$. Using $\xtp=\Ak\xt+\wt$
where $\wt\sim\Ncal(\boldsymbol{0},\Wk)$, we can rewrite the loss
function $L$ as 
\begin{align*}
L(\bA,\bW) & =T\cdot\log\det(\bW)+\sumtm\wt^{\top}\bW^{-1}\wt\\
 & \quad+\sumtm\xt^{\top}(\Ak-\bA)^{\top}\bW^{-1}(\Ak-\bA)\xt+2\sumtm\wt^{\top}\bW^{-1}(\Ak-\bA)\xt.
\end{align*}
After some basic calculation, (\ref{eq:L_compare}) can be equivalently
written as
\begin{align*}
 & L(\Alhat,\Wlhat)-L(\Akhat,\Wkhat)=({\rm A})+({\rm B})-({\rm C})>0,\\
\text{where} & \quad{\rm (A)}\coloneqq T\cdot\Big(\log\det(\Wlhat)-\log\det(\Wkhat)\Big)+\sum_{t=0}^{T-1}\wt^{\top}\Big((\Wlhat)^{-1}-(\Wkhat)^{-1}\Big)\wt\\
 & \quad{\rm (B)}\coloneqq\Big(\sumtm\xt^{\top}(\Ak-\Alhat)^{\top}(\Wlhat)^{-1}(\Ak-\Alhat)\xt+2\sumtm\wt^{\top}(\Wlhat)^{-1}(\Ak-\Alhat)\xt\Big)\\
 & \quad{\rm (C)}\coloneqq\Big(\sumtm\xt^{\top}(\Ak-\Akhat)^{\top}(\Wkhat)^{-1}(\Ak-\Akhat)\xt+2\sumtm\wt^{\top}(\Wkhat)^{-1}(\Ak-\Akhat)\xt\Big)
\end{align*}

\paragraph{Step 2: a lower bound for ${\rm (A)}+{\rm (B)}-{\rm (C)}$.}

Intuitively, we expect that ${\rm (A)}+{\rm (B)}$ should be large
because the LDS models $(\Ak,\Wk)$ and $(\Al,\Wl)$ are well separated,
while ${\rm (C)}$ should be small if $(\Akhat,\Wkhat)\approx(\Ak,\Wk)$.
More formally, we claim that the following holds for some universal
constants $C_{1},C_{2},C_{3}>0$:
\begin{itemize}
\item With probability at least $1-\delta$, 
\begin{align}
({\rm A}) & \ge T\cdot\bigg[\log\det(\Wlhat)-\log\det(\Wkhat)+\Tr\Big((\Wk)^{1/2}\big((\Wlhat)^{-1}-(\Wkhat)^{-1}\big)(\Wk)^{1/2}\Big)\bigg]\nonumber \\
 & \quad\quad-C_{1}\sqrt{T}\Big\|(\Wk)^{1/2}\big((\Wlhat)^{-1}-(\Wkhat)^{-1}\big)(\Wk)^{1/2}\Big\|_{\Frm}\log\frac{1}{\delta};\label{eq:boundA}
\end{align}
\item With probability at least $1-\delta$,
\begin{equation}
({\rm B})\ge C_{2}\bigg(T\frac{\|\Ak-\Alhat\|_{\Frm}^{2}}{\kwcross}-\kwself\kwcross\log\frac{1}{\delta}\bigg),\label{eq:boundB}
\end{equation}
provided that $T\gtrsim\kwself^{2}\log^{2}(1/\delta)$;
\item With probability at least $1-\delta$,
\begin{equation}
({\rm C})\le C_{3}\bigg(T\frac{\Dx^{2}\|\Ak-\Akhat\|^{2}}{\llb}+\kwself\log\frac{1}{\delta}\bigg),\label{eq:boundC}
\end{equation}
provided that $T\gtrsim1$ (under Case 0) or $T\gtrsim\frac{1}{1-\rho}\log(2\ka)$
(under Case 1).
\end{itemize}
Putting these together, we have with probability at least $1-\delta$,
\begin{align*}
{\rm (A)}+{\rm (B)}-{\rm (C)} & \ge C_{4}\cdot T\cdot\bigg[\log\frac{\det(\Wlhat)}{\det(\Wkhat)}+\Tr\Big(\Wk\big((\Wlhat)^{-1}-(\Wkhat)^{-1}\big)\Big)\\
 & \qquad\qquad+\frac{\|\Ak-\Alhat\|_{\Frm}^{2}}{\kwcross}-\frac{\Dx^{2}\|\Ak-\Akhat\|^{2}}{\llb}\bigg]\\
 & \qquad-C_{5}\bigg[\sqrt{T}\Big\|(\Wk)^{1/2}\big((\Wlhat)^{-1}-(\Wkhat)^{-1}\big)(\Wk)^{1/2}\Big\|_{\Frm}+\kwself\kwcross\bigg]\log\frac{1}{\delta}
\end{align*}
for some universal constants $C_{4},C_{5}>0$, provided that $T\gtrsim\kwself^{2}\log^{2}\frac{1}{\delta}$
(under Case 0), or $T\gtrsim\kwself^{2}\log^{2}\frac{1}{\delta}+\frac{1}{1-\rho}\log(2\ka)$
(under Case 1). Now we have a lower bound of ${\rm (A)}+{\rm (B)}-{\rm (C)}$
as an order-$T$ term minus a low-order term. Therefore, to show ${\rm (A)}+{\rm (B)}-{\rm (C)}>0$,
it suffices to prove that (a) the leading factor of order-$T$ term
is positive and large, and (b) the low-order term is negligible compared
with the order-$T$ term. More specifically, under the assumption
that the coarse models $\{\Ajhat,\Wjhat\}$ satisfy $\|\Ajhat-\Aj\|\le\epsA$,
$\|\Wjhat-\Wj\|\le\epsW\le0.1\llb$ for all $1\le j\le K$, we make
the following claims:
\begin{itemize}
\item (Order-$T$ term is large.) Define the leading factor $\Dklhat$ of
the order-$T$ term and a related parameter $\Dkl$ as follows:
\begin{align}
\Dklhat & \coloneqq\log\frac{\det(\Wlhat)}{\det(\Wkhat)}+\Tr\Big(\Wk\big((\Wlhat)^{-1}-(\Wkhat)^{-1}\big)\Big)+\frac{\|\Ak-\Alhat\|_{\Frm}^{2}}{\kwcross}-\frac{\Dx^{2}\|\Ak-\Akhat\|^{2}}{\llb},\nonumber \\
\Dkl & \coloneqq\frac{1}{\kwcross}\bigg(\frac{\|\Wk-\Wl\|_{\Frm}^{2}}{\lup^{2}}+\|\Ak-\Al\|_{\Frm}^{2}\bigg)\gtrsim\frac{\DAW^{2}}{\kwcross},\label{eq:Dkl_def}
\end{align}
where the last inequality follows from Assumption \ref{assu:models}.
They are related in the sense that
\begin{align}
\Dkl & \lesssim\tilde{D}_{k,\ell}\coloneqq\log\frac{\det(\Wl)}{\det(\Wk)}+\Tr\Big(\Wk\big((\Wl)^{-1}-(\Wk)^{-1}\big)\Big)+\frac{\|\Ak-\Al\|_{\Frm}^{2}}{\kwcross},\label{eq:Dkl_bound}
\end{align}
where $\tilde{D}_{k,\ell}$ is defined in the same way as $\Dklhat$,
except that the coarse models are replaced with the accurate ones;
moreover, we have
\begin{equation}
\Dkl\lesssim\Dklhat,\quad\text{provided that}\quad\epsA\lesssim\sqrt{\frac{\llb\Dkl}{\Dx^{2}}},\quad\epsW\lesssim\Wmin\sqrt{\frac{\Dkl}{d}}.\label{eq:epsW_epsA}
\end{equation}
\item (Low-order term is negligible.) With~(\ref{eq:epsW_epsA}) in place,
we have
\begin{align}
 & {\rm (A)}+{\rm (B)}-{\rm (C)}\nonumber \\
 & \quad\ge C_{6}T\cdot\Dkl-C_{5}\bigg[\sqrt{T}\Big\|(\Wk)^{1/2}\big((\Wlhat)^{-1}-(\Wkhat)^{-1}\big)(\Wk)^{1/2}\Big\|_{\Frm}+\kwself\kwcross\bigg]\log\frac{1}{\delta}.\label{eq:ABC_tmp}
\end{align}
We claim that
\begin{equation}
\text{if}\quad T\gtrsim\bigg(\frac{\kwcross^{5}}{\Dkl}+1\bigg)\log^{2}\frac{1}{\delta},\quad\text{then}\quad{\rm (A)}+{\rm (B)}-{\rm (C)}\ge C_{7}T\cdot\Dkl>0.\label{eq:claim_low_order}
\end{equation}
\end{itemize}

\paragraph{Step 3: putting things together.}

So far, we have proved that for a short trajectory generated by $(\Ak,\Wk)$
and for a fixed $\ell\neq k$, it holds with probability at least
$1-\delta$ that $L(\Alhat,\Wlhat)>L(\Akhat,\Wkhat)$, provided that
\begin{align*}
\epsA\lesssim\sqrt{\frac{\llb\Dkl}{\Dx^{2}}},\quad\epsW\lesssim\llb\cdot\min\Big\{1,\sqrt{\frac{\Dkl}{d}}\Big\}, & \quad T\gtrsim\begin{cases}
\Big(\kwself^{2}+\frac{\kwcross^{5}}{\Dkl}\Big)\log^{2}\frac{1}{\delta} & \text{for Case 0},\\
\Big(\kwself^{2}+\frac{\kwcross^{5}}{\Dkl}\Big)\log^{2}\frac{1}{\delta}+\frac{1}{1-\rho}\log(2\ka) & \text{for Case 1}.
\end{cases}
\end{align*}
Plugging in the relation $\Dkl\gtrsim\DAW^{2}/\kwcross$ and $\Dx\lesssim\sqrt{\Gmax(d+\log\frac{\Tt}{\delta})}$
(for Case 0) or $\Dx\lesssim\ka\sqrt{\Gmax(d+\log\frac{\Tt}{\delta})}$
(for Case 1), the above conditions become
\begin{align*}
\text{For Case 0:}\quad & \epsA\lesssim\sqrt{\frac{\llb\DAW^{2}}{\Gmax\kwcross(d+\log\frac{\Tt}{\delta})}},\quad\epsW\lesssim\llb\cdot\min\Big\{1,\frac{\DAW}{\sqrt{\kwcross d}}\Big\},\\
 & \qquad T\gtrsim\Big(\kwself^{2}+\frac{\kwcross^{6}}{\DAW^{2}}\Big)\log^{2}\frac{1}{\delta};\\
\text{For Case 1:}\quad & \epsA\lesssim\sqrt{\frac{\llb\DAW^{2}}{\Gmax\kwcross\ka^{2}(d+\log\frac{\Tt}{\delta})}},\quad\epsW\lesssim\llb\cdot\min\Big\{1,\frac{\DAW}{\sqrt{\kwcross d}}\Big\},\\
 & \qquad T\gtrsim\Big(\kwself^{2}+\frac{\kwcross^{6}}{\DAW^{2}}\Big)\log^{2}\frac{1}{\delta}+\frac{1}{1-\rho}\log(2\ka).
\end{align*}
Finally, taking the union bound over all $\ell\neq k$ as well as
over all trajectories in $\Mclassification$ finishes the proof of
Theorem~\ref{thm:classification}.

\subsubsection{Proof of (\ref{eq:boundA}).}

Since $\wt\iid\Ncal(\boldsymbol{0},\Wk)$, we have
\[
\sum_{t=0}^{T-1}\wt^{\top}\big((\Wlhat)^{-1}-(\Wkhat)^{-1}\big)\wt=\bz^{\top}\bM\bz,
\]
where $\bz\sim\Ncal(\boldsymbol{0},\bI_{Td})$, and $\bM\in\R^{Td\times Td}$
is a block-diagonal matrix with $\bQ\coloneqq(\Wk)^{1/2}((\Wlhat)^{-1}-(\Wkhat)^{-1})(\Wk)^{1/2}\in\R^{d\times d}$
as its diagonal blocks. Therefore, by the Hanson-Wright inequality
\cite[Theorem 6.2.1]{vershynin2018high}, we have
\[
\Pr\bigg(\Big|\bz^{\top}\bM\bz-\E[\bz^{\top}\bM\bz]\Big|\ge u\bigg)\le2\exp\bigg(-c\min\Big\{\frac{u^{2}}{\|\bM\|_{\Frm}^{2}},\frac{u}{\|\bM\|}\Big\}\bigg),
\]
where $\|\bM\|_{\Frm}^{2}=T\|\bQ\|_{\Frm}^{2},\|\bM\|=\|\bQ\|$, and
$\E[\bz^{\top}\bM\bz]=\Tr(\bM)=T\cdot\Tr(\bQ)$. Choosing $u\gtrsim\sqrt{T}\|\bQ\|_{\Frm}\log\frac{1}{\delta}$,
we have with probability at least $1-\delta$, $|\bz^{\top}\bM\bz-T\cdot\Tr(\bQ)|\le C_{1}\sqrt{T}\|\bQ\|_{\Frm}\log\frac{1}{\delta}$,
which immediately leads to our lower bound~(\ref{eq:boundA}) for
${\rm (A)}$.

\subsubsection{Proof of (\ref{eq:boundB}).}

Denote $\ut=(\Wk)^{-1/2}\wt\sim\Ncal(\boldsymbol{0},\Id)$ and $\yt=(\Wk)^{1/2}(\Wlhat)^{-1}(\Ak-\Alhat)\xt$.
Then we have
\[
{\rm (B)}=\sumtm\xt^{\top}(\Ak-\Alhat)^{\top}(\Wlhat)^{-1}(\Ak-\Alhat)\xt+2\sumtm\ut^{\top}\yt.
\]
By Lemma~\ref{lem:self_norm_scalar}, we have with probability at
least $1-\delta$, $\sumtm\ut^{\top}\yt\ge-(\frac{\lambda}{2}\sumtm\|\yt\|_{2}^{2}+\frac{1}{\lambda}\log\frac{2}{\delta})$
for any fixed $\lambda>0$. This implies that
\begin{align*}
{\rm (B)} & \ge\sumtm\xt^{\top}(\Ak-\Alhat)^{\top}(\Wlhat)^{-1}(\Ak-\Alhat)\xt-\lambda\sumtm\yt^{\top}\yt-\frac{2}{\lambda}\log\frac{2}{\delta}\\
 & =\sumtm\xt^{\top}(\Ak-\Alhat)^{\top}\Big((\Wlhat)^{-1}-\lambda\cdot(\Wlhat)^{-1}\Wk(\Wlhat)^{-1}\Big)(\Ak-\Alhat)\xt-\frac{2}{\lambda}\log\frac{2}{\delta}.
\end{align*}
Choosing $\lambda=0.05/(\kwself\kwcross)$, we have $\lambda\cdot(\Wlhat)^{-1}\Wk(\Wlhat)^{-1}\preccurlyeq0.1(\Wlhat)^{-1}$,
and thus 
\begin{align}
{\rm (B)} & \ge0.9\sumtm\xt^{\top}(\Ak-\Alhat)^{\top}(\Wlhat)^{-1}(\Ak-\Alhat)\xt-40\kwself\kwcross\log\frac{2}{\delta}\nonumber \\
 & \ge0.9\lambmin\big((\Wlhat)^{-1}\big)\sumtm\xt^{\top}(\Ak-\Alhat)^{\top}(\Ak-\Alhat)\xt-40\kwself\kwcross\log\frac{2}{\delta}.\label{eq:B_tmp}
\end{align}

Now it remains to lower bound $\sumtm\xt^{\top}\bDel\xt$, where $\bDel\coloneqq\DelA^{\top}\DelA$
and $\DelA\coloneqq\Ak-\Alhat$. Since $\bDel\succcurlyeq\mathbf{0}$,
we have
\begin{align}
\sumtm\xt^{\top}\bDel\xt & \ge\sum_{t=0}^{T-2}\xtp^{\top}\bDel\xtp=\sum_{t=0}^{T-2}(\Ak\xt+\wt)^{\top}\bDel(\Ak\xt+\wt)\nonumber \\
 & =\underset{{\rm (i)}}{\underbrace{\sum_{t=0}^{T-2}\wt^{\top}\bDel\wt}}+\underset{{\rm (ii)}}{\underbrace{\sum_{t=0}^{T-2}\xt^{\top}\Ak^{\top}\bDel\Ak\xt+2\sum_{t=0}^{T-2}\wt^{\top}\bDel\Ak\xt}}.\label{eq:xDx}
\end{align}
We can lower bound (i) by the Hanson-Wright inequality, similar to
our previous proof of~(\ref{eq:boundA}); the result is that, with
probability at least $1-\delta$, one has 
\[
{\rm (i)}\ge T\cdot\Tr\Big((\Wk)^{1/2}\bDel(\Wk)^{1/2}\Big)-C_{0}\sqrt{T}\big\|(\Wk)^{1/2}\bDel(\Wk)^{1/2}\big\|_{\Frm}\log\frac{1}{\delta}.
\]
To lower bound (ii), we apply Lemma~\ref{lem:self_norm_scalar},
which shows that with probability at least $1-\delta$, 
\[
\bigg|\sum_{t=0}^{T-2}\wt^{\top}\bDel\Ak\xt\bigg|\le\frac{\lambda}{2}\sum_{t=0}^{T-2}\xt^{\top}(\Ak)^{\top}\bDel\Wk\bDel\Ak\xt+\frac{1}{\lambda}\log\frac{2}{\delta},
\]
for any fixed $\lambda>0$, hence 
\[
{\rm (ii)}\ge\sum_{t=0}^{T-2}\xt^{\top}\Ak{}^{\top}(\bDel-\lambda\cdot\bDel\Wk\bDel)\Ak\xt-\frac{2}{\lambda}\log\frac{2}{\delta}.
\]
Recall $\bDel=\DelA^{\top}\DelA$, and thus $\bDel-\lambda\cdot\bDel\Wk\bDel=\DelA^{\top}(\Id-\lambda\cdot\DelA\Wk\DelA^{\top})\DelA\succcurlyeq\boldsymbol{0}$
if we choose $\lambda=1/(\|\Wk\|\|\DelA\|^{2})=1/(\|\Wk\|\|\bDel\|)$;
this implies that
\[
{\rm (ii)}\ge-\frac{2}{\lambda}\log\frac{2}{\delta}=-2\|\Wk\|\|\bDel\|\log\frac{2}{\delta}.
\]
Putting these back to~(\ref{eq:xDx}), we have
\begin{align*}
\sumtm\xt^{\top}\bDel\xt & \ge{\rm (i)}+{\rm (ii)}\\
 & \ge T\cdot\Tr\Big((\Wk)^{1/2}\bDel(\Wk)^{1/2}\Big)-C_{0}\sqrt{T}\big\|(\Wk)^{1/2}\bDel(\Wk)^{1/2}\big\|_{\Frm}\log\frac{1}{\delta}-2\|\Wk\|\|\bDel\|\log\frac{2}{\delta}\\
 & \ge T\cdot\lambmin(\Wk)\Tr(\DelA^{\top}\DelA)-C_{0}\sqrt{T}\|\Wk\|\|\DelA^{\top}\DelA\|_{\Frm}\log\frac{1}{\delta}-2\|\Wk\|\|\DelA^{\top}\DelA\|\log\frac{2}{\delta}\\
 & \ge T\cdot\lambmin(\Wk)\|\Ak-\Alhat\|_{\Frm}^{2}\cdot\bigg(1-\frac{C_{0}}{\sqrt{T}}\cond(\Wk)\log\frac{1}{\delta}-\frac{2}{T}\cond(\Wk)\log\frac{2}{\delta}\bigg)\\
 & \ge0.9T\cdot\lambmin(\Wk)\|\Ak-\Alhat\|_{\Frm}^{2},
\end{align*}
where the last inequality holds if $T\gtrsim\kwself^{2}\log^{2}\frac{1}{\delta}$.

Going back to~(\ref{eq:B_tmp}), we have
\begin{align*}
{\rm (B)} & \ge0.9\lambmin((\Wlhat)^{-1})\sumtm\xt^{\top}\bDel\xt-40\kwself\kwcross\log\frac{2}{\delta}\\
 & \ge0.81T\cdot\lambmin((\Wlhat)^{-1})\lambmin(\Wk)\|\Ak-\Alhat\|_{\Frm}^{2}-40\kwself\kwcross\log\frac{2}{\delta}\\
 & \ge0.7T\frac{\|\Ak-\Alhat\|_{\Frm}^{2}}{\kwcross}-40\kwself\kwcross\log\frac{2}{\delta},
\end{align*}
which finishes the proof of (\ref{eq:boundB}).

\subsubsection{Proof of (\ref{eq:boundC}).}

Denote $\bDel=\Ak-\Akhat$, $\ut=(\Wk)^{-1/2}\wt\sim\Ncal(\boldsymbol{0},\Id)$
and $\yt=(\Wk)^{1/2}(\Wkhat)^{-1}\bDel\xt$. Then one has
\begin{align*}
{\rm (C)} & =\sumtm\xt^{\top}\bDel^{\top}(\Wkhat)^{-1}\bDel\xt+2\sumtm\wt^{\top}(\Wkhat)^{-1}\bDel\xt=\sumtm\xt^{\top}\bDel^{\top}(\Wkhat)^{-1}\bDel\xt+2\sumt\ut^{\top}\yt.
\end{align*}
By Lemma~\ref{lem:self_norm_scalar}, we have with probability at
least $1-\delta$, $\sumtm\ut^{\top}\yt\le\frac{\lambda}{2}\sumtm\|\yt\|_{2}^{2}+\frac{1}{\lambda}\log\frac{2}{\delta}$
for any fixed $\lambda>0$. This implies that
\begin{align*}
{\rm (C)} & \le\sumtm\xt^{\top}\bDel^{\top}(\Wkhat)^{-1}\bDel\xt+\lambda\sumtm\yt^{\top}\yt+\frac{2}{\lambda}\log\frac{2}{\delta}\\
 & =\sumtm\xt^{\top}\bDel^{\top}(\Wkhat)^{-1}\bDel\xt+\lambda\sumtm\xt^{\top}\bDel^{\top}(\Wkhat)^{-1}\Wk(\Wkhat)^{-1}\bDel\xt+\frac{2}{\lambda}\log\frac{2}{\delta}\\
 & =\sumtm\xt^{\top}\bDel^{\top}\Big((\Wkhat)^{-1}+\lambda\cdot(\Wkhat)^{-1}\Wk(\Wkhat)^{-1}\Big)\bDel\xt+\frac{2}{\lambda}\log\frac{2}{\delta}\\
 & \le1.5\bigg(\frac{1}{\lambmin(\Wk)}+\frac{\lambda\cdot\|\Wk\|}{\lambmin(\Wk)^{2}}\bigg)\sumtm\xt^{\top}\bDel^{\top}\bDel\xt+\frac{2}{\lambda}\log\frac{2}{\delta}.
\end{align*}
 Choosing $\lambda\asymp1/\kwself$ and recalling $\|\xt\|_{2}\le\Dx$,
we have ${\rm (C)}\lesssim\frac{1}{\llb}T\cdot\Dx^{2}\|\bDel\|^{2}+\kwself\log\frac{1}{\delta}$,
which finishes the proof of (\ref{eq:boundC}).

\subsubsection{Proof of (\ref{eq:Dkl_bound}).}

Denote $\bDel\coloneqq\Wk-\Wl$. Then the right-hand side of (\ref{eq:Dkl_bound})
becomes
\begin{align*}
\tilde{D}_{k,\ell} & =\log\det\Wl-\log\det\Wk+\Tr\Big(\Wk(\Wl)^{-1}-\Id\Big)+\frac{\|\Ak-\Al\|_{\Frm}^{2}}{\kwcross}\\
 & =\log\det\Wl-\log\det(\Wl+\bDel)+\Tr\Big((\Wl+\bDel)(\Wl)^{-1}-\Id\Big)+\frac{\|\Ak-\Al\|_{\Frm}^{2}}{\kwcross}\\
 & =\log\det\Wl-\log\det\Big((\Wl)^{1/2}\big(\Id+(\Wl)^{-1/2}\bDel(\Wl)^{-1/2}\big)(\Wl)^{1/2}\Big)\\
 & \qquad+\Tr\Big((\Wl)^{-1/2}\bDel(\Wl)^{-1/2}\Big)+\frac{\|\Ak-\Al\|_{\Frm}^{2}}{\kwcross}\\
 & =\Tr(\bX)-\log\det(\Id+\bX)+\frac{\|\Ak-\Al\|_{\Frm}^{2}}{\kwcross},
\end{align*}
where we define $\bX\coloneqq(\Wl)^{-1/2}\bDel(\Wl)^{-1/2}$. Notice
that $\bX$ is symmetric and satisfies
\begin{align*}
\bX+\Id & =(\Wl)^{-1/2}\Wk(\Wl)^{-1/2}\succ\boldsymbol{0},\\
\|\bX\| & \le\big\|(\Wl)^{-1/2}\big\|^{2}\big\|\Wk-\Wl\big\|\le\frac{2\lup}{\llb}=2\kwcross,\\
\|\bX\|_{\Frm}^{2} & =\big\|(\Wl)^{-1/2}\bDel(\Wl)^{-1/2}\big\|_{\Frm}^{2}\ge\frac{\|\bDel\|_{\Frm}^{2}}{\lup^{2}}.
\end{align*}
Therefore, by Lemma~\ref{lem:Tr_logdet}, we have
\[
\Tr(\bX)-\log\det(\Id+\bX)\ge\frac{\|\bX\|_{\Frm}^{2}}{6\kwcross}\ge\frac{\|\Wk-\Wl\|_{\Frm}^{2}}{6\kwcross\lup^{2}},
\]
and thus
\begin{align*}
\tilde{D}_{k,\ell} & =\Tr(\bX)-\log\det(\Id+\bX)+\frac{\|\Ak-\Al\|_{\Frm}^{2}}{\kwcross}\ge\frac{\|\Wk-\Wl\|_{\Frm}^{2}}{6\kwcross\lup^{2}}+\frac{\|\Ak-\Al\|_{\Frm}^{2}}{\kwcross}\asymp\Dkl,
\end{align*}
which finishes the proof of (\ref{eq:Dkl_bound}).

\subsubsection{Proof of (\ref{eq:epsW_epsA}).}

Recall the definition
\[
\Dklhat=\log\frac{\det(\Wlhat)}{\det(\Wkhat)}+\Tr\Big(\Wk\big((\Wlhat)^{-1}-(\Wkhat)^{-1}\big)\Big)+\frac{\|\Ak-\Alhat\|_{\Frm}^{2}}{\kwcross}-\frac{\Dx^{2}\|\Ak-\Akhat\|^{2}}{\llb}.
\]
First, we have
\begin{align*}
 & \log\frac{\det(\Wlhat)}{\det(\Wkhat)}+\Tr\Big(\Wk\big((\Wlhat)^{-1}-(\Wkhat)^{-1}\big)\Big)\\
 & \qquad=\underset{{\rm (i)}}{\underbrace{\bigg[\log\frac{\det(\Wlhat)}{\det(\Wk)}+\Tr\Big(\Wk\big((\Wlhat)^{-1}-(\Wk)^{-1}\big)\Big)\bigg]}}\\
 & \qquad-\underset{{\rm (ii)}}{\underbrace{\bigg[\log\frac{\det(\Wkhat)}{\det(\Wk)}+\Tr\Big(\Wk\big((\Wkhat)^{-1}-(\Wk)^{-1}\big)\Big)\bigg]}}.
\end{align*}
We can lower bound (i) by the same idea of our earlier proof for (\ref{eq:Dkl_bound}),
except that we replace $\Wl$ in that proof with $\Wlhat$; this gives
us
\[
{\rm (i)}\gtrsim\frac{\|\Wk-\Wlhat\|_{\Frm}^{2}}{\kwcross\Wmax^{2}}.
\]
As for (ii), applying Lemma \ref{lem:Tr_logdet} with $\bX=(\Wk)^{-1/2}(\Wkhat-\Wk)(\Wk)^{-1/2}$,
one has
\[
{\rm (ii)}=\Tr(\bX)-\log\det(\bX+\Id)\le\|\bX\|_{\Frm}^{2}\le\frac{\epsW^{2}d}{\Wmin^{2}}.
\]
Putting things together, we have
\begin{align}
\Dklhat & ={\rm (i)}-{\rm (ii)}+\frac{\|\Ak-\Alhat\|_{\Frm}^{2}}{\kwcross}-\frac{\Dx^{2}\|\Ak-\Akhat\|^{2}}{\llb}\nonumber \\
 & \ge\underset{{\rm (iii)}}{\underbrace{C_{1}\frac{1}{\kwcross}\bigg(\frac{\|\Wk-\Wlhat\|_{\Frm}^{2}}{\Wmax^{2}}+\|\Ak-\Alhat\|_{\Frm}^{2}\bigg)}}-\underset{{\rm (iv)}}{\underbrace{C_{2}\bigg(\frac{\epsW^{2}d}{\Wmin^{2}}+\frac{\Dx^{2}\epsA^{2}}{\llb}\bigg)}}\label{eq:Dklhat_tmp}
\end{align}
If $\epsW,\epsA$ satisfy~(\ref{eq:epsW_epsA}), then ${\rm (iv)}\le c_{0}\Dkl$
for some sufficiently small constant $c_{0}>0$. As for ${\rm (iii)}$,
according to the definition of $\Dkl$ in (\ref{eq:Dkl_def}), there
are two possible cases:
\begin{itemize}
\item If $\frac{1}{\kwcross}\frac{\|\Wk-\Wl\|_{\Frm}^{2}}{\Wmax^{2}}\ge\frac{\Dkl}{2}$,
then it is easy to check that (\ref{eq:epsW_epsA}) implies that
\[
\frac{\epsW\sqrt{d}}{\Wmax}\le\frac{1}{4}\frac{\|\Wk-\Wl\|_{\Frm}}{\Wmax},
\]
and hence
\begin{align*}
{\rm (iii)} & \gtrsim\frac{1}{\kwcross}\frac{\|\Wk-\Wlhat\|_{\Frm}^{2}}{\Wmax^{2}}\ge\frac{1}{\kwcross}\bigg(\frac{\|\Wk-\Wl\|_{\Frm}}{\Wmax}-\frac{\epsW\sqrt{d}}{\Wmax}\bigg)^{2}\\
 & \gtrsim\frac{1}{\kwcross}\frac{\|\Wk-\Wl\|_{\Frm}^{2}}{\Wmax^{2}}\gtrsim\Dkl.
\end{align*}
\item On the other hand, if $\frac{1}{\kwcross}\|\Ak-\Al\|_{\Frm}^{2}\ge\frac{\Dkl}{2}$,
then one can check that (\ref{eq:epsW_epsA}) implies that
\[
\epsA\sqrt{d}\le\frac{1}{4}\|\Ak-\Al\|_{\Frm},
\]
and hence
\[
{\rm (iii)}\gtrsim\frac{\|\Ak-\Alhat\|_{\Frm}^{2}}{\kwcross}\ge\frac{(\|\Ak-\Al\|_{\Frm}-\epsA\sqrt{d})^{2}}{\kwcross}\gtrsim\frac{\|\Ak-\Al\|_{\Frm}^{2}}{\kwcross}\gtrsim\Dkl.
\]
\end{itemize}
In sum, it is always guaranteed that ${\rm (iii)}\gtrsim\Dkl$. Going
back to~(\ref{eq:Dklhat_tmp}), we claim that $\Dklhat\ge{\rm (iii)}-{\rm (iv)}\gtrsim\Dkl$
as long as $\epsW$ and $\epsA$ satisfy~(\ref{eq:epsW_epsA}), which
finishes the proof of~(\ref{eq:epsW_epsA}).

\subsubsection{Proof of (\ref{eq:claim_low_order}).}

Recall the lower bound~(\ref{eq:ABC_tmp}) for ${\rm (A)+(B)-(C)}$.
We want to show that the low-order term is dominated by the order-$T$
term, namely $T\cdot\Dkl$. First, if $T\gtrsim\frac{\kwself\kwcross}{\Dkl}\log\frac{1}{\delta}$,
then $\kwself\kwcross\log\frac{1}{\delta}\lesssim T\cdot\Dkl$. Next,
we have

\begin{align*}
 & \Big\|(\Wk)^{1/2}\big((\Wlhat)^{-1}-(\Wkhat)^{-1}\big)(\Wk)^{1/2}\Big\|_{\Frm}\\
 & \qquad\le\big\|\Wk\big\|\cdot\big\|(\Wlhat)^{-1}-(\Wkhat)^{-1}\big\|_{\Frm}=\big\|\Wk\big\|\cdot\big\|(\Wlhat)^{-1}(\Wkhat-\Wlhat)(\Wkhat)^{-1}\big\|_{\Frm}\\
 & \qquad\le\big\|\Wk\big\|\cdot\big\|(\Wlhat)^{-1}\big\|\cdot\big\|(\Wkhat)^{-1}\big\|\cdot\Big(\|\Wk-\Wl\|_{\Frm}+2\sqrt{d}\epsW\Big)\\
 & \qquad\le\frac{2\lup}{\llb^{2}}\left(\|\Wk-\Wl\|_{\Frm}+2\sqrt{d}\epsW\right)=\frac{2\kwcross}{\llb}\left(\|\Wk-\Wl\|_{\Frm}+2\sqrt{d}\epsW\right).
\end{align*}
Notice that the definition (\ref{eq:Dkl_def}) of $\Dkl$ implies
that $\|\Wk-\Wl\|_{\Frm}\le\sqrt{\kwcross\lup^{2}\Dkl}$, and thus
\[
\Big\|(\Wk)^{1/2}\big((\Wlhat)^{-1}-(\Wkhat)^{-1}\big)(\Wk)^{1/2}\Big\|_{\Frm}\lesssim\frac{\kwcross}{\llb}\Big(\sqrt{\kwcross\lup^{2}\Dkl}+\sqrt{d}\epsW\Big).
\]
Now it is easy to checked that
\begin{align*}
\text{if}\quad & T\gtrsim\bigg(\frac{\kwcross^{5}}{\Dkl}+\Big(\frac{\kwcross\sqrt{d}\epsW}{\llb\Dkl}\Big)^{2}\bigg)\log^{2}\frac{1}{\delta},\\
\text{then}\quad & \sqrt{T}\Big\|(\Wk)^{1/2}\big((\Wlhat)^{-1}-(\Wkhat)^{-1}\big)(\Wk)^{1/2}\Big\|_{\Frm}\log\frac{1}{\delta}\lesssim T\cdot\Dkl.
\end{align*}
Due to our assumption (\ref{eq:epsW_epsA}) on $\epsW$, we have $(\frac{\kwcross\sqrt{d}\epsW}{\llb\Dkl})^{2}\lesssim\frac{\kwcross^{2}}{\Dkl}\le\frac{\kwcross^{5}}{\Dkl}$,
and thus it suffices to have $T\gtrsim(\frac{\kwcross^{5}}{\Dkl}+1)\log^{2}\frac{1}{\delta}$,
which finishes the proof of (\ref{eq:claim_low_order}).

\section{Miscellaneous results\label{sec:misc_lemmas}}
\begin{prop}
\label{prop:covariance_concentration} Consider $\ba_{t},\bb_{t}\iid\Ncal(\boldsymbol{0},\Id),1\le t\le N$,
where $N\gtrsim d+\log(1/\delta)$. Then with probability at least
$1-\delta$, 
\[
\bigg\|\frac{1}{N}\sum_{t=1}^{N}\at\at^{\top}-\Id\bigg\|\lesssim\sqrt{\frac{d+\log\frac{1}{\delta}}{N}},\quad\bigg\|\frac{1}{N}\sum_{t=1}^{N}\at\bb_{t}^{\top}\bigg\|\lesssim\sqrt{\frac{d+\log\frac{1}{\delta}}{N}}.
\]
\end{prop}
The proof follows from a standard covering argument (cf.~\cite{vershynin2018high}),
which we skip for brevity.
\begin{lem}
[Self-normalized concentration, scalar version] \label{lem:self_norm_scalar}
Suppose that random vectors $\{\ut,\yt\}_{1\le t\le T}$ and filtrations
$\{\Ft\}_{0\le t\le T-1}$ satisfy $\Ft=\sigma(\bu_{i},1\le i\le t)$,
$\yt\in\Ftm$, and $\ut|\Ftm\sim\Ncal(\boldsymbol{0},\Id)$. Then
for any fixed $\lambda>0$, with probability at least $1-\delta$,
we have
\[
\bigg|\sumt\ut^{\top}\yt\bigg|<\frac{\lambda}{2}\sumt\big\|\yt\big\|_{2}^{2}+\frac{1}{\lambda}\log\frac{2}{\delta}.
\]
\end{lem}
\begin{proof}
Using basic properties of the Gaussian distribution, we have for all
fixed $\lambda\in\R$,
\begin{align*}
 & \E\bigg[\exp\Big(\sumt\big(\lambda\cdot\ut^{\top}\yt-\frac{1}{2}\lambda^{2}\|\yt\|_{2}^{2}\big)\Big)\bigg]\\
 & \qquad=\E\bigg[\exp\Big(\sum_{t=1}^{T-1}\big(\lambda\cdot\ut^{\top}\yt-\frac{1}{2}\lambda^{2}\|\yt\|_{2}^{2}\big)\Big)\cdot\underset{\le1}{\underbrace{\E\Big[\exp\big(\lambda\cdot\bu_{T}^{\top}\by_{T}-\frac{1}{2}\lambda^{2}\|\by_{T}\|_{2}^{2}\big)|\Fcal_{T-1}\Big]}}\bigg]\\
 & \qquad\le\E\bigg[\exp\Big(\sum_{t=1}^{T-1}\big(\lambda\cdot\ut^{\top}\yt-\frac{1}{2}\lambda^{2}\|\yt\|_{2}^{2}\big)\Big)\bigg].
\end{align*}
Continuing this expansion leads to the result $\E[\exp(\sumt(\lambda\cdot\ut^{\top}\yt-\frac{1}{2}\lambda^{2}\|\yt\|_{2}^{2}))]\le1$.
Now, letting $z=\log(2/\delta)$ and using Markov's inequality, we
have
\begin{align*}
\Pr\Big(\sumt\big(\lambda\cdot\ut^{\top}\yt-\frac{1}{2}\lambda^{2}\|\yt\|_{2}^{2}\big)\ge z\Big) & =\Pr\bigg(\exp\Big(\sumt\big(\lambda\cdot\ut^{\top}\yt-\frac{1}{2}\lambda^{2}\|\yt\|_{2}^{2}\big)\Big)\ge\exp(z)\bigg)\\
 & \le\exp(-z)\cdot\E\bigg[\exp\Big(\sumt\big(\lambda\cdot\ut^{\top}\yt-\frac{1}{2}\lambda^{2}\|\yt\|_{2}^{2}\big)\Big)\bigg]\le\exp(-z)=\frac{\delta}{2}.
\end{align*}
In other words, with probability at least $1-\delta/2$, we have $\sumt(\lambda\cdot\ut^{\top}\yt-\frac{1}{2}\lambda^{2}\|\yt\|_{2}^{2})<z=\log(2/\delta)$,
which implies $\sumt\ut^{\top}\yt<\frac{\lambda}{2}\sumt\|\yt\|_{2}^{2}+\frac{1}{\lambda}\log\frac{2}{\delta}$
if $\lambda>0$. By a similar argument (but with $\lambda$ replaced
by $-\lambda$), we have with probability at least $1-\delta/2$,
$\sumt(-\lambda\cdot\ut^{\top}\yt-\frac{1}{2}\lambda^{2}\|\yt\|_{2}^{2})<\log(2/\delta)$,
which implies $\sumt\ut^{\top}\yt>-(\frac{\lambda}{2}\sumt\|\yt\|_{2}^{2}+\frac{1}{\lambda}\log\frac{2}{\delta})$
if $\lambda>0$. Finally, taking the union bound finishes the proof
of the lemma.
\end{proof}
\begin{lem}
\label{lem:Tr_logdet}Consider a symmetric matrix $\bX\in\R^{d\times d}$
that satisfies $\Id+\bX\succ\boldsymbol{0}$. If $\|\bX\|\le B$ for
some $B\ge2$, then we have
\[
\Tr(\bX)-\log\det(\Id+\bX)\ge\frac{\|\bX\|_{\Frm}^{2}}{3B}.
\]
On the other hand, if $\bX\succcurlyeq-\frac{1}{2}\Id$, then we have
\[
\Tr(\bX)-\log\det(\Id+\bX)\le\|\bX\|_{\Frm}^{2}.
\]
\end{lem}
\begin{proof}
Denote $\{\lambda_{i}\}_{1\le i\le d}$ as the eigenvalues of $\bX$,
which satisfies $\lambda_{i}>-1$ for all $1\le i\le d$. Then 
\[
\Tr(\bX)-\log\det(\Id+\bX)=\sumi\lambda_{i}-\log\prod_{i=1}^{d}(1+\lambda_{i})=\sumi\Big(\lambda_{i}-\log(1+\lambda_{i})\Big).
\]
It can be checked (via elementary calculus) that, for all $-1<\lambda\le B$
where $B\ge2$, it holds that $\lambda-\log(1+\lambda)\ge\lambda^{2}/(3B)$.
Therefore,
\[
\Tr(\bX)-\log\det(\Id+\bX)\ge\sumi\frac{\lambda_{i}^{2}}{3B}=\frac{\|\bX\|_{\Frm}^{2}}{3B},
\]
which completes the proof of our first claim. Similarly, it can be
checked that, for all $\lambda\ge-1/2$, one has $\lambda-\log(1+\lambda)\le\lambda^{2}$,
which implies that
\[
\Tr(\bX)-\log\det(\Id+\bX)\le\sumi\lambda_{i}^{2}=\|\bX\|_{\Frm}^{2};
\]
this completes the proof of our second claim.
\end{proof}
\begin{fact}
\label{fact:parameters}In the setting of Section \ref{subsec:models_assumptions},
it holds that $\Gmax\le\lup\ka^{2}/(1-\rho)$ and $\DGY\ge\DAW\cdot\Wmax\Wmin/(4\ka^{2}\Gmax)$.
\end{fact}
\begin{proof}
First, consider $\bGamma=\sum_{i=0}^{\infty}\bA^{i}\bW(\bA^{i})^{\top}$,
where $\|\bW\|\le\Wmax$ and $\|\bA^{i}\|\le\ka\rho^{i}$. Then we
have $\|\bGamma\|\le\sum_{i=0}^{\infty}\|\bA^{i}\|^{2}\|\bW\|\le\Wmax\ka^{2}\sum_{i=0}^{\infty}\rho^{2i}\le\Wmax\ka^{2}/(1-\rho)$,
which proves our upper bound for $\Gmax$.

Next, let us turn to $\DGY$. Our proof below can be viewed as a quantitative
version of the earlier proof for Fact~\ref{fact:separation}. We
will show that, if the autocovariance matrices between the $k$-th
and the $\ell$-th models are close (in Frobenius norm), then the
models themselves should also be close; our lower bound for $\DGY$
in terms of $\DAW$ then follows from contraposition. 

Consider two LDS models $(\Ak,\Wk)\neq(\Al,\Wl)$ and their autocovariance
matrics $(\Gammak,\Yk),(\Gammal,\Yl)$. Recall from (\ref{eq:GY_to_AW})
that $\Ak=\Yk\Gammak^{-1}$ and $\Wk=\Gammak-\Ak\Gammak\Ak^{\top}=\Gammak-\Ak\Yk^{\top}$;
$\Al$ and $\Wl$ can be expressed similarly. 
\begin{itemize}
\item First, regarding $\Ak-\Al$, one has
\[
\Ak-\Al=\Yk\Gammak^{-1}-\Yl\Gammal^{-1}=(\Yk-\Yl)\Gammak^{-1}+\Yl(\Gammak^{-1}-\Gammal^{-1}),
\]
where 
\[
\Yl(\Gammak^{-1}-\Gammal^{-1})=\Yl\Gammal^{-1}(\Gammal-\Gammak)\Gammak^{-1}=\Al(\Gammal-\Gammak)\Gammak^{-1}.
\]
Therefore, we have
\begin{align}
\|\Ak-\Al\|_{\Frm} & \le\|\Yk-\Yl\|_{\Frm}\|\Gammak^{-1}\|+\|\Al\|\|\Gammal-\Gammak\|_{\Frm}\|\Gammak^{-1}\|\nonumber \\
 & \le\frac{1}{\Wmin}\|\Yk-\Yl\|_{\Frm}+\frac{\ka}{\Wmin}\|\Gammal-\Gammak\|_{\Frm},\label{eq:DelA_Fro}
\end{align}
where the last line follows from $\Gammak\succcurlyeq\Wk\succcurlyeq\Wmin\Id$
and $\|\Al\|\le\ka\rho\le\ka$.
\item Next, we turn to the analysis of $\Wk-\Wl$, which satisfies
\begin{align*}
\Wk-\Wl & =(\Gammak-\Gammal)-(\Ak\Yk^{\top}-\Al\Yl^{\top}),\\
\|\Wk-\Wl\|_{\Frm} & \le\|\Gammak-\Gammal\|_{\Frm}+\|\Ak\Yk^{\top}-\Al\Yl^{\top}\|_{\Frm}.
\end{align*}
Notice that
\begin{align*}
\|\Ak\Yk^{\top}-\Al\Yl^{\top}\|_{\Frm} & =\|\Ak(\Yk-\Yl)^{\top}+(\Ak-\Al)\Yl^{\top}\|_{\Frm}\\
 & \le\|\Ak(\Yk-\Yl)^{\top}\|_{\Frm}+\|(\Ak-\Al)\Yl^{\top}\|_{\Frm}\\
 & \le\ka\|\Yk-\Yl\|_{\Frm}+\ka\Gmax\|\Ak-\Al\|_{\Frm},
\end{align*}
where the last line is due to $\|\Yl\|=\|\Al\Gammal\|\le\|\Al\|\cdot\|\Gammal\|\le\ka\Gmax$.
Therefore, we have
\begin{equation}
\|\Wk-\Wl\|_{\Frm}\le\|\Gammak-\Gammal\|_{\Frm}+\ka\|\Yk-\Yl\|_{\Frm}+\ka\Gmax\|\Ak-\Al\|_{\Frm}.\label{eq:DelW_Fro}
\end{equation}
\end{itemize}
For notational simplify, denote $\DelA\coloneqq\Ak-\Al$, $\DelW\coloneqq\Wk-\Wl$,
$\DelG\coloneqq\Gammak-\Gammal$, $\DelY\coloneqq\Yk-\Yl$. From (\ref{eq:DelW_Fro}),
one has
\begin{align*}
\|\DelA\|_{\Frm}^{2}+\frac{\|\DelW\|_{\Frm}^{2}}{\Wmax^{2}} & \le\|\DelA\|_{\Frm}^{2}+\frac{3}{\Wmax^{2}}\Big(\|\DelG\|_{\Frm}^{2}+\ka^{2}\|\DelY\|_{\Frm}^{2}+\ka^{2}\Gmax^{2}\|\DelA\|_{\Frm}^{2}\Big)\\
 & \le\frac{3}{\Wmax^{2}}\|\DelG\|_{\Frm}^{2}+\frac{3\ka^{2}}{\Wmax^{2}}\|\DelY\|_{\Frm}^{2}+\Big(1+\frac{3\ka^{2}\Gmax^{2}}{\Wmax^{2}}\Big)\|\DelA\|_{\Frm}^{2}\\
 & \le\frac{3}{\Wmax^{2}}\|\DelG\|_{\Frm}^{2}+\frac{3\ka^{2}}{\Wmax^{2}}\|\DelY\|_{\Frm}^{2}+\frac{4\ka^{2}\Gmax^{2}}{\Wmax^{2}}\|\DelA\|_{\Frm}^{2},
\end{align*}
where the last line follows from $\ka^{2}\Gmax^{2}/\Wmax^{2}\ge1$.
Morever, (\ref{eq:DelA_Fro}) tells us that $\|\DelA\|_{\Frm}^{2}\le\frac{2}{\Wmin^{2}}\|\DelY\|_{\Frm}^{2}+\frac{2\ka^{2}}{\Wmin^{2}}\|\DelG\|_{\Frm}^{2}$.
Putting together, we have
\begin{align*}
\|\DelA\|_{\Frm}^{2}+\frac{\|\DelW\|_{\Frm}^{2}}{\Wmax^{2}} & \le\frac{3}{\Wmax^{2}}\|\DelG\|_{\Frm}^{2}+\frac{3\ka^{2}}{\Wmax^{2}}\|\DelY\|_{\Frm}^{2}+\frac{4\ka^{2}\Gmax^{2}}{\Wmax^{2}}\|\DelA\|_{\Frm}^{2}\\
 & \le\frac{3}{\Wmax^{2}}\|\DelG\|_{\Frm}^{2}+\frac{3\ka^{2}}{\Wmax^{2}}\|\DelY\|_{\Frm}^{2}+\frac{4\ka^{2}\Gmax^{2}}{\Wmax^{2}}\Big(\frac{2}{\Wmin^{2}}\|\DelY\|_{\Frm}^{2}+\frac{2\ka^{2}}{\Wmin^{2}}\|\DelG\|_{\Frm}^{2}\Big)\\
 & =\Big(\frac{3}{\Wmax^{2}}+\frac{4\ka^{2}\Gmax^{2}}{\Wmax^{2}}\frac{2\ka^{2}}{\Wmin^{2}}\Big)\|\DelG\|_{\Frm}^{2}+\Big(\frac{3\ka^{2}}{\Wmax^{2}}+\frac{4\ka^{2}\Gmax^{2}}{\Wmax^{2}}\frac{2}{\Wmin^{2}}\Big)\|\DelY\|_{\Frm}^{2}\\
 & \le\frac{11\ka^{4}\Gmax^{2}}{\Wmax^{2}\Wmin^{2}}\Big(\|\DelG\|_{\Frm}^{2}+\|\DelY\|_{\Frm}^{2}\Big).
\end{align*}
In sum, we have just shown that, if $\|\DelG\|_{\Frm}^{2}+\|\DelY\|_{\Frm}^{2}<\DGY^{2}$,
then $\|\DelA\|_{\Frm}^{2}+\frac{\|\DelW\|_{\Frm}^{2}}{\Wmax^{2}}<\frac{11\ka^{4}\Gmax^{2}}{\Wmax^{2}\Wmin^{2}}\DGY^{2}$.
Equivalently (by contraposition), if $\|\DelA\|_{\Frm}^{2}+\frac{\|\DelW\|_{\Frm}^{2}}{\Wmax^{2}}\ge\DAW^{2}$,
then $\|\DelG\|_{\Frm}^{2}+\|\DelY\|_{\Frm}^{2}\ge\frac{\Wmax^{2}\Wmin^{2}}{11\ka^{4}\Gmax^{2}}\DAW^{2}$.
This proves our lower bound for $\DGY$ in terms of $\DAW$. 
\end{proof}
\begin{example}
\label{exa:switching}It has been known that in our Case 1 (i.e.~a
single continuous trajectory), the quick switching of multiple LDS
models may lead to exponentially large states, even if each individual
model is stable \cite{liberzon2003switching}. We give a quick example
for completeness. Consider 
\[
\Ak=0.99\begin{bmatrix}0 & 2\\
\frac{1}{2} & 0
\end{bmatrix},\quad\Al=0.99\begin{bmatrix}0 & 3\\
\frac{1}{3} & 0
\end{bmatrix},
\]
both satisfying the stability condition (\ref{eq:stability}) with
$\rho=0.99<1$ and hence $\tmix\asymp1/(1-\rho)=100$. Suppose that
each short trajectory has only a length of 2, and the $m$-th (resp.~$(m+1)$-th)
trajectory has label $k_{m}=\ell$ (resp.~$k_{m+1}=k$). Then $\bx_{m+2,0}$
is equal to $\Ak\Al\bx_{m,0}$ plus a mean-zero noise term, where
\[
\Ak\Al=0.99^{2}\begin{bmatrix}0 & 2\\
\frac{1}{2} & 0
\end{bmatrix}\begin{bmatrix}0 & 3\\
\frac{1}{3} & 0
\end{bmatrix}=0.99^{2}\begin{bmatrix}\frac{2}{3} & 0\\
0 & \frac{3}{2}
\end{bmatrix}
\]
has spectral radius $0.99^{2}\cdot3/2>1$; this will cause the exponential
explosion of the states.
\end{example}

\section{Extensions of Algorithm~\ref{alg:overall}\label{sec:extensions}}

\paragraph{Different trajectory lengths.}

Recall that in Section \ref{sec:algorithms}, we assume that all short
trajectories within each subset of data $\Mo$ have the same length
$T_{m}=\Tso$. If this is not the case, we can easily modify our algorithms
in the following ways:
\begin{itemize}
\item For subspace estimation, the easiest way to handle different $T_{m}$'s
is to simply \emph{truncate} the trajectories in $\Msubspace$ so
that they have the same length $\Tssubspace=\min_{m\in\Msubspace}T_{m}$,
and then apply Algorithm~\ref{alg:subspace} without modification.
However, this might waste many samples when some trajectories of $\Msubspace$
are much longer than others; one way to resolve this is to \emph{manually
divide} the longer trajectories into shorter segments of comparable
lengths, before doing truncation. A more refined method is to modify
Algorithm~\ref{alg:subspace} itself, by \emph{re-defining the index
sets} $\Omegaone,\Omegatwo$ separately for each trajectory; moreover,
in the definition of $\Hihat$ and $\Gihat$, one might consider assigning
larger weights to longer trajectories, instead of using the uniform
weight $1/|\Msubspace|$.
\item For clustering (or pairwise testing) of $\Mclustering$, we can handle
various $T_{m}$'s similarly, by either truncating each pair of trajectories
to the same length, or modifying Algorithm~\ref{alg:clustering} itself
(via re-defining the index sets $\{\Omega_{g,1},\Omega_{g,2}\}_{1\le g\le G}$
separately for each trajectory).
\item Our methods for model estimation and classification, namely Algorithms~\ref{alg:LS} and~\ref{alg:classification},
are already adaptive to different $T_{m}$'s in $\Mclustering$ and
$\Mclassification$, and hence need no modification.
\end{itemize}

\paragraph{Unknown parameters.}

Next, we show how to handle the case when certain parameters are unknown
to the algorithms: 
\begin{itemize}
\item In Algorithm~\ref{alg:subspace}, we set the dimension of the output
subspaces $\{\Vi,\Ui\}$ to be $K$ (the number of models). If $K$
is unknown, we might instead examine the eigenvalues of $\Hihat+\Hihat^{\top}$
and $\Gihat+\Gihat^{\top}$, and pick the subspace dimension that
covers most of the energy in the eigenvalues.
\item In Algorithm~\ref{alg:clustering}, we need to know the separation
parameter $\DGY$ (in order to choose the testing threshold $\tau$
appropriately) and the number of models $K$ (for clustering). If
either $\DGY$ or $K$ is unknown, we might instead try different
values of threshold $\tau$, and pick the one that (after permutation)
makes the similarity matrix $\bS$ block-diagonal with as few blocks
as possible.
\end{itemize}

\section{Additional experiments} 
\label{sec:additional_exp}

\subsection{Synthetic experiments: clustering and classification}

First, we take a closer look at the performance of our clustering method (Algorithm~\ref{alg:clustering}) through synthetic experiments.
We set the parameters $d=40, K=2, \rho=0.5, \delta=0.12$. 
The LDS models are generated by $\Ak = (\rho \pm \delta) \bR$ and $\Wk = \Id$, where $\bR$ is a random orthogonal matrix. 
We let $|\Mclustering| = 5 d$, and vary $\Tsclustering \in [10, 60]$.
We run our clustering method on the dataset $\Mclustering$, either with or without the assistance of subspace estimation (Algorithm~\ref{alg:subspace}) and dimensionality reduction.
For the former case, we use the same dataset $\Mclustering$ for subspace estimation, without sample splitting, which is closer to practice;
for the latter, we simply replace the subspaces $\{\Vi,\Ui\}$ with $\Id$.
The numerical results are illustrated in Figure~\ref{fig:synthetic} (left), confirming that (1) in both cases, the clustering error decreases as $\Tsclustering$ increases, and (2) subspace estimation and dimensionality reduction significantly improve the clustering accuracy.

Next, we examine the performance of our classification method (Algorithm~\ref{alg:classification}) in the same setting as above. 
We first obtain a coarse model estimation by running Stage~1 of Algorithm~\ref{alg:overall} on the dataset $\Mclustering$, with $|\Mclustering|=10 d$ and $\Tsclustering=30$. 
Then, we run classification on the dataset $\Mclassification$, with varying $\Tsclassification \in [4, 50]$.
The numerical results are included in Figure~\ref{fig:synthetic} (right), showing that the classification error rapidly decreases to zero as $\Tsclassification$ grows.

\begin{figure}[ht]
\vskip 0.2in
\begin{center}
\includegraphics[width=0.4\columnwidth]{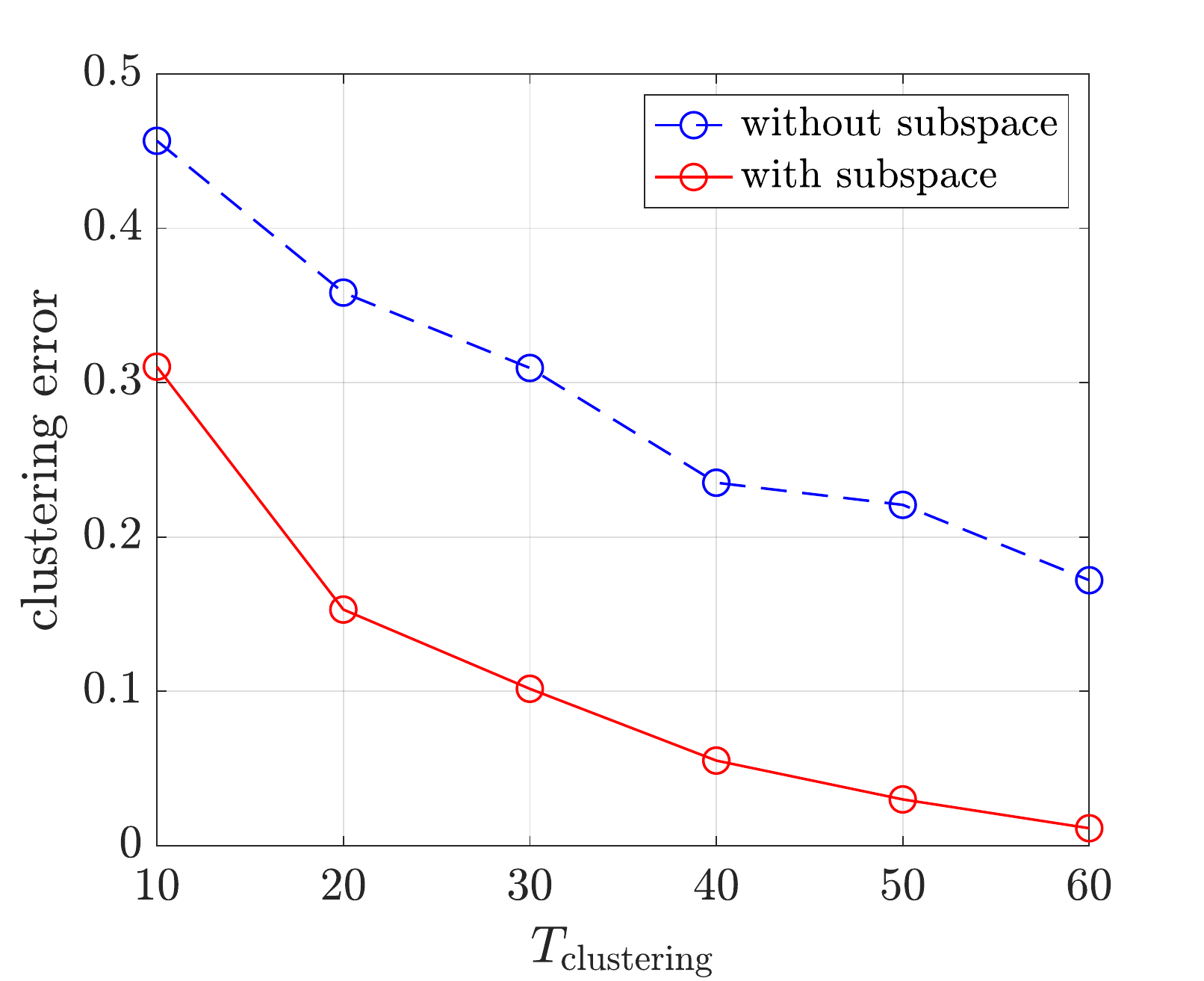}%
\includegraphics[width=0.4\columnwidth]{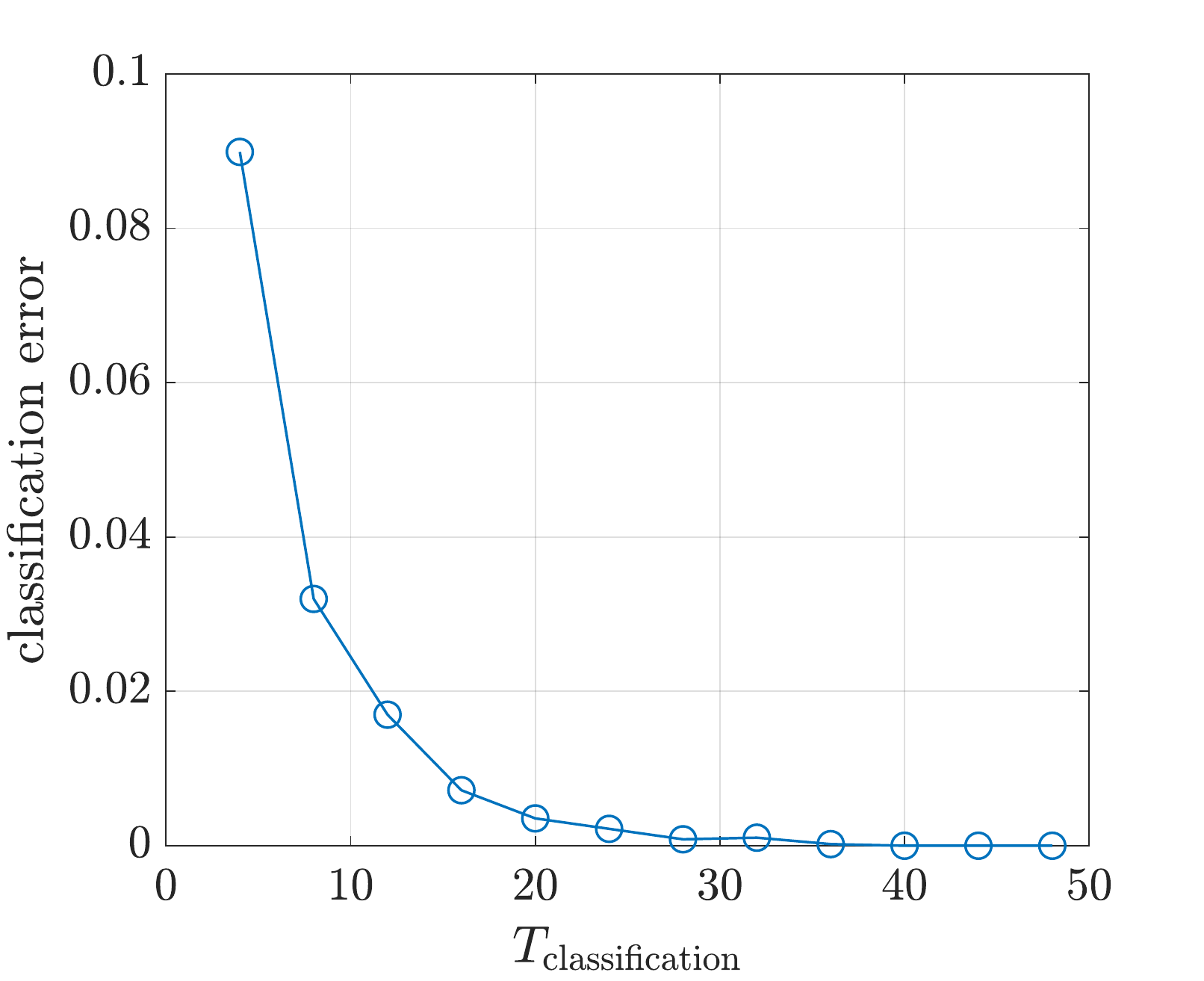}
\caption{\textbf{Left:} mis-clustering rate versus $\Tsclustering$. \textbf{Right:} mis-classification rate versus $\Tsclassification$.}
\label{fig:synthetic}
\end{center}
\vskip -0.2in
\end{figure}

\subsection{Real-data experiments: MotionSense}

To show the practical relevance of the proposed algorithms, we work with the MotionSense dataset~\cite{Malekzadeh:2019:MSD:3302505.3310068}.
This datasets consists of multivariate time series of dimension $d=12$, collected (at a rate of 50Hz) by accelerometer and gyroscope sensors on a mobile phone while a person performs various activities, such as ``jogging'', ``walking'', ``sitting'', and so on.
In our experiments, we break the data into 8-second short trajectories, and treat the human activities as latent variables.
Figure~\ref{fig:motionsense} (left) illustrates what the data looks like. 
Notice that the time series do not satisfy the mixing property assumed in our theory, but are rather periodic instead.

As a preliminary attempt to apply our algorithms in the real world, we show that the proposed clustering method (which is one of the most crucial step in our overall approach), without any modification, works reasonably well even for this dataset.
To be concrete, we apply Algorithm~\ref{alg:clustering} (without dimensionality reduction, i.e.~$\{\Vi,\Ui\}$ are set to $\Id$) to a mixture of 12 ``jogging'' and 12 ``walking'' trajectories.
Figure~\ref{fig:motionsense} (right) shows the resulted \emph{distance matrix}, which is defined in the same way as Line 11 of Algorithm~\ref{alg:clustering}, but without thresholding. 
Its clear block structure confirms that, with an appropriate choice of threshold $\tau$,  Algorithm~\ref{alg:clustering} will return an accurate/exact clustering of the mixed trajectories. 
These results are strong indication that the proposed algorithms in this work might generalize to much broader settings than what our current theory suggests, and we hope that this will inspire further extensions and applications of the proposed methods.

\begin{figure}[ht]
\vskip 0.2in
\begin{center}
\includegraphics[width=0.6\columnwidth]{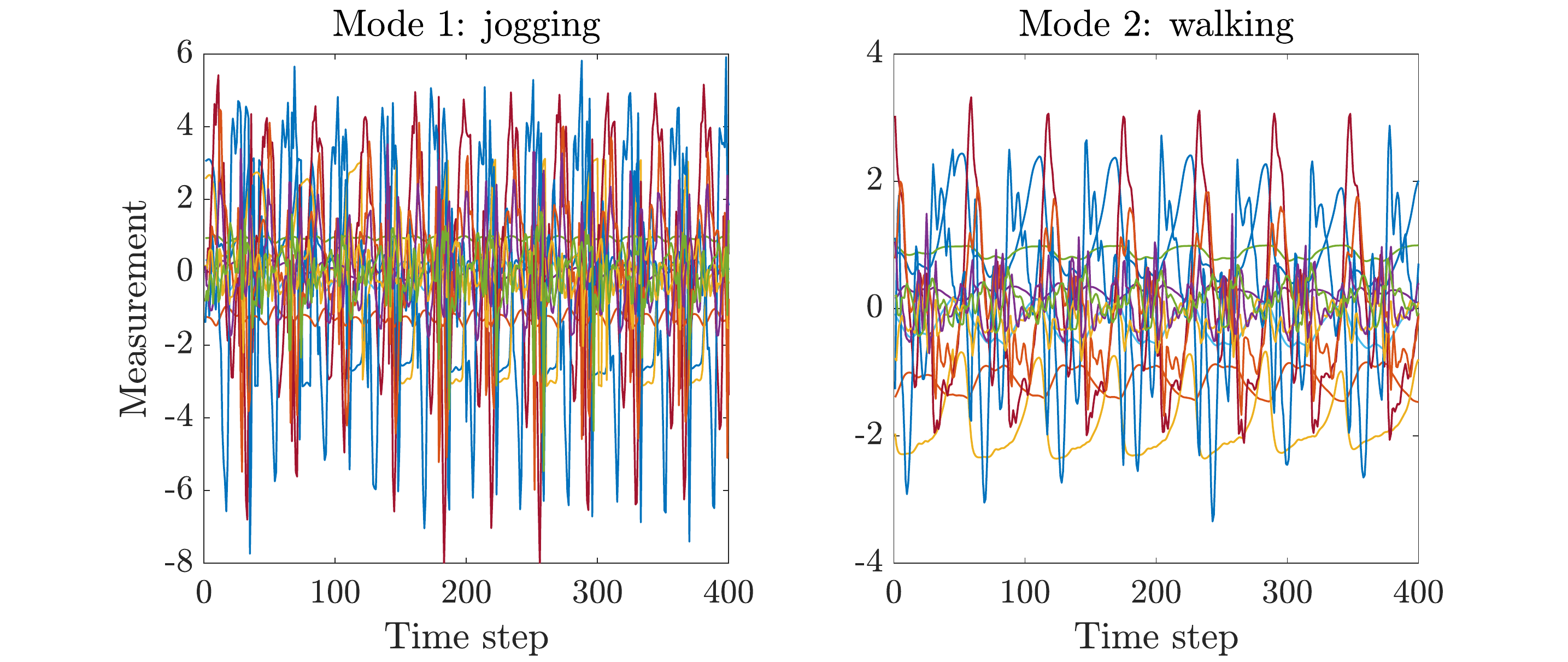}%
\includegraphics[width=0.35\columnwidth]{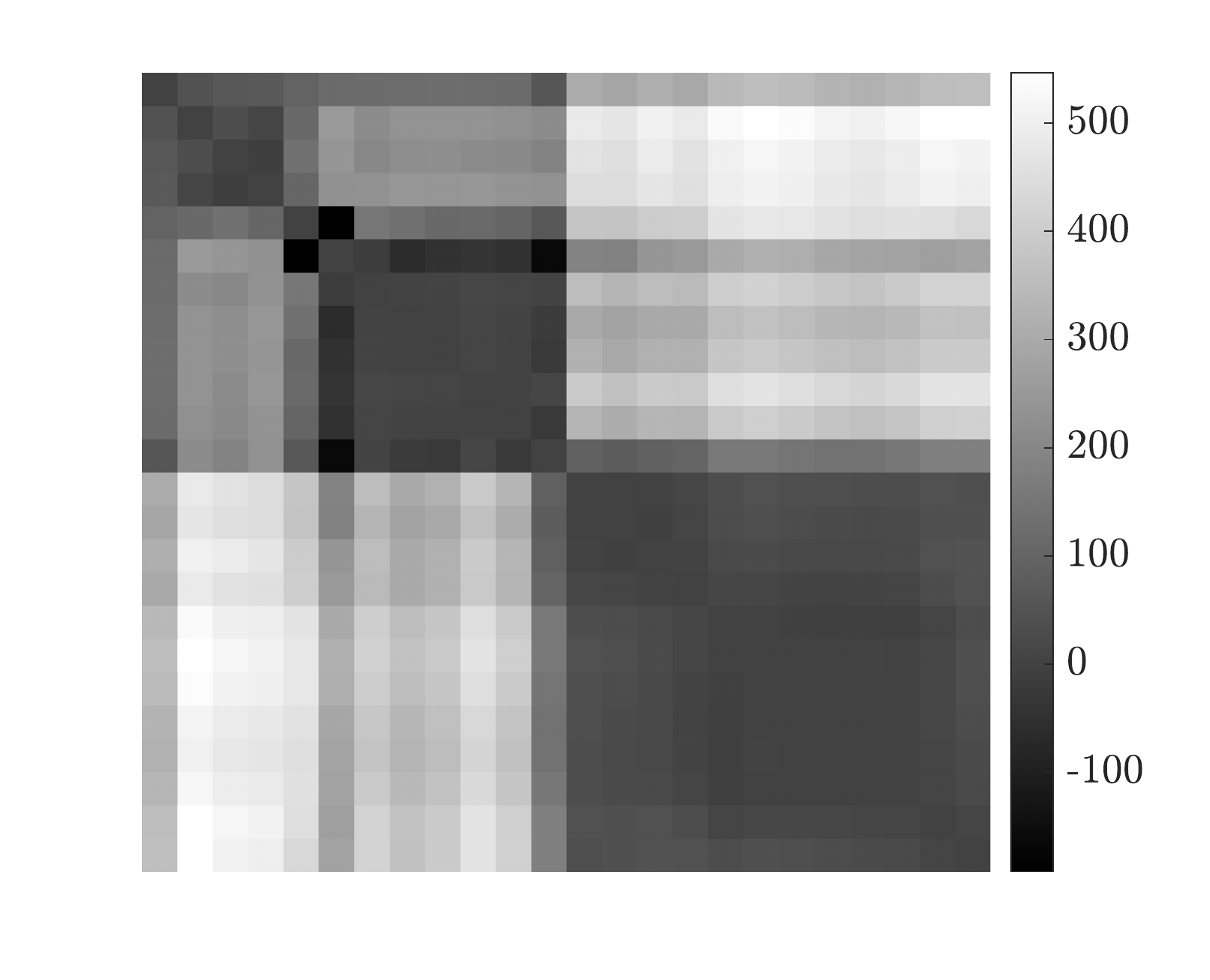}
\caption{\textbf{Left:} examples of ``jogging'' and ``walking'' trajectories from the MotionSense dataset. \textbf{Right:} the distance matrix constructed by Algorithm~\ref{alg:clustering} for 12 ``jogging'' and 12 ``walking'' trajectories.}
\label{fig:motionsense}
\end{center}
\vskip -0.2in
\end{figure}

\bibliographystyle{alphaabbr}
\bibliography{refs}

\end{document}